\documentclass{article} 

\usepackage{graphicx,epstopdf,amsfonts,,amsmath,algorithmic,hyperref, enumerate, paralist, bm,mathabx}
\usepackage{amsthm}
\usepackage{latexsym}
\usepackage{bm} 
\usepackage{amsmath,amssymb,amsfonts}
\usepackage{url}
\usepackage{enumitem}

\usepackage{lmodern}
\hypersetup{colorlinks=true,allcolors=blue}

\usepackage[affil-it]{authblk}

\usepackage[authoryear]{natbib}
\setcitestyle{open={(},close={)}}
\usepackage{bm}
\usepackage{algorithmic}
\usepackage{algorithm}
\usepackage{wrapfig}

\usepackage{geometry}

\newcommand{\sX}{{\mathcal X}}

\newcommand{\sP}{\mathcal{P}}
\newcommand{\sS}{\mathcal{S}}

\newcommand{\sC}{\mathcal{C}}

\newcommand{\sD}{\mathcal{D}}

\newcommand{\bbR}{\mathbb{R}} 

\newcommand{\bbZ}{\mathbb{Z}}

\newcommand{\si}{\Delta}

\newcommand{\supp}{\mathop{\mathrm{supp}}} 
\newcommand{\argmin}{\operatornamewithlimits{arg\ min}}

\newcommand{\aff}{\mathop{\text{aff}}} 
\newcommand{\co}{\mathop { \textnormal{conv}}} 
\newcommand{\spa}{\mathop { \textnormal{span}}} 
\newcommand{\rank}{\mathop { \textnormal{rank}}}
\newcommand{\unif}{\mathop { \text{unif}}}
\newcommand{\range}{\mathop { \text{range}}}

\newcommand{\ind}[1]{\bm{1}_{\{#1\}}} 

\newcommand{\set}[1]{\{#1\}}

\newcommand{\norm}[1]{\left\|#1\right\|}

\renewcommand{\vec}[1]{\bm{#1}}
\renewcommand{\set}[1]{\{#1\}}

\newcommand{\p}[1]{\tilde{P}_#1}

\newcommand{\est}[1]{\widehat{#1}}
\newcommand{\ed}[1]{#1^\dagger} 

\newtheorem{thm}{Theorem}
\newtheorem{prop}{Proposition}
\newtheorem{lemma}{Lemma}
\newtheorem{cor}{Corollary}
\newtheorem{defn}{Definition}

\newtheorem{rem}{Remark}

\title{A Mutual Contamination Analysis of Mixed Membership and Partial Label Models}
\usepackage{times}

\author{Julian Katz-Samuels}
\author{Clayton Scott \\ \texttt{\{jkatzsam,clayscot\}@umich.edu}}
\affil{Department of Electrical Engineering and Computer Science, \\
University of Michigan}

\begin{document}
\maketitle
\begin{abstract}
Many machine learning problems can be characterized by \emph{mutual contamination models}.
In these problems, one observes several random samples from different convex combinations of a set of unknown base distributions. It is of interest to \emph{decontaminate} mutual contamination models, i.e., to recover the base distributions either exactly or up to a permutation. This paper considers the general setting where the base distributions are defined on arbitrary probability spaces. We examine the decontamination problem in two  mutual contamination models that describe popular machine learning tasks: recovering the base distributions up to a permutation in a mixed membership model, and recovering the base distributions exactly in a partial label model for classification. We give necessary and sufficient conditions for identifiability of both mutual contamination models, algorithms for both problems in the infinite and finite sample cases, and introduce novel proof techniques based on affine geometry.
\end{abstract}
\section{Introduction}
In a mutual contamination model \citep{blanchard2014}, there are $L$ distributions $P_1,\ldots, P_L$ called \emph{base distributions}. The learner observes $M$ training random samples
\begin{equation*}
X^i_1, \ldots, X_{n_i}^i \overset{i.i.d.}{\sim} \tilde{P}_i = \sum_{j =1}^L \pi_{i, j} P_j
\end{equation*}
$i=1,\ldots, M$, where $\pi_{i,j} \geq 0$ and $\sum_{j} \pi_{i,j} = 1$. Here $\pi_{i,j}$ is the probability that an instance of the \emph{contaminated distribution} $\tilde{P}_i$ is a realization of $P_j$. The $\pi_{i,j}$ and $P_j$s are unknown and $\tilde{P}_i$ is observed through data. In this work, we avoid parametric models and assume that the sample space is arbitrary. The model can be stated concisely as
\begin{equation}
\tilde{\vec{P}} = \vec{\Pi} \vec{P} \label{model}
\end{equation}
where $\vec{P} = (P_1, \ldots, P_L)^T$, $\tilde{\vec{P}} = (\p{1},\ldots, \p{M})^T$, and $\vec{\Pi} = \begin{pmatrix}
\pi_{i,j}
\end{pmatrix}$ is an $M \times L$ matrix (which we call the \emph{mixing matrix}). The \emph{decontamination} problem is to recover the base distributions either exactly or up to a permutation.

We study the decontamination problem in two mutual contamination models that describe popular machine learning tasks. First, in \emph{mixed membership models}, the learner observes samples from $\p{1}, \ldots, \p{M}$ and the decontamination problem is to recover a permutation of $P_1, \ldots, P_L$. We refer to this problem as the decontamination or \emph{demixing} problem interchangeably. Recently, mixed membership models have become a powerful modeling tool for data where data points are associated with multiple distributions. Applications have appeared in a wide range of fields including image processing \citep{li}, population genetics \citep{pritchard}, document analysis \citep{blei2003}, and surveys \citep{berkman}. There is potential value in developing a theory of mixed membership models for arbitrary sample spaces since in some applications (e.g., image processing and computer vision) it may be preferable to model features as more general random variables \citep{poczos2012}. Such a theory will necessarily differ from previous work (reviewed below) which relies heavily on modeling the base and contaminated distributions as finite-dimensional probability vectors.

Consider the following application from natural language processing. In topic models, a form of mixed membership models, the $P_i$s are distributions over words and the $\p{i}$s represent documents. As such, $P_i$s are often treated as discrete random variables. However, it may be desirable to use word representations (mappings of words to $\bbR^d$) since word representations have contributed immensely to the success of systems on many tasks including parsing, entity recognition, and part-of-speech tagging \citep{luong2013}. In such an approach, the $P_i$s are continuous distributions over $\bbR^d$. Whereas most topic modeling algorithms could not operate on such a useful representation, our results and algorithms do apply. 

A partial label model is an alternative setting for multiclass classification with $L$ classes.\footnote{The partial label problem has also been referred to as the ``superset learning problem" or the ``multiple label problem" \citep{liu2014}.} In a partial label model, each data point is labeled with a \emph{partial label} $S \subset \set{1, \ldots, L}$; the true label is in $S$, but is it not known which label is the true one. If $S_i$ is a partial label, then we view patterns with label $S_i$ as distributed according to $\tilde{P}_i = \sum_{j \in S_i} \pi_{i,j} P_j$. In the mutual contamination model setting, the learner has access to the \emph{partial label matrix} $\vec{S} = (\ind{\vec{\Pi}_{i,j} > 0})$, an $M \times L$ matrix. The decontamination problem is to recover $(P_1, \ldots, P_L)^T$ \emph{exactly} (not just up to a permutation) from $\vec{S}$ and $\tilde{\vec{P}}$ (or its empirical version).

Most work has approached the partial label problem by trying to minimize the partial label error \citep{jin2002, nyugen2008, cour2011, liu2012}. The partial label error is the probability that a given classifier assigns a label to a training instance that is not contained in the partial label associated with the training instance. To our knowledge our work is the first to consider recovery of the base distributions $P_j$. Once the base distributions are known, it would then be possible to design a classifier according to any criterion, such as probability of error.

We make the following contributions: (i) We give necessary and sufficient conditions on $\vec{P}$, $\vec{\Pi}$, and $\vec{S}$ for identifiability of the mixed membership and partial label models under mild assumptions on the decontamination procedure.
(ii) We introduce novel algorithms for the demixing problem and partial label classification in the infinite and finite sample settings. These algorithms are nonparametric in the sense that they do not model $P_i$ as a probability vector or other parametric model. (iii) We develop novel estimators for distributions obtained by iteratively applying the $\kappa^*$ operator (defined below). (iv) We introduce novel proof techniques based on affine geometry.
\subsection{Notation}
Let $\bbZ^+$ denote the positive integers. For $n \in \bbZ^+$, let $[n] = \set{1, \ldots, n}$. If $\vec{x} \in \bbR^K$, let $x_i$ denote the $i$th entry of $\vec{x}$. In contrast, if $\vec{x}_j \in \bbR^K$, then $x_{j,i}$ denotes the $i$th entry of $\vec{x}_j$. Let $\vec{e}_i$ denote the length $L$ vector with $1$ in the $i$th position and zeros elsewhere. Let $\vec{\pi}_i \in \si_L \subset \bbR^L$ be the transpose of the $i$th row of $\vec{\Pi}$ where $\si_L$ denotes the $(L-1)$-dimensional simplex, i.e., $\si_L = \{ \vec{\mu} = (\mu_1,\ldots, \mu_L)^T \in \bbR^L \,| \, \sum_{i=1}^L \mu_i = 1 \text{ and }  \forall i:  \mu_i   \geq 0  \}$. Let $\Delta_L^M$ denote the product of $M$ $(L-1)$-dimensional simplices, the space of $M \times L$ stochastic matrices. Let $\sP$ denote the space of probability distributions on a measurable space $(\sX, \sC)$. Let $\supp(F)$ denote the support of a distribution $F$ on a Borel space. Let $\vec{B}^{M \times L}$ denote the space of binary matrices of dimension $M \times L$. Let $A$ be a set. Let $\aff A$ denote the affine hull of $A$, i.e., 
$\aff A = \set{\sum_{i=1}^K \theta_i \vec{x}_i | \vec{x}_1, \ldots , \vec{x}_K \in A, \sum_{i=1}^K \theta_i = 1}$.

\section{Related Work}
Mutual contamination models were employed by \citet{scott2013} and \citet{blanchard2014} to study classification with label noise. \citet{scott2013} assume that $M = L = 2$ and show how to recover $P_1$ and $P_2$ exactly (not just up to a permutation). \citet{blanchard2014} assume that  $M = L \geq 2$ and recover $P_1, \ldots, P_L$ exactly. The demixing problem in the current paper differs from these problems in that \emph{(i)} the number of contaminated distributions is not necessarily equal to the number of base distributions (i.e., we allow $M \neq L$) and, more significantly, \emph{(ii)} the demixing problem only requires recovering $P_1,\ldots, P_L$ up to a permutation. The partial label model differs from these problems in that it assumes access to the partial label matrix $\vec{S}$. We use some ideas from \citep{blanchard2014}, but make significant extensions. These differences lead us to make different (substantially weaker) assumptions on $\vec{\Pi}$ to achieve identifiability, although we use the same assumption on $\vec{P}$, namely that $P_1, \ldots, P_L$ are jointly irreducible (defined below).
\subsection{The Demixing Problem and Topic Models}
Our demixing problem may also be viewed as topic modeling on general domains. In topic modeling, the base distributions $P_i$ correspond to topics and the contaminated distributions $\p{i}$ to documents, which are regarded as mixtures of topics. In most cases, the $P_i$s are assumed to have a finite sample space. A variety of approaches have been proposed for topic modeling. The most common approach assumes a generative model for a corpus of documents and determines the maximum likelihood fit of the model given data. However, because maximum likelihood is NP-hard, these approaches must rely on heuristics that can get stuck in local minima \citep{arora_beyond_2012}. 

Recently, a trend towards algorithms for topic modeling with provable guarantees has emerged. These methods rely on the separability assumption \textbf{(SEP)} and its variants. According to \textbf{(SEP)}, $P_1,\ldots, P_L$ are distributions on a finite sample space and for every $i \in \set{1,\ldots,L}$, there exists a word $x \in \supp(P_i)$ such that $x \not \in \cup_{j \neq i} \supp(P_j)$. Our requirement that $P_1,\ldots, P_L$ are jointly irreducible is a natural generalization of separability of $P_1,\ldots, P_L$, as we will argue below. Specifically, if $P_1,\ldots, P_L$ have discrete sample spaces, separability and joint irreducibility coincide; however, if $P_1,\ldots, P_L$ are continuous, under joint irreducibility, $P_1,\ldots, P_L$ can have the same support. 

A key ingredient in these algorithms is to use the assumption of a finite sample space to view the distributions as probability vectors in Euclidean space; this leads to approaches based on non-negative matrix factorization (NMF), linear programs, and random projections \citep{donoho, arora_beyond_2012, arora_practical_2012, ding_2013, ding2014, recht_2012}. However, more general distributions cannot be viewed as finite-dimensional vectors. Therefore, topic modeling on general domains requires new techniques. Our work seeks to provide such techniques.
\subsection{Partial Label Model}
The partial label model has had two main formulations in previous work. In \textbf{(PL-1)}, instances from each class are drawn independently and the partial label for each instance is drawn independently from a set-valued distribution. In \textbf{(PL-2)}, training data are in the form of bags where each bag is a set of instances and the bag has a set of labels. Each instance belongs to a single class, and the set of labels associated with the bag is given by the union of the labels of the instances in the bag \citep{liu2014}. Our framework is similar to \textbf{(PL-2)}, although it does not assume a joint distribution on the features of instances and the partial labels. 

Most algorithms approach classification in a partial label model by picking a classifier that minimizes the partial label error \citep{jin2002, nyugen2008, cour2011, liu2012}. \citet{liu2014} develop learnability results in the realizable case for algorithms that use this approach. One of the key concepts that they take from \citet{cour2011} is \emph{the ambiguity degree}. It bounds the probability that a specific incorrect label appears in the partial label of an instance. Our approach has the advantage that it makes no assumption regarding realizability, which essentially means that $P_i$ have disjoint supports.

Our paper makes three main contributions to the literature. First, we address the fundamental question of identifiability of a partial label model. Second, we provide nonparametric algorithms for the infinite and finite sample settings. Finally, we introduce a novel approach based on affine geometry that differs significantly from approaches that minimize the partial label error.

\section{Necessary Conditions}
\label{nec_cond_sec}

We begin by developing necessary conditions for identifiability of a mixed membership model and, then, discuss necessary conditions for the identifiability of a partial label model. A mixed membership model is identifiable if, given $\tilde{\vec{P}}$, the pair $(\vec{\Pi}, \vec{P})$ that solves (\ref{model}) is uniquely determined. In general, this pair is not unique. For example, consider the case where $L=M$, $(\vec{\Pi},\vec{P})$ solves (\ref{model}), and $\vec{\Pi}$ is not a permutation matrix. Then, another solution is $\tilde{\vec{P}} = \vec{I} \tilde{\vec{P}}$. 

Furthermore, there are infinitely many non-trivial solutions in the realistic scenario where there is some $\p{i}$ in the interior of $\co(P_1, \ldots, P_L)$ and at least two of the $\p{j}$s are distinct. We can construct such solutions as follows. Without loss of generality, suppose that $i = 1$ and $\p{1} \neq \p{2}$. Then, since $\p{1}$ is in the interior of $\co(P_1, \ldots, P_L)$, there is some $\delta > 0$ such that for any $\alpha \in (1, 1 + \delta)$, $Q_\alpha = \alpha \p{1} + (1- \alpha) \p{2}$ is a distribution. Then, $\co(\p{1}, \ldots, \p{L}) \subseteq \co(Q_\alpha, \p{2}, \ldots, \p{L})$   and, consequently, there is some $\vec{\Pi^\prime} \in \Delta_L^L$ such that $(\vec{\Pi^\prime}, (Q_\alpha, \p{2}, \ldots, \p{L})^T)$ solves (\ref{model}). Clearly, by varying $\alpha$, there are infinitely many solutions to (\ref{model}). Moreover, we can replace $\p{2}$ in the above argument with any distribution in the convex hull of $\set{\p{j} : j \neq 1}$ that is not equal to $\p{1}$. 

In light of the above, it is natural to impose conditions on the decontamination procedure so as to eliminate trivial and other simple solutions to (\ref{model}). We now formalize the notion of a decontamination operator and present two conditions that we believe it should satisfy. Let $\phi$ denote a \emph{decontamination operator}: a function from $\sP^M$ to $\Delta_L^M \times \sP^L$ such that $\phi(\tilde{\vec{P}})$ returns $(\vec{\Pi}, \vec{P})$ that solves (\ref{model}). Let $\phi_1(\tilde{\vec{P}})$ return $\vec{\Pi}$ and $\phi_2(\tilde{\vec{P}})$ return $\vec{P}$.
\begin{defn}
A decontamination operator $\phi$ satisfies Maximality \textbf{(M)} iff $\phi(\tilde{\vec{P}}) = (\vec{\Pi}, \vec{P})$ implies that any pair $(\vec{\Pi}^\prime, \vec{P}^\prime)$ with $\vec{\Pi}^\prime \in \Delta_L^M$ and $\vec{P}^\prime = (P^\prime_1,\ldots, P^\prime_L)^T \in \mathcal{P}^L$ that solves (\ref{model}) satisfies $\set{P^\prime_1,\ldots, P^\prime_L} \subseteq \co(P_1,\ldots, P_L)$.
\end{defn} 
\noindent In words, \textbf{(M)} states that $\vec{P} = (P_1, \ldots, P_L)^T$ is a maximal collection of base distributions in the sense that it is not possible to move any of the $P_i$s outside of $\co(P_1, \ldots,P_L)$ and represent $\tilde{\vec{P}}$. 
\begin{defn}
A decontamination operator $\phi$ satisfies Linearity \textbf{(L)} iff $\phi_2(\tilde{\vec{P}}) = \vec{P}$ implies $\set{P_1, \ldots, P_L} \subseteq \spa(\p{1}, \ldots, \p{M})$.
\end{defn}
We believe that \textbf{(L)} is a reasonable requirement since it holds in the common situation in which there exist $\vec{\pi}_{i_1}, \ldots, \vec{\pi}_{i_L}$ that are linearly independent. Then for $I = \set{i_1, \ldots, i_L}$, we can write $\tilde{\vec{P}}_I = \vec{\Pi}_I \vec{P}$ where $\vec{\Pi}_I$ is the submatrix of $\vec{\Pi}$ containing only the rows indexed by $I$ and  $\tilde{\vec{P}}_I$ is similarly defined. Then, $\vec{\Pi}_I$ is invertible and $\vec{P} = \vec{\Pi}_I^{-1} \tilde{\vec{P}}$.

To formulate a necessary condition, we introduce another definition. For distributions $G$ and $H$, we say that $G$ is \emph{irreducible} with respect to $H$ if it is not possible to write $G = \gamma H + (1-\gamma) F$ where $F$ is a distribution and $0 < \gamma \leq 1$. This condition was used previously in the study of classification with label noise \citep{blanchard2010, scott2013}; here, we show that it arises in the context of a necessary condition for demixing mixed membership models. 
\begin{thm}
\label{nec_cond}
Let $\phi$ denote a decontamination operator and $\phi(\tilde{\vec{P}}) = (\vec{\Pi}, \vec{P})$. If $\phi$ satisfies \textbf{(M)}, then
\begin{description}
\item[(A)] $\forall i$, $P_i$ is irreducible with respect to every distribution in $\co(\set{P_j : j \neq i})$.
\end{description}
If $\phi$ satisfies \textbf{(L)}, then 
\begin{description}
\item[(B)] $\rank(\vec{\Pi}) \geq \dim \spa(P_1, \ldots, P_L)$.
\end{description}
\end{thm}
\begin{proof}[Sketch]
\begin{description}
\item[(A)] Proof by contraposition. Let $\phi$ be a demixing operator and $\phi(\tilde{\vec{P}}) = (\vec{\Pi}, \vec{P})$. Suppose that there is some $P_i$ and $Q \in \co(\set{P_j : j \neq i})$ such that $P_i$ is not irreducible with respect to $Q$. Then, there is some distribution $G$ and $\gamma \in (0,1]$ such that $P_i = \gamma Q + (1-\gamma) G$. First, we show that if $\gamma = 1$ or $G \in \co(P_1, \ldots, P_L)$, then $P_i \in \co(\set{P_k : k \neq i})$ in which case $\p{1}, \ldots, \p{M} \in \co(\set{P_j : j \neq i} \cup \set{R})$ for any distribution $R \not \in \co(P_1, \ldots, P_L)$. Thus, $\phi$ violates \textbf{(M)}. Second, we show that if $\gamma \in (0,1)$ and $G \not \in \co(P_1, \ldots, P_L)$, then $\co(P_1, \ldots, P_L) \subseteq \co(\set{P_k : k \neq i} \cup \set{G})$, from which it follows that $\phi$ violates \textbf{(M)}.
\item[(B)] Let $\phi(\tilde{\vec{P}}) = (\vec{\Pi}, \vec{P})$. Using the hypothesis that $\phi$ satisfies \textbf{(L)} and the relation $\tilde{\vec{P}} = \vec{\Pi} \vec{P}$, we show that $\dim \spa (\p{1}, \ldots , \p{M}) = \dim \spa(P_1, \ldots, P_L)$. Then, the result follows from the fact that $\p{1}, \ldots, \p{M} \in \range(\vec{\Pi})$.
\end{description}
\end{proof}
In sum, we have obtained necessary conditions on the base distributions and the mixing matrix for identifiability of a mixed membership model. In particular, Theorem \ref{nec_cond} implies that if $P_1, \ldots, P_L$ are linearly independent, then there must be at least as many contaminated distributions as base distributions, i.e., $M \geq L$. We also remark that \textbf{(A)} and \textbf{(B)} apply to \citet{blanchard2014} (in which no necessary conditions were given) since in that paper, the goal is to identify $P_1, \ldots, P_L$ exactly. Further, note that \textbf{(A)} appears as a sufficient condition in \citet{sanderson2014}.

Now, we turn to the partial label model. A partial label model is identifiable if given $\tilde{\vec{P}}$ and $\vec{S}$, the pair $(\vec{\Pi}, \vec{P})$ that is consistent with $\vec{S}$ and solves $\tilde{\vec{P}} = \vec{\Pi} \vec{P}$ is unique. A decontamination operator for the partial label model $\psi$ is a function from $\sP^M \times \vec{B}^{M \times L}$ to $\si_L^M \times \sP^L$. This difference necessitates slight modifications to our notions of Maximality and Linearity, which we refer to as \textbf{(M$'$)} and \textbf{(L$'$)}, respectively. Let $\psi(\tilde{\vec{P}}, \vec{S}) = (\vec{\Pi}, \vec{P})$.  If $\psi$ satisfies \textbf{(M$'$)}, then $(\vec{\Pi}, \vec{P})$ need not satisfy \textbf{(A)}. However, if $\psi$ satisfies \textbf{(B)}, then $(\vec{\Pi}, \vec{P})$ satisfies \textbf{(B)}. See the Appendix for details. Identifiability of the partial label model implies a condition \textbf{(C)} on the partial label matrix.
\begin{prop}
\label{partial_necessity}
Let $\psi$ denote a decontamination operator for the partial label model and $\psi(\tilde{\vec{P}}, \vec{S}) = (\vec{\Pi}, \vec{P})$. If $(\tilde{\vec{P}}, \vec{S})$ is identifiable, then 
\begin{description}
\item[(C)] $\vec{S}$ does not contain a pair of identical columns.
\end{description}
\end{prop}
\begin{proof}[Sketch]
Let $(\tilde{\vec{P}}, \vec{S})$ such that $\vec{S}$ does not satisfy \textbf{(C)}. Without loss of generality, suppose that the $\vec{S}_{:,1}$ and $\vec{S}_{:,2}$ are identical. Let $(\vec{\Pi}, \vec{P})$ satisfy (\ref{model}). Interchange the first two columns of $\vec{\Pi}$ to obtain $\vec{\Pi}^\prime$ and the first two rows of $\vec{P}$ to obtain $\vec{P}^\prime$. Then, ($\vec{\Pi}^\prime, \vec{P}^\prime)$ solves (\ref{model}), establishing that $(\tilde{\vec{P}}, \vec{S})$ is not identifiable.
\end{proof}
\vspace{-2em}
\section{Sufficient Conditions}
Now, we state the sufficient conditions for demixing a mixed membership model and decontaminating a partial label model. Consider the following property from \citet{blanchard2014}.
\begin{defn}
The distributions $\set{P_i}_{1 \leq i \leq L}$ are \emph{jointly irreducible} iff the following equivalent conditions hold
\begin{enumerate}
\item[(a)] It is not possible to write
\begin{equation*}
\sum_{i \in I} \epsilon_i P_i = \alpha\sum_{i \not \in I} \epsilon_i P_i + (1- \alpha) H
\end{equation*}
where $I \subset [L]$ such that $1 \leq |I| < L, \epsilon_i$ are such that $\epsilon_i \geq 0$ and $\sum_{i \in I} \epsilon_i = \sum_{i \not \in I} \epsilon_i = 1$, $\alpha \in (0,1]$ and $H$ is a distribution;
\item[(b)] $\sum_{i=1}^L \gamma_i P_i$ is a distribution implies that $\gamma_i \geq 0 \, \forall i$.
\end{enumerate}
\end{defn}
\noindent Conditions \emph{(a)} and \emph{(b)} give two ways to think about joint irreducibility. Condition \emph{(a)} says that every convex combination of a subset of the $P_i$s is irreducible with respect to every convex combination of the other $P_i$s. Condition \emph{(b)} says that if a distribution is in the span of $P_1, \ldots, P_L$, it is in their convex hull. Joint irreducibility holds when each $P_i$ has a region of positive probability that does not belong to the support of any of the other $P_i$s; thus, separability of $P_i$s entails joint irreducibility of $P_1,\ldots, P_L$. However, the converse is not true: the $P_i$s can have the same support and still be jointly irreducible (e.g., $P_i$s Gaussian with a common variance and distinct means \citep{scott2013}). 

Our two sufficient conditions for the mixed membership model are:
\begin{description}
\item[(A$'$)] $P_1, \ldots, P_L$ are jointly irreducible.
\item[(B$'$)] $\vec{\Pi}$ has full rank.
\end{description}
These conditions are consistent with requiring that a decontamination operator satisfy \textbf{(M)} and \textbf{(L)}, as indicated by the following Proposition.\footnote{We give an analogous result for the partial label model in the Appendix.}
\begin{prop}
\label{max_ji_consistency}
If a decontamination operator $\phi$ is such that $\phi(\tilde{\vec{P}}) = (\vec{\Pi}, \vec{P})$ only if $(\vec{\Pi}, \vec{P})$ satisfy \textbf{(A$'$)} and \textbf{(B$'$)}, then $\phi$ satisfies \textbf{(M)} and \textbf{(L)}. 
\end{prop}
\noindent By comparing \textbf{(A)} with \textbf{(A$'$)} and \textbf{(B)} with \textbf{(B$'$)}, we see that the proposed sufficient conditions are not much stronger than \textbf{(M)} and \textbf{(L)} require. Since joint irreducibility of $P_1, \ldots, P_L$ entails their linear independence by Lemma B.1 of \citet{blanchard2014}, under \textbf{(A$'$)}, \textbf{(B)} and \textbf{(B$'$)} are the same. \textbf{(A$'$)} differs from \textbf{(A)} in that it requires that a slightly larger set of distributions are irreducible with respect to convex combinations of the remaining distributions. Specifically, under \textbf{(A$'$)}, \emph{every convex combination} of a subset of the $P_i$s is irreducible with respect to every convex combination of the other $P_i$s whereas \textbf{(A)} only requires that every $P_i$ be irreducible with respect to every convex combination of the other $P_i$s. 

Our sufficient conditions for the partial label model include \textbf{(A$'$)}, \textbf{(B$'$)}, as well as \textbf{(C)}. Note that our sufficient and necessary conditions on $\vec{S}$ are identical. 

A couple of points are in order regarding how \textbf{(A$'$)} and \textbf{(B$'$)} relate to the sufficient conditions in \citet{blanchard2014}. \citet{blanchard2014} also assume \textbf{(A$'$)}. However, the so-called \emph{recoverability} assumption on $\vec{\Pi}$ in \citet{blanchard2014} is substantially stronger than \textbf{(B$'$)}. Recoverability says that $\vec{\Pi}^{-1}$ has a very specific structure: positive diagonal entries and nonpositive off-diagonal entries. To demix a mixed membership model, we are able to weaken the assumption on $\vec{\Pi}$ because the goal is to recover the base distributions only up to a permutation (as opposed to a specific permutation of the base distributions). Recall that in the classification problem in a partial label model, the goal is to recover $P_1, \ldots, P_L$ exactly. In comparison to \citet{blanchard2014}, we are able to weaken the assumption on $\vec{\Pi}$ because the partial label matrix $\vec{S}$ gives a considerable amount of useful information.
 
\section{Population Case}
In this section, to establish that the above conditions are indeed sufficient, we show that under these assumptions, mixed membership models and partial label models can be decontaminated in the population case. Henceforth, we assume that $P_1, \ldots, P_L$ satisfy are jointly irreducible, $\vec{\Pi}$ has full rank, and $M=L$. We give extensions to the non-square case in the Appendix. We begin by reviewing the necessary background.
\subsection{Background}
This paper relies on two important quantities from \citet{blanchard2010} and \citet{blanchard2014}, which we now review.
\begin{prop}
\label{label_noise_bin_case}
Given probability distributions $F_0, F_1$ on a measurable space $(\sX, \sC)$, define
\begin{align*}
\kappa^*(F_0 \, | \, F_1)  = \max \{ \kappa \in [0,1] | \, \exists \text{ a distribution } G \text{ s.t. } F_0 = (1-\kappa)G + \kappa F_1 \};
\end{align*}
If $F_0 \neq F_1$, then $\kappa^*(F_0 \, | \, F_1) < 1$ and the above supremum is attained for a unique distribution $G$ (which we refer to as the \emph{residue} of $F_0$ wrt. $F_1$). Furthermore, the following equivalent characterization holds:
\begin{equation*}
\kappa^*(F_0 \, | \, F_1) = \inf_{C \in \sC, F_1(C) > 0} \frac{F_0(C)}{F_1(C)}.
\end{equation*}
\end{prop}
$\kappa^*(F_0 \, | \, F_1)$ can be thought of as the maximum possible proportion of $F_1$ in $F_0$. Note that $\kappa^*( F_0 \, | \, F_1) = 0$ iff $F_0$ is irreducible wrt $F_1$. We can think of $1-\kappa^*(F_0 \, | \, F_1)$ as a statistical distance since it it non-negative and equal to zero if and only if $F_0 = F_1$. We refer to this quantity as the two-sample $\kappa^*$ operator. To obtain the residue of $F_0$ wrt $F_1$, one computes Residue($F_0 \, | \, F_1$) (see Algorithm \ref{two_residue}); this is well-defined under Proposition \ref{label_noise_bin_case} when $F_0 \neq F_1$.

We now turn to the multi-sample generalization of $\kappa^*$ defined in \citet{blanchard2014}, which we call the multi-sample $\kappa^*$ operator.
\begin{defn}
Given distributions $F_0, \ldots, F_K$, define $\kappa^*(F_0 \, | \, F_1, \ldots, F_K) =$
\begin{align*}
\max(  \sum_{i=1}^K \nu_i :   \nu_i \geq 0, \sum_{i=1}^K \nu_i \leq 1, \exists \text{ distribution } G \text{ s.t. } F_0 = (1 - \sum_{i=1}^K \nu_i) G + \sum_{i=1}^K \nu_i F_i)
\end{align*}
\end{defn}
Under our sufficient conditions, there exists some $G$ attaining the above maximum. However, the $G$ is not necessarily unique. Any $G$ attaining the maximum is called a multi-sample residue of $F_0$ wrt $F_1, \ldots, F_K$. The algorithm Residue($F_0 \, | \, \set{F_1, \ldots, F_K}$) returns one of these $G$ (see Algorithm \ref{multi_residue}).

In previous work that assumes $P_i$ are probability vectors, distributions are compared using $l_p$ distances. In our setting of general probability spaces, we use $\kappa^*$ to compare different distributions. 

\subsection{Mixture Proportions}
If $\vec{\eta} \in \bbR^L$ and $Q = \vec{\eta}^T \vec{P}$, we say that $\vec{\eta}$ is the \emph{mixture proportion} of $Q$. Since by Lemma B.1 of \citet{blanchard2014}, joint irreducibility of $P_1,\ldots, P_L$ implies linear independence of $P_1,\ldots, P_L$, mixture proportions are well-defined.

An important feature of our solution strategy involves determining where the mixture proportions of distributions are in the simplex $\si_L$. To make this precise, we introduce the following definitions. If $i \in [L]$, we say that $\co(\set{\vec{e}_j : j \neq i})$ is a \emph{face} of the simplex $\si_L$; if $A \subset [L]$ and $|A| = k$, we also say that $\co(\set{\vec{e}_j : j \in A})$ is a \emph{$k$-face} of $\si_L$. If $\vec{\eta} \in \bbR^L$, $Q$ is a distribution, and $Q = \vec{\eta}^T \vec{P}$, we say that $\sS(\vec{\eta}) = \set{j : \eta_j > 0}$ is the \emph{support set} of $\vec{\eta}$ or the support set of $Q$. Note that in this setting, $\sS(\vec{\eta})$ consists of the indices of all the nonzero entries in the mixture proportion $\vec{\eta}$ by definition of joint irreducibility. Finally, for $\vec{\eta}_i \in \si_L$, and $Q_i = \vec{\eta}_i^T \vec{P}$ for $i =1,2$, we say that the distributions $Q_1$ and $Q_2$ (or the mixture proportions $\vec{\eta}_1$ and $\vec{\eta}_2$) are \emph{on the same face} of the simplex $\si_L$ if there exists $j \in [L]$ such that $\vec{\eta}_1, \vec{\eta}_2 \in \co (\set{\vec{e}_k : k \neq j})$. 

The heart of our approach is that under joint irreducibility, one can interchange distributions $Q_1, \ldots, Q_K$ and their mixture proportions $\vec{\eta}_1, \ldots, \vec{\eta}_K$, as indicated by the following Proposition. 
\begin{prop}
\label{equiv_opt}
Let $Q_i = \vec{\eta}_i^T \vec{P}$ for $i \in [L]$ and $\vec{\eta}_i \in \si_L$. If $\vec{\eta}_1, \ldots, \vec{\eta}_L$ are linearly independent and $P_1, \ldots, P_L$ are jointly irreducible, then for any $i \in [L]$ and $A \subseteq [L] \setminus \set{i}$, $\kappa^*(Q_i \, | \, \set{Q_j : j \in  A}) =  \kappa^*(\vec{\eta}_i \, | \, \set{\vec{\eta}_j : j \in  A})<1$. Further, $\vec{\gamma} \in \si_L$ is a residue of $\vec{\eta}_i$ wrt $\set{\vec{\eta}_j : j \in  A}$ if and only if $G = \vec{\gamma}^T \vec{P}$ is a residue of $Q_i$ wrt $\set{Q_j : j \in  A}$.
\end{prop}
In words, this proposition says that the optimization problem given by $\kappa^*(Q_i \, | \, \set{Q_j : j \in  A})$ is equivalent to the optimization problem given by $\kappa^*(\vec{\eta}_i \, | \, \set{\vec{\eta}_j : j \in  A})$. Thus, joint irreducibility of $P_1, \ldots, P_L$ and linear independence of the mixture proportions ensure that we can reduce the decontamination problem of a mutual contamination model to the geometric problem of recovering the vertices of a simplex by applying $\kappa^*$ to points (i.e., the mixture proportions) in the simplex. This makes the figures below valid for general distributions (see Figures \ref{2_case}, \ref{3_case}, and \ref{partial_label_figure}).
\subsection{The Demix Algorithm}
In essence, the Demix algorithm (see Algorithm \ref{demix_alg}) computes the inverse of $\vec{\Pi}$ in a sequential fashion using $\kappa^*$ (recall we are assuming $M=L$). It is a recursive algorithm. Let $S_1, \ldots, S_K$ denote $K$ contaminated distributions. In the base case, the algorithm takes as its input two contaminated distributions $S_1$ and $S_2$. It returns Residue($S_1 \, | \, S_2$) and Residue($S_2 \, | \, S_1$), which are a permutation of the two base distributions (see Figure \ref{2_case}). When $K > 2$, Demix finds $K-1$ distributions $R_2,\ldots, R_K$ on the same $(K-1)$-face using a subroutine FaceTest (see Algorithm \ref{face_test_alg}) and recursively applies Demix to $R_2,\ldots, R_K$ to obtain distributions $Q_1,\ldots, Q_{K-1}$. $Q_1, \ldots, Q_{K-1}$ are a permutation of $K-1$ of the base distributions. Subsequently, the algorithm computes a sequence of residues using $Q_1, \ldots, Q_{K-1}$ and $\frac{1}{K} \sum_{i=1}^K S_i$ to obtain $Q_K$, the remaining base distribution (see Figure \ref{3_case} for an execution of the algorithm).
\begin{algorithm}
\caption{Demix($S_1, \ldots, S_K$)}
\textbf{Input: }$S_1, \ldots, S_K$ are distributions
\label{demix_alg}
\begin{algorithmic}[1]
\IF{$K = 2$}
\STATE $Q_1 \longleftarrow$ Residue($S_1 \, | \, S_2$)
\STATE $Q_2 \longleftarrow$ Residue($S_2 \, | \, S_1$)
\RETURN $(Q_1, Q_2)^T$
\ELSE
\STATE $Q \longleftarrow \text{ unif distributed element in} \co (S_2,\ldots,S_K)$
\STATE $n \longleftarrow 1$
\STATE $T \longleftarrow 0$
\WHILE{$T == 0$}
\STATE $n \longleftarrow n + 1$
\FOR{$i=2,\ldots, K$}
\STATE $R_i \longleftarrow \text{Residue}(\frac{1}{n} S_i + \frac{n-1}{n} Q \, | \, S_1)$
\ENDFOR
\STATE $T \longleftarrow \text{FaceTest}(R_2,\ldots,R_K)$	
\ENDWHILE
\STATE $(Q_1, \ldots, Q_{K-1})^T \longleftarrow \text{Demix}(R_2,\ldots,R_K)$
\STATE $Q_K \longleftarrow \frac{1}{K} \sum_{i=1}^K S_i $ 
\FOR{$i = 1, \ldots, K-1$} \label{last_loop_start_demix}
\STATE $Q_K \longleftarrow \text{Residue}(Q_K \, | \, Q_i)$
\ENDFOR \label{last_loop_end_demix}
\RETURN $(Q_1, \ldots, Q_K)^T$
\ENDIF
\end{algorithmic}
\end{algorithm}
A number of remarks are in order. First, although we compute the residue of $\frac{1}{n} S_i + \frac{n-1}{n} Q$ wrt $S_1$ for each $i \neq 1$, there is nothing special about the distribution $S_1$. We could replace $S_1$ with any $S_j$ where $j \in [K]$, provided that we adjust the rest of the algorithm accordingly. Second, we can replace the sequence $\set{\frac{n-1}{n}}_{n=1}^\infty$ with any sequence $\alpha_n \nearrow 1$. Third, $Q$ is sampled randomly to ensure w.p. 1 it does not lie on a union of $k-2$ dimensional affine subspaces (a zero measure set) on which the algorithm fails. Specifically, under the assumptions of Theorem \ref{demix_identification}, w.p. $1$, the residue of $Q$ wrt $S_1$ lies on the interior of one of the $(K-1)$-faces, so if we pick distributions close enough to $Q$, their residues wrt $S_1$ lie on the same face.
\begin{thm}
\label{demix_identification}
Let $P_1, \ldots, P_L$ be jointly irreducible, $\vec{\pi}_1, \ldots, \vec{\pi}_L \in \si_L$ be linearly independent, and $\p{i} = \vec{\pi}_i^T \vec{P}$ for $i \in [L]$. Then, with probability $1$, Demix$(\p{1}, \ldots, \p{L})$ terminates and returns a permutation of  $(P_1, \ldots, P_L)^T$.
\end{thm}

\subsection{The Partial Label Algorithm}
\begin{algorithm}
\caption{PartialLabel($\vec{S}, (\p{1}, \ldots, \p{L})^T$)}
\begin{algorithmic}[1]
\label{partial_label_alg}
\FOR{$i = 1, \ldots, L$}
\STATE $Q_i \longleftarrow \text{ uniformly random distribution in } \co(\p{1}, \ldots, \p{L})$ 
\STATE $W_i \longleftarrow Q_i$
\ENDFOR
\STATE $k = 2$\\
\STATE $\text{FoundVertices } \longleftarrow 0$
\WHILE{$\text{FoundVertices } == 0$}
\FOR{$i = 1, \ldots, L$}
\STATE $\bar{Q}_i \longleftarrow \frac{1}{L-1} [\sum_{j > i} Q_j + \sum_{j < i} W_j]$
\STATE $W_i \longleftarrow \text{Residue}(\frac{1}{k} Q_i + (1- \frac{1}{k}) \bar{Q}_i \, | \,   \set{Q_j }_{j > i} \cup \set{W_j}_{j < i})$
\ENDFOR
\STATE $k = k + 1$\\
\STATE $(\text{FoundVertices, } \vec{C}) \longleftarrow \text{VertexTest}(\vec{S}, \p{1}, \ldots, \p{L}, W_1, \ldots, W_L)$
\ENDWHILE
\RETURN $\vec{C} (W_1, \ldots, W_L)^T$
\end{algorithmic}
\end{algorithm}
The PartialLabel algorithm (see Algorithm \ref{partial_label_alg}) proceeds by iteratively finding sets of candidate distributions $(W_1, \ldots, W_L)^T$ for increasing values of $k$. Given each $(W_1, \ldots, W_L)^T$, it runs an algorithm VertexTest (see Algorithm \ref{vertex_test_alg}) that uses $\p{1}, \ldots, \p{L}$ and the partial label matrix $\vec{S}$ to determine whether $(W_1, \ldots, W_L)^T$ is a permutation of the base distributions $(P_1, \ldots, P_L)^T$. If $(W_1, \ldots, W_L)^T$ is a permutation of $(P_1, \ldots, P_L)^T$, VertexTest constructs the corresponding permutation matrix for relating these distributions. If not, it returns a value indicating that the PartialLabel algorithm increment $k$ to find another candidate set of distributions. 
\begin{thm}
\label{partial_identication}
Suppose that $P_1, \ldots, P_L$ satisfy \textbf{(A$'$)}, $\vec{\Pi}$ satisfies \textbf{(B$'$)}, and $\vec{S}$ satisfies \textbf{(C)}. If $(R_1, \ldots, R_L)^T \longleftarrow \text{PartialLabel}(\vec{S}, (\p{1}, \ldots, \p{L})^T)$, then $R_i = P_i$ for all $i \in [L]$. 
\end{thm}
\begin{proof}[Sketch]
We adopt the notation from the description of the algorithm with the exception that we make explicit the dependence on $k$ by writing $W_i^{(k)}$ instead of $W_i$ and $\bar{Q}^{(k)}_i$ instead of $\bar{Q}_i$. In this proof sketch, we  only show that there is a $K$ such that for all $k \geq K$, $(S^{(k)}_1, \ldots, S^{(k)}_L)^T$ is a permutation of $(P_1, \ldots, P_L)^T$. 

Let $Q_i = \vec{\tau}^T_i \vec{P}$, $\bar{Q}^{(k)}_i = \vec{\tau}^{{(k)}^T} \vec{P}$, and  $W_i^{(k)} = \vec{\gamma}_i^{{(k)}^T} \vec{P}$. We prove the claim inductively. First, we prove the base case: there is large enough $k$ such that $W_1^{(k)} = P_i$ for some $i \in [L]$. By Lemma \ref{random_lin_indep}, $\vec{\tau}_1, \ldots, \vec{\tau}_L$ are linearly independent. Therefore, $\aff(\vec{\tau}_2, \ldots, \vec{\tau}_L)$ gives a hyperplane with an associated open halfspace $\vec{H}$ that contains $\vec{\tau}_1$ and at least one $\vec{e}_j$. We can pick $k$ large enough such that for all $\vec{e}_j \in \vec{H}$, $\lambda_k \equiv \frac{1}{k} \vec{\tau}_1 + \frac{k-1}{k} \vec{\tau}^{(k)}_1 \in \co(\vec{e}_j, \vec{\tau}_2, \ldots, \vec{\tau}_L)$. Then, for all $\vec{e}_j \in \vec{H}$, there exists $\kappa_j > 0$ such that 
\begin{align*}
\vec{\lambda}_k & = \kappa_j \vec{e}_j + (1- \kappa_j) \tilde{\tau}_j
\end{align*}
for some $\tilde{\tau}_j \in \co(\vec{\tau}_2, \ldots, \vec{\tau}_L)$. Suppose that there is $\vec{e}_i, \vec{e}_j \in \vec{H}$ such that $\kappa_i = \kappa_j$. Then, it can be shown that $\vec{e}_j - \vec{e}_i, \vec{\tau}_2, \ldots, \vec{\tau}_L$ are linearly dependent, which cannot happen by the randomness of $\vec{\tau}_2, \ldots, \vec{\tau}_L$ (Lemma \ref{random_lin_indep}). Therefore, there is a unique minimum $\kappa_j$ that satisfies the above relation. By Lemma \ref{single}, $\vec{e}_j$ is the residue of $\vec{\lambda}_k$ wrt $\set{\vec{\tau}_2, \ldots, \vec{\tau}_L}$. By Proposition \ref{equiv_opt}, it follows that $W_1^{(k)}$ is one of the base distributions.
\end{proof}
The VertexTest algorithm proceeds as follows on a vector of candidate distributions $(Q_1, \ldots, Q_L)^T$. First, it determines whether there are two distinct distributions $Q_i, Q_j$ such that $Q_i$ is not irreducible wrt $Q_j$ (see Section \ref{nec_cond_sec} for definition of irreducibility), in which case $(Q_1, \ldots, Q_L)^T$ cannot be a permutation of $(P_1, \ldots, P_L)^T$. If there is such a pair, it returns a value indicating that $(Q_1, \ldots, Q_L)^T$ is not a permutation of $(W_1, \ldots, W_L)^T$. Otherwise, it finds for each $Q_i$ the set of $\p{j}$ such that $\p{j}$ and $Q_i$ lie on the same face of the simplex. Finally, using $\vec{S}$, it iteratively determines for each $Q_i$ whether there is some $P_j$ such that $P_j = Q_i$. If $(Q_1, \ldots, Q_L)^T$ is a permutation of $(P_1, \ldots, P_L)^T$, this procedure finds the appropriate permutation matrix; otherwise, the algorithm returns a value indicating that $(Q_1, \ldots, Q_L)^T$ is not a permutation of $(W_1, \ldots, W_L)^T$. 

\section{Estimation}
In this section, we develop novel estimators that can be used to extend the Demix algorithm to the finite sample case. This will be the basis of finite sample algorithms for the demixing problem in mixed membership models and the decontamination problem of partial label models. Let $\sX = \bbR^d$ be equipped with the standard Borel $\sigma$-algebra $\sC$ and $\p{1}, \ldots, \p{L}$ be probability distributions on this space.  Suppose that we observe for $i = 1, \ldots, L$, 
\begin{equation*}
X_1^i, \ldots, X_{n_i}^i \sim \p{i}.
\end{equation*}
Let $\sS$ be any VC class with VC-dimension $V < \infty$, containing the set of all open balls, all open rectangles, or some other collection of sets that generates the Borel $\sigma$-algebra $\sC$. Define $\epsilon_i(\delta_i) \equiv 3 \sqrt{\frac{V \log(n_i +1) - \log(\delta_i /2)}{n_i}}$ for $i = 1, \ldots, L$.

Our goal is to establish estimators $\est{Q}_1, \ldots, \est{Q}_L$ that, when suitably permuted, converge uniformly on $\sS$ to $P_1, \ldots, P_L$. Previously developed uniform convergence results assume access to i.i.d. samples and are based on the VC inequality \citep{blanchard2010}. This inequality says that for each $i \in [L]$, and $\delta >0$, the following holds with probability at least $1 - \delta$:
\begin{equation*}
\sup_{S \in \sS}|\p{i}(S) - \ed{\tilde{P}}_i(S) | \leq \epsilon_i(\delta)
\end{equation*}
where the \emph{empirical distribution} is given by $\ed{\tilde{P}}_i(S) = \frac{1}{n_i} \sum_{j=1}^{n_i} \ind{X^i_{j} \in S}$. The challenge is that because of the recursive nature of the Demix algorithm, we cannot assume access to i.i.d. samples to estimate every distribution that arises. We show that uniform convergence of distributions propagates through the algorithm if we employ an estimator of $\kappa^*$ with a known rate of convergence. 

Let $\est{F}$ be an estimate of a distribution $F$ and $\sD(\est{F}) = \set{i : \est{F} \text{ relies on data from the distribution } \p{i}}$. We introduce the following estimator using the estimates $\est{F}$ and $\est{H}$: 
\begin{equation*}
\est{\kappa}( \est{F} \, | \, \est{H}) = \inf_{S \in \sS} \frac{\est{F}(S) + \gamma( \sD(\est{F})) }{(\est{H}(S) - \gamma(\sD(\est{H})))_+}
\end{equation*}
where $\gamma(I) = \sum_{i \in I} \epsilon_i(\frac{1}{n_i})$ and $I \subseteq [L]$. Notice that when $\est{F}$ and $\est{H}$ are empirical distributions, e.g., $\est{F} = \ed{\tilde{P}}_i$ and $\est{H} = \ed{\tilde{P}}_j$, we recover the consistent estimator from \citet{blanchard2010}: $\est{\kappa}( \est{F} \, | \, \est{H}) = \inf_{S \in \sS} \frac{\est{F}(S) + \epsilon_i(\frac{1}{n_i})  }{(\est{H}(S) - \epsilon_j(\frac{1}{n_j}))_+}$.

Based on the estimator $\est{\kappa}$, we introduce the following estimator of the residue of $F$ wrt $H$.
\begin{algorithm}
\caption{ResidueHat($\est{F} \, | \, \est{H}$)}
\label{residue_hat}
\textbf{Input: }$\est{F}, \est{H} \text { are estimates of } F,H$
\begin{algorithmic}[1]
\STATE $\est{\kappa} \longleftarrow \est{\kappa}(\est{F} \, | \, \est{H})$
\RETURN $\frac{\est{F} - \est{\kappa}(1 - \est{H})}{1 - \est{\kappa}}$
\end{algorithmic}
\end{algorithm}

\begin{defn}
Let $G \longleftarrow \text{Residue}(F \, | \, H)$ and $\est{G} \longleftarrow \text{ResidueHat}(\est{F} \, | \, \est{H})$. We call $\est{G}$ a \emph{ResidueHat estimator} of $G$ if (i) $F \neq H$, (ii) $F, H \in \co(P_1, \ldots, P_L)$, and (iii) $\est{F}$ and $\est{H}$ are either empirical distributions or ResidueHat estimators of $F$ and $H$.
\end{defn}
\noindent Note that the above definition is recursive and matches the recursive structure of the Demix algorithm. 

To use ResidueHat estimators to estimate the $P_i$s, we build on the rate of convergence result from \citet{scott2015}. In \citet{scott2015}, a rate of convergence was established for an estimator of $\kappa^*$ using empirical distributions; we extend these results to our setting of recursive estimators and achieve the same rate of convergence. To ensure that this rate of convergence holds for every estimate in our algorithm, we introduce the following condition.
\begin{description}
\item[(A$''$)] $P_1, \ldots, P_L$ are such that $\forall i$ $\supp(P_i) \not \subseteq \cup_{j \neq i} \supp(P_j) $. 
\end{description}
Note that this assumption is a natural generalization of the separability assumption.

Let $\vec{n} \equiv (n_1,\ldots, n_L)$; we write $\vec{n} \longrightarrow \infty$ to indicate that $\min_{i} n_i \longrightarrow \infty$. The following result establishes sufficient conditions under which ResidueHat estimates converge uniformly.
\begin{prop}
\label{uniform_convergence}
If $P_1, \ldots, P_L$ satisfy \textbf{(A$''$)} and $\est{G}$ is a ResidueHat estimator of a distribution $G \in \co(P_1, \ldots, P_L)$, then $\sup_{S \in \sS} | \est{G}(S) - G(S)| \overset{i.p.}{\longrightarrow} 0$ as $\vec{n} \longrightarrow \infty$.
\end{prop}
\begin{proof}[Sketch]
Let $\est{F}_1, \ldots, \est{F}_K$ denote the ResidueHat estimators in terms of which $\widehat{G}$ is defined and $F_1, \ldots, F_K$ the distributions which they estimate. We show that if $\est{F}_i$ and $\est{F}_j$ satisfy uniform deviation inequalities, $F_i \neq F_j$, and $\est{F}_l \longleftarrow \text{ResidueHat}(\est{F}_i \, | \, \est{F}_j)$, then there exists constants $A_1, A_2 > 0$ such that for large enough $\min(n_i : i \in \sD(\est{F}_i) \cup \sD(\est{F}_j))$, with probability at least $1 - A_2 \sum_{i \in \sD(\est{F}_i) \cup \sD(\est{F}_j)} \frac{1}{n_i}$ and for all $S \in \sS$,
\begin{align*}
|\est{\kappa}(\est{F}_i \, | \, \est{F}_j) - \kappa^*(F_i \, | \, F_j)| & \leq A_1 \gamma(\sD(\est{F}_i) \cup \sD(\est{F}_j)) \\
|\est{F}_l(S) - F_l(S) | & < A_1 \gamma(\sD(\est{F}_l))
\end{align*}
Since $F_i \neq F_j$, $\est{F}_i$ and $\est{F}_j$ satisfy uniform deviation inequalites by assumption, and $P_1, \ldots, P_L$ satisfy \textbf{(A$''$)}, we show that by Lemma \ref{rate}, the first inequality (a rate of convergence) holds. Then, by the first inequality, $F_i \neq F_j$, and the assumption that $\est{F}_i$ and $\est{F}_j$ satisfy uniform deviation inequalites, we have that Lemma \ref{vc_type_ineq} implies the second inequality. Then, reasoning inductively, we obtain the result.
\end{proof}
Based on the ResidueHat estimators, we introduce an empirical version of the Demix algorithm, namely, DemixHat (see Algorithm \ref{demix_hat_alg}). The only substantial difference is that we replace the Residue function with the ResidueHat function. See the Appendix for details.

We now state our main estimation result. 
\begin{thm}
\label{demix_estim}
Let $\epsilon > 0$. Suppose that $P_1, \ldots, P_L$ satisfy \textbf{(A$''$)} and $\vec{\Pi}$ satisfies \textbf{(B$'$)}. Then, with probability tending to $1$ as $\vec{n} \longrightarrow \infty$, DemixHat($\ed{\tilde{P}}_1, \ldots, \ed{\tilde{P}}_L$) returns $(\est{Q}_1,\ldots, \est{Q}_L)$ and there exists a permutation $\sigma: [L] \longrightarrow [L]$ such that for every $i \in [L]$,
\begin{equation*}
\sup_{S \in \sS} |\est{Q}_i(S) - P_{\sigma(i)}(S)| < \epsilon.
\end{equation*}
\end{thm}
In the preceding, we have assumed a fixed VC class to simplify the presentation. However, these results easily extend to the setting where $\sS = \sS_k$ and $k \longrightarrow \infty$ at a suitable rate depending on the growth of the VC dimensions $V_k$. This allows for the $P_i$s to be estimated uniformly on arbitrarily complex events, e.g., $\sS_k$ is the set of unions of $k$ open balls.

\section{Discussion}

In the Appendix, we present a finite sample algorithm for the decontamination of a partial label model (see Algorithm \ref{partial_label_hat}). This algorithm is based on a different approach from the PartialLabel Algorithm \ref{partial_label_alg}: it combines DemixHat with an empirical version of the VertexTest algorithm (see Algorithm \ref{vertex_hat_test}). The reason for this hinges in the advantages and disadvantages associated with the two-sample $\kappa^*$ operator and multi-sample $\kappa^*$ operator, respectively. Algorithms that only use the two-sample $\kappa^*$ operator have the following two advantages: (i) the geometry of the 2-sample $\kappa^*$ operator is simpler than the geometry of the multi-sample $\kappa^*$ operator and, as such, can be more tractable. Indeed, in recent years, several practical algorithms for estimating the two-sample $\kappa*$ have been developed (see \citet{jain2016} and references therein). (ii) We have estimators with established rates of convergence for the two-sample $\kappa^*$ operator, but not for the multi-sample $\kappa^*$ operator. On the other hand, algorithms that use the multi-sample $\kappa^*$ operator can be simpler and have the potential to be more practical since they can reduce the number of estimation steps.

\pagebreak
\section{Appendix}

\begin{figure}[h]
\centering
\setlength{\textfloatsep}{5pt}
\vspace{-2em}
\includegraphics[scale = .45]{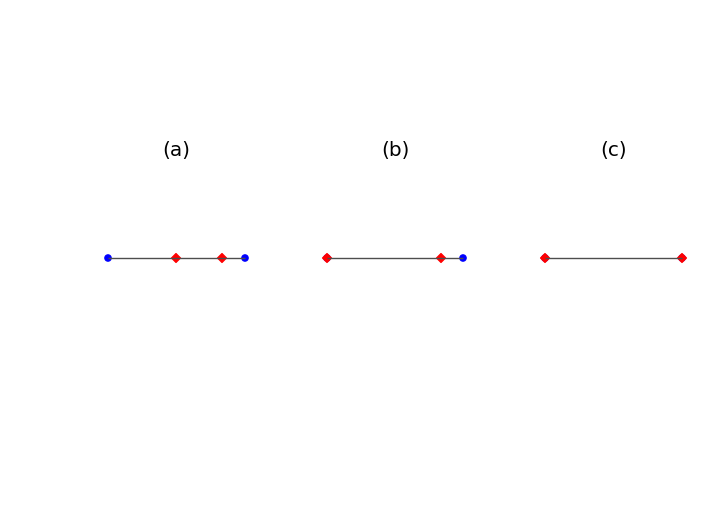}
\vspace{-6em}
\caption{In (a), we are given a demixing problem where $L=2$. The diamonds represent the mixture proportions of $\p{1}$ and $\p{2}$. The circles represent the base distributions. In (b), the residue of a contaminated distribution wrt the other contaminated distribution is computed (line 3), yielding a base distribution. In (c), the residue is computed again switching the roles of the contaminated distributions (line 4); this yields the remaining base distribution.}
\label{2_case}
\end{figure}
\begin{figure}[h]
\centering
\setlength{\textfloatsep}{5pt}
\includegraphics[scale = .45]{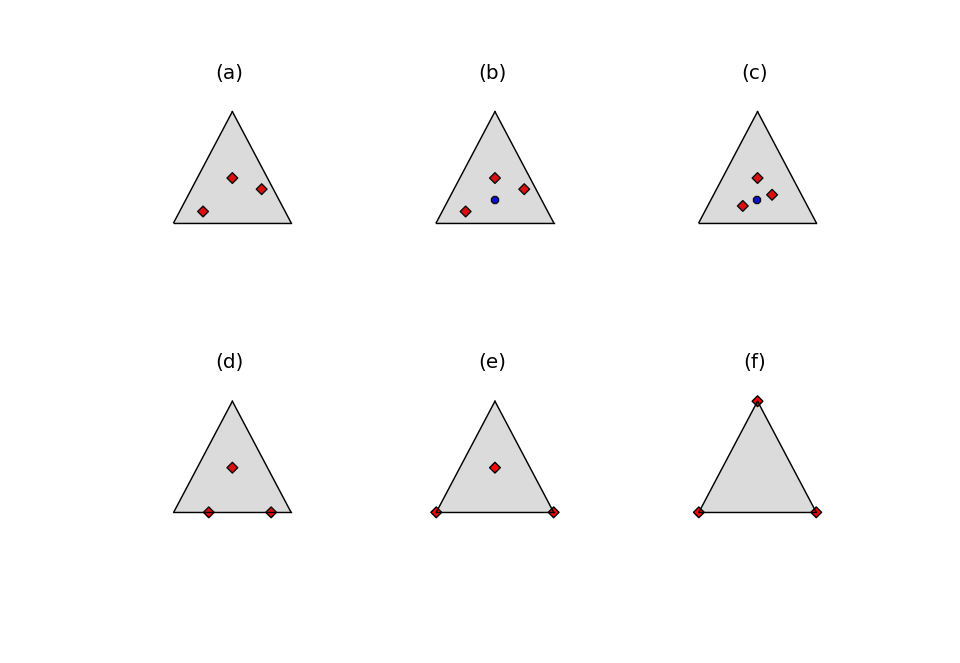}
\vspace{-4em}
\caption{In (a), we are given a demixing problem with $M=L=3$. The diamonds represent the mixture proportions of $\p{1},\p{2}$ and $\p{3}$. In (b), the blue circle is a random distribution chosen in the convex hull of two of the distributions (line 7). In (c), two of the distributions are resampled so that their residues wrt the other distribution are on the same face of the simplex (lines 12-15). In (d), these particular residues are computed (lines 12-15). In (e), two of the distributions are demixed (lines 3-5). In (f), the residue of the final distribution wrt the final two demixed distributions is computed to obtain the final demixing (line 18-21).}
\label{3_case}
\end{figure}

\begin{figure}[h]
\centering
\includegraphics[scale = .5]{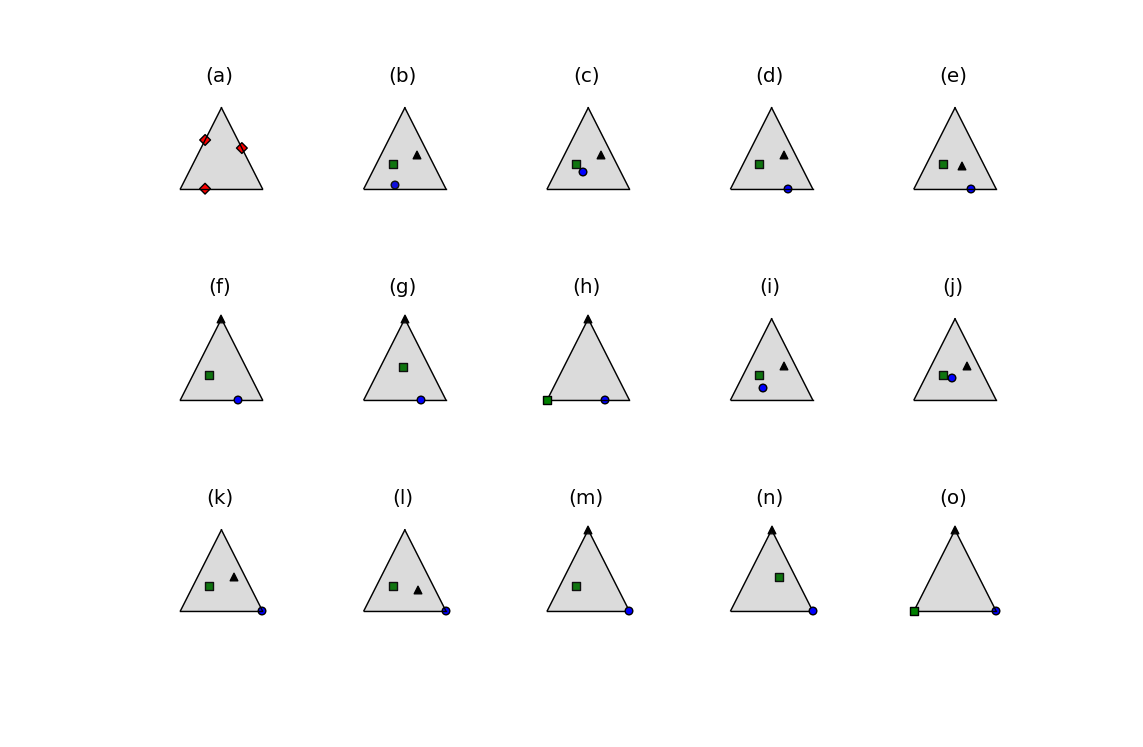}
\caption{We are given an instance of a partial label problem where $M=L=3$ and the partial labels each have two $1$s and a single $0$. In (a), the red diamonds represent the mixture proportions of the distributions $\p{1}, \p{2}, \p{3}$. In (b), three distributions $Q_1, Q_2, Q_3$ are sampled uniformly randomly from the convex hull of $\p{1}, \p{2}, \p{3}$; the green square, black triangle, and blue circle represent their mixture proportions. Figures (c)-(h) show how the algorithm generates a set of candidate distributions $(W^{(2)}_1,W^{(2)}_2,W^{(2)}_3)^T$ with $k=2$. In (h), PartialLabel runs VertexTest on $(W^{(2)}_1,W^{(2)}_2,W^{(2)}_3)^T$ and determines that $(W^{(2)}_1,W^{(2)}_2,W^{(2)}_3)^T$ is not a permutation of $(P_1, P_2,P_3)^T$. In (i)-(o), PartialLabel begins again with $Q_1, Q_2, Q_3$ and executes the same series of steps with $k=3$, generating $(W^{(3)}_1,W^{(3)}_2,W^{(3)}_3)^T$. In (o), it runs VertexTest on $(W^{(3)}_1,W^{(3)}_2,W^{(3)}_3)^T$ and determines that $(W^{(3)}_1,W^{(3)}_2,W^{(3)}_3)^T$ is a permutation of $(P_1, P_2, P_3)^T$.}
\label{partial_label_figure}
\end{figure}

\subsection{Notation}
Let $A$ be a set. $A^\circ$ denotes the relative interior of $A$, i.e., $A^\circ = \set{x \in A | B(x,r) \cap \aff A \subseteq A \text{ for some } r > 0}$. Then, $\partial A$ denotes the relative boundary of $A$, i.e., $\partial A = A \setminus A^\circ$. In addition, let $\norm{\cdot}$ denote an arbitrary norm. For two vectors, $\vec{x}, \vec{y} \in \bbR^K$, define
\begin{align*}
\min(\vec{x}^T, \vec{y}^T) & = (\min(x_1, y_1), \ldots, \min(x_K, y_K))
\end{align*}

We use the following affine mapping throughout the paper: $m_\nu(\vec{x}, \vec{y}) = (1-\nu) \vec{x} + \nu\vec{y}$ where $\vec{x}, \vec{y} \in \bbR^L$ and $\nu \in [0,1]$. Overloading notation, when $Q_1$ and $Q_2$ are distributions, we define $m_\nu(Q_1, Q_2) = (1-\nu)Q_1 + \nu Q_2$. We call $\nu$ the \emph{resampling proportion}. Note that if $\vec{\eta}_1, \vec{\eta}_2 \in \si_L$ and $Q_1 = \vec{\eta}_1^T \vec{P}$ and $Q_2 = \vec{\eta}_2^T \vec{P}$, then $m_\nu(\vec{\eta}_1, \vec{\eta}_2)$ is the mixture proportion for the distribution $m_\nu(Q_1, Q_2)$.

\subsection{Necessary Conditions}
\subsubsection{Mutual Contamination Models}
\begin{proof}[Theorem \ref{nec_cond}]
\begin{description}
\item[(A)] We prove the contrapositive. Let $\phi$ be a demixing operator and $\phi(\tilde{\vec{P}}) = (\vec{\Pi}, \vec{P})$. Suppose that there is some $P_i$ and $Q \in \co(\set{P_j : j \neq i})$ with $Q = \sum_{j \neq i} \beta_j P_j$ such that $P_i$ is not irreducible wrt $Q$. Then, there is some distribution $G$ and $\gamma \in (0,1]$ such that $P_i = \gamma Q + (1-\gamma) G$. 

Suppose $\gamma = 1$. Then, $P_i = Q \in \co(\set{P_k : k \neq i})$. But, then $\p{1}, \ldots, \p{M} \in \co(\set{P_j : j \neq i} \cup \set{R})$ for any distribution $R \not \in \co(P_1, \ldots, P_L)$. This shows that $\phi$ does not satisfy \textbf{(M)}.  

Therefore, assume that $\gamma \in (0,1)$. Either $G \in \co(P_1, \ldots, P_L)$ or $G \not \in \co(P_1, \ldots, P_L)$. Suppose that $G \in \co(P_1, \ldots, P_L)$. Then, there exist $\alpha_1, \ldots, \alpha_L$ all nonnegative and summing to $1$ such that 
\begin{align*}
P_i & = \gamma Q + (1-\gamma)(\alpha_1 P_1 + \ldots + \alpha_L P_L)
\end{align*}
\noindent Therefore, $P_i \in \co(\set{P_k : k \neq i})$. But then $\p{1}, \ldots, \p{M} \in \co(\set{P_j : j \neq i} \cup \set{R})$ for any distribution $R \not \in \co(P_1, \ldots, P_L)$. This shows that that $\phi$ does not satisfy \textbf{(M)}.

Now, suppose that $G \not \in \co(P_1, \ldots, P_L)$. Since $P_i \in \co(G, Q)$ and $Q \in \co(\set{P_j : j \neq i})$, we have that $\co(\set{P_j : j \neq i} \cup \set{G}) \supset \co(P_1, \ldots, P_L)$. Then, $\p{1}, \ldots, \p{M} \in \co(\set{P_j : j \neq i} \cup \set{G})$. This shows that $\phi$ does not satisfy \textbf{(M)}. The result follows.

\item[(B)] Let $\phi(\tilde{\vec{P}}) = (\vec{\Pi}, \vec{P})$. Clearly, $\dim \spa (\p{1}, \ldots , \p{M}) \leq \dim \spa(P_1, \ldots, P_L)$ since $\p{i} = \vec{\pi}_i^T \vec{P}$ for all $i \in [M]$. Since $\phi$ satisfies \textbf{(L)}, $\spa (P_1, \ldots, P_L) \subset \spa (\p{1}, \ldots, \p{L})$, which implies that $\dim \spa(P_1, \ldots, P_L) \leq \dim \spa (\p{1}, \ldots , \p{M})$. Therefore, $\dim \spa (\p{1}, \ldots , \p{M}) = \dim \spa(P_1, \ldots, P_L)$. Then, since $\p{1}, \ldots, \p{M} \in \range(\vec{\Pi})$, $\dim \spa (P_1, \ldots , P_L) \leq \dim \range \vec{\Pi}$. By Result 3.117 of \citet{axler}, $\rank (\vec{\Pi}) = \dim \text{range} \, \vec{\Pi} \geq \dim \spa(P_1, \ldots, P_L)$.
\end{description}
\end{proof}
\begin{proof}[Proposition \ref{max_ji_consistency}]
Suppose that $\phi(\tilde{\vec{P}}) = (\vec{\Pi}, \vec{P})$. We first show that $\phi$ satisfies \textbf{(L)}. By hypothesis, $P_1, \ldots, P_L$ are jointly irreducible. By Lemma B.1 of \citet{blanchard2014}, $P_1, \ldots, P_L$ are linearly independent. Since by hypothesis $\vec{\Pi}$ has full rank, there exist $L$ rows in $\vec{\Pi}$, $\vec{\pi}_{i_1}, \ldots, \vec{\pi}_{i_L}$, that are linearly independent. By Lemma B.1 of \citet{blanchard2014}, $\p{{i_1}}, \ldots, \p{{i_L}}$ are linearly independent. Since $\vec{\Pi} \vec{P} = \tilde{\vec{P}}$, $\spa(\p{1}, \ldots, \p{M}) \subseteq \spa(P_1, \ldots, P_L) $. But, then since $\dim \spa(\p{1}, \ldots, \p{M}) \geq L$, we have $\spa(\p{1}, \ldots, \p{M}) = \spa(P_1, \ldots, P_L)$. Therefore, $\phi$ satisfies \textbf{(L)}.

Now, we show that $\phi$ satisfies \textbf{(M)}. Suppose that there is another solution $(\vec{\Pi}^\prime, \vec{P}^\prime)$ with $\vec{P}^\prime = (P_1, \ldots, P_L)^T$ such that $\vec{\Pi}^\prime \vec{P}^\prime = \tilde{\vec{P}}$ and with some $P^\prime_i$ such that $P_i^\prime \not \in \co(P_1, \ldots, P_L)$. We claim that $P_i^\prime \not \in \spa(P_1, \ldots, P_L)$. If $P_i^\prime \in \spa(P_1, \ldots, P_L)$, then we must have that $P_i^\prime = \sum_{i=1}^L a_i P_i$ where at least one of the $a_i$ is negative. But, by joint irreducibility of $P_1, \ldots, P_L$, $P_i^\prime$ is not a distribution, which is a contradiction. So, the claim follows. But, then, since $\spa(P_1, \ldots, P_L) = \spa(\p{1}, \ldots, \p{M})$, we must have that $\spa(\p{1}, \ldots, \p{M}) \subseteq \spa(P_1^\prime, \ldots, P_{i-1}^\prime, P_{i+1}^\prime, \ldots, P_L)$, which is impossible since $\dim \spa(\p{1}, \ldots, \p{M}) = L$.
\end{proof}
\subsubsection{Partial Label Model}

\begin{defn}
A decontamination operator $\psi$ satisfies Maximality \textbf{(M$'$)} iff $\psi(\tilde{\vec{P}}, \vec{S}) = (\vec{\Pi}, \vec{P})$ implies that any pair $(\vec{\Pi}^\prime, \vec{P}^\prime)$ with $\vec{\Pi}^\prime \in \Delta_L^M$ and $\vec{P}^\prime = (P^\prime_1,\ldots, P^\prime_L)^T \in \mathcal{P}^L$ that solves (\ref{model}) and is consistent with $\vec{S}$ is such that $\set{P^\prime_1,\ldots, P^\prime_L} \subseteq \co(P_1,\ldots, P_L)$.
\end{defn} 
The notion of linearity for a decontamination operator for the partial label model is identical to the notion of linearity for a decontamination operator for a mixed membership model.

\begin{defn}
A decontamination operator $\psi$ satisfies Linearity \textbf{(L$'$)} iff $\psi_2(\tilde{\vec{P}}, \vec{S}) = \vec{P}$ implies $\set{P_1, \ldots, P_L} \subseteq \spa(\p{1}, \ldots, \p{M})$.
\end{defn}

\begin{prop}
Let $\psi$ denote a decontamination operator for the partial label model and $\psi(\tilde{\vec{P}}, \vec{S}) = (\vec{\Pi}, \vec{P})$. (i) If $\psi$ satisfies \textbf{(L$'$)}, then $(\vec{\Pi}, \vec{P})$ need not satisfy \textbf{(A)}. (ii) If $\psi$ satisfies \textbf{(L$'$)}, then $(\vec{\Pi}, \vec{P})$ satisfies \textbf{(B)}.
\end{prop}

\begin{proof}
\begin{description}

\item[(i)] Let $Q_1 \sim \unif(0,2)$ and $Q_2 \sim \unif(1,3)$. Let $P_1 = \frac{2}{3} Q_1 + \frac{1}{3} Q_2$, $P_2 = \frac{1}{3} Q_1 + \frac{2}{3} Q_2$, $\p{1} = P_1$, and $\p{2} = P_2$. Then, $\vec{S} = \vec{I}_2$, identity matrix. Then, any $(\vec{\Pi}^\prime, \vec{P}^\prime)$ that satisfies (\ref{model}) and is consistent with $\vec{S}$ must be such that $(P_1, P_2)^T = \vec{P}^\prime$. Therefore, \textbf{(M$'$)} is satisfied. But, clearly, \textbf{(A)} is not satisfied.

\item[(ii)] The proof is identical to the proof that \textbf{(L)} implies \textbf{(B)} for mixed membership models.
\end{description}
\end{proof}

\begin{proof}[Proposition \ref{partial_necessity}]
We give a proof by contraposition. Suppose that there exists $i \neq j$ such that $\vec{S}_{:, i} = \vec{S}_{:,j}$. Without loss of generality, let $i = 1$ and $j= 2$. Suppose that $(\vec{\Pi}, \vec{P})$ is consistent with $\vec{S}$ and solves $\tilde{\vec{P}} = \vec{\Pi} \vec{P}$. Then, the pair $(\vec{\Pi}^\prime, \vec{P}^\prime)$
\begin{align*}
\vec{\Pi}^\prime = & \begin{pmatrix}
\vec{\Pi}_{:,2}  & \vec{\Pi}_{:,1} & \vec{\Pi}_{:,3} & \ldots & \vec{\Pi}_{:,L}
\end{pmatrix} \\
\vec{P}^\prime & = \begin{pmatrix}
P_2 \\
P_1 \\
P_3 \\
\vdots \\
P_L \\
\end{pmatrix}
\end{align*}
solves $\tilde{\vec{P}} = \vec{\Pi}^\prime \vec{P}^\prime$ and is consistent with $\vec{S}$. 
\end{proof}

\begin{prop}
If a decontamination operator $\psi$ is such that $\psi((\tilde{\vec{P}}, \vec{S})) = (\vec{\Pi}, \vec{P})$ only if $(\vec{\Pi}, \vec{P})$ satisfy \textbf{(A$'$)} and \textbf{(B$'$)}, then $\phi$ satisfies \textbf{(L)}. 
\end{prop}

\begin{proof}
The proof is identical to the first paragraph of the proof of Proposition \ref{max_ji_consistency}.
\end{proof}

\subsection{Residue Algorithms}

\begin{algorithm}
\caption{Residue($F_0 \, | \, F_1$)}
\begin{algorithmic}[1]
\label{two_residue}
\STATE $\kappa \longleftarrow \kappa^*(F_0 \, | \, F_1)$
\RETURN $\frac{F_0 - \kappa(1 - F_1)}{1 - \kappa}$
\end{algorithmic}
\end{algorithm}

\begin{algorithm}
\caption{Residue($F_0 \, | \, \set{F_1, \ldots, F_K}$)}
\begin{algorithmic}[1]
\label{multi_residue}
\STATE 	$(\nu_1, \ldots, \nu_K) \longleftarrow (\nu^\prime_1, \ldots, \nu^\prime_K) \text{ achieving the maximum in } \kappa^*(F_0 \, | \, F_1, \ldots, F_K)$\\
\RETURN $\frac{F_0 - \sum_{i=1}^K \nu_i F_i}{1 - \sum_{i=1}^K \nu_i}$
\end{algorithmic}
\end{algorithm}

\subsection{General Lemmas}
\begin{proof}[Prop \ref{equiv_opt}]
Without loss of generality, suppose $l =1$ and let $A = [L] \setminus \set{1}$; the other cases follow by a similar argument. Suppose $G$ is such that 
\begin{align*}
Q_1 & = (1 - \sum_{j \geq 2} \mu_j) G + \sum_{j \geq 2} \mu_j Q_j
\end{align*}
\sloppy Note that since $\vec{\eta}_1, \ldots, \vec{\eta}_L$ are linearly independent and $P_1, \ldots, P_L$ are jointly irreducible, $Q_1, \ldots, Q_L$ are linearly independent by Lemma B.1 from \cite{blanchard2014}. Therefore, $\sum_{j \geq 2} \mu_j < 1$ because, if not, $Q_1 = \sum_{j \geq 2} \mu_j Q_j$. Therefore, $\kappa^*(Q_1 \, | \, \set{Q_j : j \neq 1}) < 1$.

Further, any $G$ that satisfies the above equation has the form $\sum_{i=1}^L \gamma_i P_i$. The $\gamma_i$ must sum to one, and we have that they are nonnegative by joint irreducibility. That is, $\vec{\gamma} \equiv \begin{pmatrix}
\gamma_1, \ldots, \gamma_L
\end{pmatrix}^T$
is a discrete distribution. Then, the above equation is equivalent to 
\begin{align*}
\vec{\eta}^T_1 \vec{P} & = (1- \sum_{j \geq 2} \mu_j) \vec{\gamma}^T \vec{P} + \sum_{j \geq 2} \mu_j \vec{\eta}_j^T \vec{P}
\end{align*}
Since $P_1, \ldots, P_L$ are jointly irreducible, $P_1, \ldots, P_L$ are linearly independent by Lemma B.1 \cite{blanchard2014}. By linear independence of $P_1, \ldots, P_L$, we obtain 
\begin{align*}
\vec{\eta}_1 & = (1- \sum_{j \geq 2} \mu_j) \vec{\gamma} + \sum_{j \geq 2} \mu_j \vec{\eta}_j
\end{align*}
\noindent Consequently, $\kappa^*(Q_1 \, | \, \set{Q_j : j \neq 1}) = \kappa^*(\vec{\eta}_1 \, | \,\set{\vec{\eta}_j : j \neq 1}) < 1$ and there is a one-to-one correspondence between the feasible solution in $\kappa^*(Q_1 \, | \, \set{Q_j : j \neq 1})$ and the feasible solution in $\kappa^*(\vec{\eta}_1 \, | \, \set{\vec{\eta}_j : j \neq 1})$. The one-to-one correspondence is given by $G = \vec{\gamma}^T \vec{P}$. 

There is at least one point attaining the maximum in the optimization problem $\kappa^*(\vec{\eta}_1 \, | \, \set{\vec{\eta}_j : j \neq 1})$ by Lemma A.1 in \citet{blanchard2014}.

We see that the maximizing $\mu_j$ are unique as follows. Suppose
\begin{align*}
Q_1 = (1- \kappa^*) G + \sum_{j \geq 2} \mu_j Q_j  = (1- \kappa^*) G + \sum_{j \geq 2} \mu_j^\prime Q_j
\end{align*}
\noindent The linear independence of $Q_1, \ldots, Q_L$ implies that $\mu_j = \mu_j^\prime$. 
\end{proof}
Lemma \ref{facts} gives us some useful properties of the two-sample $\kappa^*$ that we exploit in the PartialLabel and Demix algorithms. Statement \emph{1} gives an alternative form of $\kappa^*$. Statement \emph{2} gives the intuitive result that the residues lie on the boundary of the simplex. Statement \emph{3} gives a useful relation for determining whether two mixture proportions are on the same face; we use this relation extensively in our algorithms.
\begin{lemma}
\label{facts}
Let $F_1, \ldots, F_K$ be jointly irreducible distributions with $\vec{F} = (F_1, \ldots, F_K)^T$, $Q_1, Q_2$ be two distributions such that $Q_i = \vec{\eta}_i^T \vec{F}$ where $\vec{\eta}_i \in \si_K$ for $i =1,2$ and $\vec{\eta}_1 \neq \vec{\eta}_2$. Let $R$ be the residue of $Q_1$ wrt $Q_2$ and $R = \vec{\mu}^T \vec{F}$.
\begin{enumerate}
\item We have the following equivalence of optimization problems:
\begin{align*}
\kappa^*(Q_1 \, | \, Q_2) & = \max(\alpha \geq 1 |  \exists \text{ a distribution } G \text{ s.t } G = Q_2 + \alpha(Q_1 - Q_2))
\end{align*}
\item $ \vec{\mu} \in \partial \si_K$. 

\item $\sS(\vec{\eta}_2) \not \subseteq \sS(\vec{\eta}_1)$ if and only if $R=Q_1$ if and only if $\kappa^*(Q_1 \, | \, Q_2) = 0$.
\end{enumerate}
\end{lemma}
\begin{proof}[Lemma \ref{facts}]
\begin{enumerate}
\item Consider the linear relation: $Q_1 = (1-\kappa) G + \kappa Q_2$ where $\kappa \in [0,1]$. Since $F_1, \ldots, F_K$ are jointly irreducible and $\vec{\eta}_1$ and $\vec{\eta}_2$ are linearly independent, $Q_1$ and $Q_2$ are linearly independent by Lemma B.1 of \citet{blanchard2014}. Therefore, $\kappa < 1$. We can rewrite the relation as follows:
\begin{align*}
G = \frac{1}{1-\kappa}Q_1 - \frac{\kappa}{1-\kappa} Q_2 = \alpha Q_1 + (1-\alpha)Q_2
\end{align*}
\noindent where $\alpha = \frac{1}{1-\kappa}$. The equivalence follows.

\item Since $R$ is the residue of $Q_1$ wrt $Q_2$, by Proposition \ref{equiv_opt}, $\vec{\mu}$ is the residue of $\vec{\eta}_1$ wrt $\vec{\eta}_2$ and $\vec{\mu} \in \si_K$. Therefore, by statement \emph{1} in Lemma \ref{facts}, $\vec{\mu}$ is such that $\alpha^*$ is maximized subject to the following constraints:
\begin{align*}
\vec{\mu} & = (1- \alpha^*)\vec{\eta}_2 + \alpha^*\vec{\eta}_1 \\
\alpha^* & \geq 1 \\
\vec{\mu} & \in \si_K
\end{align*}
\noindent Suppose that $\min_{i} \mu_i > 0$. Then, there is some $\epsilon > 0$ such that 
\begin{align*}
\vec{\mu}^\prime & = (1- \alpha^* - \epsilon)\vec{\eta}_2 + (\alpha^* + \epsilon)\vec{\eta}_1 \\
\alpha^* + \epsilon & \geq 1 \\
\vec{\mu}^\prime & \in \si_K
\end{align*}
\noindent But, this contradicts the definition of $\alpha^*$ and $\vec{\mu}$. Therefore, $\min_{i} \mu_i = 0$. Consequently, $\vec{\mu} \in \partial \si_K$. 

\item By definition of $\kappa^*$, it is clear that $R=Q_1$ if and only if $\kappa^*(Q_1 \, | \, Q_2) = 0$. Therefore, it suffices to show that $\sS(\vec{\eta}_2) \not \subseteq \sS(\vec{\eta}_1)$ if and only if $\kappa^*(Q_1 \, | \, Q_2) = 0$. Suppose $\sS(\vec{\eta}_2) \not \subseteq \sS(\vec{\eta}_1)$. Then, there must be $i \in [K]$ such that $\eta_{2,i} > 0$ and $\eta_{1,i} = 0$. Then, for $\alpha > 1$, 
\begin{align*}
\min_{i \in [K]} (1-\alpha) \eta_{2,i} + \alpha \eta_{1,i} < 0
\end{align*}
\noindent But, this violates the constraint of the optimization problem. Therefore, $\alpha = 1$. By statement \emph{1} in Lemma \ref{facts}, $\kappa^*(\vec{\eta}_1 \, | \, \vec{\eta}_2) = 0$. By Proposition \ref{equiv_opt}, $\kappa^*(Q_1 \, | \, Q_2) = \kappa^*(\vec{\eta}_1 \, | \, \vec{\eta}_2) = 0$.

Now, suppose $\sS(\vec{\eta}_2) \subseteq \sS(\vec{\eta}_1)$. Then, for any $i \in [K]$, if $\eta_{2,i} > 0$, then $\eta_{1,i} > 0$. Then, there is $\alpha > 1$ sufficiently close to $1$ such that
\begin{align*}
\min_{i \in [K]} \eta_{2,i} + \alpha (\eta_{1,i} - \eta_{2,i}) \geq 0
\end{align*}
\noindent By statement \emph{1} in this Lemma, $\kappa^*(\vec{\eta}_1 \, | \, \vec{\eta}_2) > 0$. By proposition \ref{equiv_opt}, $\kappa^*(Q_1 \, | \, Q_2) = \kappa^*( \vec{\eta}_1 \, | \, \vec{\eta}_2) > 0$.
\end{enumerate}
\end{proof}

\subsection{Demixing Mixed Membership Models}
\subsubsection{Lemmas}
Lemma \ref{continuity} establishes an intuitive continuity property of the two-sample version of $\kappa^*$ and the residue.
\begin{lemma}
\label{continuity}
Let $\vec{\eta}_1, \vec{\eta}_2, \vec{\eta}_3 \in \si_L$ be distinct vectors and let $\vec{\mu}$ be the residue of $\vec{\eta}_2$ wrt $\vec{\eta}_1$. Let $\alpha_n$ be a sequence such that $\alpha_n \longrightarrow 1$ as $n \longrightarrow \infty$ and let $\vec{\tau}_n$ be the residue of $\alpha_n \vec{\eta}_2 + (1-\alpha_n) \vec{\eta}_3$ wrt $\vec{\eta}_1$. Then,
\begin{enumerate}
\item $\lim_{n \longrightarrow \infty} \kappa^*( \alpha_n \vec{\eta}_2 + (1-\alpha_n) \vec{\eta}_3 \, | \, \vec{\eta}_1) = \kappa^*(\vec{\eta}_2 \, | \, \vec{\eta}_1)$

\item $\lim_{n \longrightarrow \infty} \norm{\vec{\tau}_n - \vec{\mu}} = 0$
\end{enumerate}
\end{lemma}
\begin{proof}[Lemma \ref{continuity}]

\begin{enumerate}
\item \sloppy In order to apply the residue operator $\kappa^*$ to $\vec{\eta}_1, \vec{\eta}_2, \vec{\eta}_3$, we think of $\vec{\eta}_1,\vec{\eta}_2, \vec{\eta}_3$ as discrete probability distributions. Let $( \mathcal{X}, \sC)$ be the measurable space on which $\vec{\eta}_1, \vec{\eta}_2, \vec{\eta}_3$ are defined. Note that $\mathcal{X}$ is finite and let $\mathcal{C} = 2^\mathcal{X}$ (i.e., the power set of $\mathcal{X}$), which is also finite. By Proposition 1 of \citet{blanchard2014}, 
\begin{align*}
\kappa^*(\vec{\eta}_2 \, | \, \vec{\eta}_1) & = \inf_{C \in \mathcal{C}, \vec{\eta}_1(C) > 0} \frac{\vec{\eta}_2(C)}{\vec{\eta}_1(C)}
\end{align*}
Since $\mathcal{C}$ is finite, there is $\epsilon > 0$ such that $\inf_{C \in \mathcal{C}, \vec{\eta}_1(C) > 0} \vec{\eta}_1(C) > \epsilon$. Then,
\begin{align*}
\kappa^*(\alpha_n \vec{\eta}_2 + (1-\alpha_n) \vec{\eta}_3 \, | \, \vec{\eta}_1) & = \inf_{C \in \mathcal{C}, \vec{\eta}_1(C) > 0} \frac{\alpha_n \vec{\eta}_2(C) + (1-\alpha_n) \vec{\eta}_3(C)}{\vec{\eta}_1(C)}  \\
& \leq \alpha_n \inf_{C \in \mathcal{C}, \vec{\eta}_1(C) > 0} \frac{ \vec{\eta}_2(C)}{\vec{\eta}_1(C)} +  \frac{1-\alpha_n}{\epsilon} \\
& \longrightarrow \inf_{C \in \mathcal{C}, \vec{\eta}_1(C) > 0} \frac{ \vec{\eta}_2(C)}{\vec{\eta}_1(C)}
\end{align*}
as $n \longrightarrow \infty$. Further, since $\frac{ \vec{\eta}_2(\cdot)}{\vec{\eta}_1(\cdot)}$ and $\frac{  \vec{\eta}_3(\cdot)}{\vec{\eta}_1(\cdot)}$ are bounded over $\set{C \in \sC : \vec{\eta}_1(C) > 0}$,
\begin{align*}
\kappa^*(\alpha_n \vec{\eta}_2 + (1-\alpha_n) \vec{\eta}_3 \, | \, \vec{\eta}_1) & = \inf_{C \in \mathcal{C}, \vec{\eta}_1(C) > 0} \frac{\alpha_n \vec{\eta}_2(C) + (1-\alpha_n) \vec{\eta}_3(C)}{\vec{\eta}_1(C)} \\ 
& \geq \alpha_n \inf_{C \in \mathcal{C}, \vec{\eta}_1(C) > 0} \frac{ \vec{\eta}_2(C)}{\vec{\eta}_1(C)} + (1-\alpha_n) \inf_{C \in \mathcal{C}, \vec{\eta}_1(C) > 0}\frac{  \vec{\eta}_3(C)}{\vec{\eta}_1(C)} \\
& \longrightarrow \inf_{C \in \mathcal{C}, \vec{\eta}_1(C) > 0} \frac{ \vec{\eta}_2(C)}{\vec{\eta}_1(C)}
\end{align*}
as $n \longrightarrow \infty$. By the sandwich principle of limits, the result follows.

\item \sloppy We write $\vec{\mu} = \kappa \vec{\eta}_2 + (1-\kappa) \vec{\eta}_1$ and $\vec{\tau}_n = \kappa_n (\alpha_n \vec{\eta}_2 + (1-\alpha_n) \vec{\eta}_3) + (1-\kappa_n) \vec{\eta}_1$ where $\kappa = \kappa^*(\vec{\eta}_2 \, | \, \vec{\eta}_1)$ and $\kappa_n = \kappa^*(\alpha_n \vec{\eta}_2 + (1-\alpha_n) \vec{\eta}_3 \, | \, \vec{\eta}_1)$. Then, by the triangle inequality,
\begin{align*}
\norm{\vec{\mu} - \vec{\tau}_n} & = \norm{(\kappa_n - \kappa) \vec{\eta}_1 + (\kappa - \kappa_n \alpha_n) \vec{\eta}_2 - (1-\alpha_n) \vec{\eta}_3} \\
& \leq |\kappa_n - \kappa|\norm{\vec{\eta}_1} + |\kappa - \kappa_n \alpha_n| \norm{ \vec{\eta}_2} + |1-\alpha_n| \norm{\vec{\eta}_3} \\
& \longrightarrow 0
\end{align*}
as $n \longrightarrow \infty$ since $\alpha_n \longrightarrow 1$ and $\kappa_n \longrightarrow \kappa$ by statement 1 in Lemma \ref{continuity}.
\end{enumerate}
\end{proof}
Lemma \ref{pres_lin_ind} guarantees that certain operations in the Demix algorithm preserve linear independence of the mixture proportions.
\begin{lemma}
\label{pres_lin_ind}
Let $\vec{\tau}_1, \ldots, \vec{\tau}_K \in \si_L$ be linearly independent and $P_1,\ldots, P_L$ be jointly irreducible. Let $Q_i = \vec{\tau}_i^T \vec{P}$ for $i \in [K]$. Then for any $i,j \in [K]$ such that $i \neq j$, 
\begin{enumerate}
\item If $\vec{\eta} = \sum_{k=1}^K a_k \vec{\tau}_k$ with $a_j \neq 0$, then $\vec{\tau}_1,\ldots, \vec{\tau}_{j-1}, \vec{\eta}, \vec{\tau}_{j+1}, \ldots, \vec{\tau}_K$ are linearly independent.

\item Let $R_k$ be the residue of $Q_k$ with respect to $Q_j$ for all $k \in [K] \setminus \set{j}$. Then, $R_k = \vec{\eta}^T_k \vec{P}$ where $\vec{\eta}_k \in \si_K$ and $\vec{\eta}_1,\ldots,\vec{\eta}_{j-1},\vec{\tau}_j,\vec{\eta}_{j+1},\ldots,\vec{\eta}_K$ are linearly independent.

\item Let $\vec{\tau}^* \in \co (\vec{\tau}_1,\ldots, \vec{\tau}_k)^\circ$ and $\vec{\eta}_i \in \co (\vec{\tau}_i, \vec{\tau}^*)^\circ$ for $i \in [k]$ where $k \leq K$. Then, 
\begin{align*}
\vec{\eta}_1, \vec{\eta}_2, \ldots, \vec{\eta}_k, \vec{\tau}_{k+1},\ldots, \vec{\tau}_K
\end{align*}
are linearly independent. 

\end{enumerate}

\end{lemma}
\begin{proof}[Lemma \ref{pres_lin_ind}]
\begin{enumerate}
\item Let $\vec{\eta} = \sum_{k=1}^K a_k \vec{\tau}_k$ with $a_j \neq 0$ for some $j \in [K]$.  Let $b_1,\ldots, b_K \in \bbR$ such that 
\begin{align*}
0 & = b_1 \vec{\tau}_1 + \ldots + b_{j-1} \vec{\tau}_{j-1} + b_j \vec{\eta} + b_{j+1} \vec{\tau}_{j+1} + \ldots + b_K \vec{\tau}_K \\
& = (b_1 + b_j a_1) \vec{\tau}_1 + \ldots + (b_{j-1} + b_j a_{j-1}) \vec{\tau}_{j-1} + b_j a_j \vec{\tau}_j + \\
& (b_{j+1} + b_j a_{j+1}) \vec{\tau}_{j+1} + \ldots + (b_K + b_j a_K) \vec{\tau}_K
\end{align*}
Since $\vec{\tau}_1,\ldots,\vec{\tau}_K$ are linearly independent, $b_j a_j = 0$. Since $a_j \neq 0$ by hypothesis, $b_j = 0$. Therefore, the previous equation reduces to
\begin{align*}
b_1 \vec{\tau}_1 + \ldots + b_{j-1} \vec{\tau}_{j-1}  + b_{j+1}  \vec{\tau}_{j+1} + \ldots + b_K \vec{\tau}_K & = 0
\end{align*}
By the linear independence of $\vec{\tau}_1,\ldots,\vec{\tau}_K$, $b_1 = \ldots = b_K = 0$. The result follows.

\item Let $k \in [K] \setminus \set{j}$. By Proposition \ref{equiv_opt}, $\vec{\eta}_k$ is the residue of $\vec{\tau}_k$ with respect to $\vec{\tau}_j$ and $\vec{\eta}_k \in \si_L$. By statement 1 of Lemma \ref{facts}, there is $\alpha_k \geq 1$ such that $\vec{\eta}_k = (1- \alpha_k) \vec{\tau}_j + \alpha \vec{\tau}_k$. By statement 1 of Lemma \ref{pres_lin_ind}, $\vec{\eta}_1, \vec{\tau}_2, \ldots, \vec{\tau}_k$ are linearly independent. Using induction and statement 1 of \ref{pres_lin_ind}, $\vec{\eta}_1, \ldots, \vec{\eta}_{j-1}, \vec{\tau}_j, \vec{\eta}_{j+1}, \ldots, \vec{\eta}_k$ are linearly independent.

\item We show the result inductively. By statement \emph{1} of Lemma \ref{pres_lin_ind}, it follows immediately that $\vec{\eta}_1, \vec{\tau}_2,\ldots, \vec{\tau}_K$ are linearly independent. Now, suppose $\vec{\eta}_1,\ldots, \vec{\eta}_{k-1}, \vec{\tau}_k, \ldots, \vec{\tau}_K$ are linearly independent. Then, there exist unique $\alpha_1, \ldots, \alpha_K$ such that $\vec{\tau}^* = \sum_{i=1}^{k-1} \alpha_i \vec{\eta}_i + \sum_{i=k}^K \alpha_i \vec{\tau}_i$. Further, since $\vec{\eta}_k \in \co(\vec{\tau}_j, \vec{\tau}^*)^\circ$, there exists $\beta \in (0,1)$ such that $\vec{\eta}_k = \beta (\sum_{i=1}^{k-1} \alpha_i \vec{\eta}_i + \sum_{i=k}^K \alpha_i \vec{\tau}_i) + (1- \beta) \vec{\tau}_k$. By statement 1 of Lemma \ref{pres_lin_ind}, it suffices to show that $\beta \alpha_k + (1-\beta) \neq 0$. Suppose to the contrary that $\beta \alpha_k + (1-\beta) = 0$. Then, $\alpha_k = 1- \frac{1}{\beta} < 0$. So, it is enough to show that $\vec{\tau}^* \in \co(\vec{\eta}_1, \ldots, \vec{\eta}_k, \vec{\tau}_{k+1}, \ldots, \vec{\tau}_K)$.

We prove the claim inductively. By Statement 1 of \ref{pres_lin_ind}, $\vec{\eta}_1, \vec{\tau}_2, \ldots, \vec{\tau}_K$ are linearly independent, so there exists unique $c_1, \ldots, c_K$ such that $\vec{\tau}^* = c_1 \vec{\eta}_1 + \sum_{i=2}^K c_i \vec{\tau}_i$. Since $\vec{\eta}_1 \in \co(\vec{\tau}_1, \vec{\tau}^*)$, there exists $\gamma \in (0,1)$ such that $c_1 (\gamma \vec{\tau}_1 + (1- \gamma) \vec{\tau}^*) + \sum_{i=2}^K c_i \vec{\tau}_i = \vec{\tau}^*$. If $c_1 < 0$, then $\vec{\tau}^* \not \in \co(\vec{\tau}_1, \ldots, \vec{\tau}_K)$, which is a contradiction. The claim follows inductively.

Having established the claim, the result follows.
\end{enumerate}
\end{proof}
Lemma \ref{mod_ident_lemma} establishes that the operation performed in the last for loop at the end of the Demix algorithm gives the desired distribution. 
\begin{lemma}
\label{mod_ident_lemma}
Let $P_1, \ldots,P_L$ be jointly irreducible, $(Q_1, \ldots, Q_{K-1})$ be a permutation of $(P_{i_1}, \ldots, P_{i_{K-1}})$ and $Q^1_K \in \co(P_{i_1}, \ldots, P_{i_K})^\circ$. Define the sequence 
\begin{align*}
Q_K^i & \longleftarrow \text{Residue}(Q_K^{i-1} \, | \, Q_{i-1})
\end{align*}
\noindent Then, $Q_K^K = P_{i_K}$.
\end{lemma}
\begin{proof}[Lemma \ref{mod_ident_lemma}]
Relabel the distributions so that $Q_j = P_j$. Let $\vec{\mu}_1$ denote the mixture proportion of $Q_K^1$ and $\vec{e}_j$ the mixture proportion of $P_j$. Consider the sequence given by
\begin{align*}
\vec{\mu}_i & \longleftarrow \text{Residue}(\vec{\mu}_{i-1} \, | \, \vec{e}_{i-1})
\end{align*}
We claim that $\vec{\mu}_i, \vec{e}_1, \ldots, \vec{e}_{K-1}$ are linearly independent and $\vec{\mu}_i$ is the mixture proportion of $\vec{Q}_K^i$. We prove it inductively. The base case is clear by the hypothesis. Suppose that $\vec{\mu}_n, \vec{e}_1, \ldots, \vec{e}_{K-1}$ are linearly independent and $\vec{\mu}_n$ is the mixture proportion of $Q_K^n$. Since $P_1, \ldots, P_L$ are jointly irreducible, by Proposition \ref{equiv_opt}, $\text{Residue}(\vec{\mu}_{n} \, | \, \vec{e}_{n})$ gives the mixture proportion of $Q_2^{n+1}$. By statement \emph{1} of Lemma \ref{pres_lin_ind}, for any $\alpha \geq 1$ $\alpha \vec{\mu}_n + (1-\alpha) \vec{e}_n, \vec{e}_1, \ldots, \vec{e}_{K-1}$ is linearly independent. Therefore, in particular, by statement $1$ of Lemma \ref{facts}, $\text{Residue}(\vec{\mu}_{n} \, | \, \vec{e}_n), \vec{e}_1, \ldots, \vec{e}_{K-1}$ are linearly independent. The claim follows inductively.

It is enough to show that $\vec{\mu}_K = \vec{e}_K$. We claim that for $n \leq K-1$, $\vec{\mu}_{n} \in \co(\vec{e}_{n}, \ldots, \vec{e}_K)^\circ$. Consider the base case. $\vec{\mu}_2 \in \partial \co(\vec{e}_1, \ldots, \vec{e}_K)$ by Statement \emph{2} of Lemma \ref{facts}. Since $\vec{\mu}_2 = \alpha \vec{\mu}_1 + (1-\alpha) \vec{e}_1$ for some $\alpha \geq 1$, $\vec{\mu}_2 \in \co(\vec{e}_2, \ldots, \vec{e}_K)$. Suppose that $\vec{\mu}_2 \in \partial \co(\vec{e}_2, \ldots, \vec{e}_K)$. Without loss of generality, suppose that $\vec{\mu}_2 \in \co(\vec{e}_2, \ldots, \vec{e}_{K-1})$. Then, since $\vec{\mu}_2 = \alpha \vec{\mu}_1 + (1-\alpha) \vec{e}_1$ for some $\alpha \geq 1$, $\vec{\mu}_1 \in \co(\vec{e}_1, \ldots, \vec{e}_{K-1})$. This implies that $\vec{\mu}_1 \in \partial \co(\vec{e}_1, \ldots, \vec{e}_K)$, which is a contradiction. Hence, $\vec{\mu}_2 \in \co(\vec{e}_2, \ldots, \vec{e}_K)^\circ$. Inductively, repeating the same argument gives the claim. Therefore, $\vec{\mu}_{K-1} \in \co(\vec{e}_{K-1}, \vec{e}_K)^\circ$. By considering statement \emph{1} of Lemma \ref{facts}, it is easy to see that $\vec{\mu}_K = \vec{e}_K$. This completes the proof.
\end{proof}
\subsubsection{The FaceTest Algorithm}
We turn to a key subroutine in the Demix algorithm: the FaceTest algorithm (see Algorithm \ref{face_test_alg}). This algorithm tests whether a set of distributions $Q_1,\ldots, Q_K$ are on the interior of the same face. It computes $R_{i,j} \longleftarrow \text{Residue}(Q_i \, | \, Q_j)$ for every pair; it then computes $\kappa^*(Q_i \, | \, R_{i,j})$ determining whether the two distributions are equal. If there is a single pair for which the two distributions are equal, then $Q_1,\ldots, Q_K$ are not on the interior of the same face. Otherwise, they are on the interior of the same face.
\begin{algorithm}
\caption{FaceTest($Q_1,\ldots, Q_K$)}
\begin{algorithmic}[1]
\label{face_test_alg}
\FOR{$i=1,\ldots, K$}
\FOR{$j=1,\ldots, K$}
\IF{$i \neq j$}
\STATE $R_{i,j} \longleftarrow \text{Residue}(Q_i \, | \, Q_j)$
\IF{$\kappa^*(Q_i \, | \, R_{i,j}) = 1$}
\RETURN $0$
\ENDIF
\ENDIF
\ENDFOR
\ENDFOR
\RETURN $1$
\end{algorithmic}
\end{algorithm}

\begin{prop}
\label{face_test}
Let $P_1,\ldots, P_K$ be jointly irreducible and $Q_1, \ldots, Q_K \in \co(P_1, \ldots, P_K)$ be distinct. FaceTest($Q_1, \ldots, Q_K$) returns $1$ if and only if $Q_1,\ldots, Q_K$ lie on the interior of the same face of $\si_K$. 
\end{prop}
\begin{proof}[Proposition \ref{face_test}]
Let $Q_j = \vec{\eta}^T_j \vec{P}$ for $\vec{\eta}_j \in \si_K$ and all $j \in [K]$. Suppose that $Q_1,\ldots, Q_K$ lie on the interior of the same face, i.e., there exists $i \in [K]$ such that $\vec{\eta}_1,\ldots, \vec{\eta}_K \in \co (\set{\vec{e}_j : j \neq i})^\circ$. Then, $\sS(Q_1) = \ldots = \sS(Q_K)$. By statement 3 of Lemma \ref{facts}, $R_{i,j} \neq Q_i$ for all $i,j \in [K]$ such that $i \neq j$. Then, by Proposition \ref{label_noise_bin_case}, $\kappa^*(Q_i \, | \,  R_{i,j}) < 1$. Hence, FaceTest($ Q_1, \ldots, Q_K$) returns $1$.

Suppose that $Q_1,\ldots,Q_K$ do not all lie on the interior of the same face. Then, there exists $Q_i, Q_j$ that do not lie on the interior of the same face. Without loss of generality, suppose that $\sS(Q_j) \not \subseteq S(Q_i)$. Then, by statement 3 of Lemma \ref{facts}, $R_{i,j} = Q_i$. By Proposition \ref{label_noise_bin_case}, $\kappa^*(R_{i,j} \, | \, Q_i) = 1$. Hence, FaceTest($Q_1, \ldots, Q_K$) returns $0$.
\end{proof}
\subsubsection{The Demix Algorithm}
\begin{proof}[Theorem \ref{demix_identification}]
We use the notation from the description of the algorithm only replacing $K$ with $L$ and $S_i$ with $\p{i}$. We prove the result by induction. Suppose $L=2$. Then, $\vec{\pi}_i \in \co(P_1, P_2)$ and $\vec{\pi}_2 \in \co(P_1, P_2)$ and $\vec{\pi}_1 \neq \vec{\pi}_2$ by linear independence of $\vec{\pi}_1$ and $\vec{\pi}_2$. By Lemma 2 of \citet{blanchard2014}, it follows that
\begin{align*}
\set{\vec{e}_1, \vec{e}_2} & = \set{ \vec{\mu} : \vec{\mu} \text{ is the residue of } \vec{\pi}_i \text{ with respect to } \vec{\pi}_j, i,j \in [2], i \neq j}
\end{align*}
\noindent This fact and Proposition \ref{equiv_opt} imply that if $Q_1$ is the residue of $\tilde{P}_1$ with respect to $\p{2}$ and $Q_2$ is the residue of $\p{2}$ with respect to $\p{1}$, then $(Q_1,Q_2)$ is a permutation of $(P_1, P_2)$, completing the base case.

Suppose $L > 2$. With probability $1$, $Q \in \co (\p{2},\ldots, \p{L})^\circ$. We can write $Q = \vec{\eta}^T \vec{P}$ where clearly we have that $\vec{\eta}$ is a uniformly distributed random vector in $\co(\vec{\pi}_2,\ldots,\vec{\pi}_L)$. Let $R$ be the residue of $Q$ with respect to $\p{1}$. By Proposition \ref{equiv_opt}, we can write $R = \vec{\lambda}^T \vec{P}$ where $\vec{\lambda}$ is the residue of $\vec{\eta}$ with respect to $\vec{\pi}_1$. By statement \emph{2} of Lemma \ref{facts}, $\vec{\lambda} \in \partial \si_L$. 
\begin{description}
\item[Step 1:] We claim that with probability $1$, there is $l \in [L]$ such that $\vec{\lambda} \in \co (\set{\vec{e}_j : j \neq l})^\circ$. Let $B_{i,j} = \co (\set{\vec{\pi}_1} \cup \set{\vec{e}_k :k \neq i,j})$ where $i,j \in [L]$ and $i \neq j$ and let $C = \co(\vec{\pi}_2,\ldots, \vec{\pi}_L)$. First, we argue that $C \cap B_{i,j}$ has affine dimension at most $L-3$.\footnote{Note that if $\vec{v}_1,\ldots,\vec{v}_n \in \bbR^L$ are linearly independent and $n \leq L$, then $\aff(\vec{v}_1,\ldots,\vec{v}_n)$ has affine dimension $n-1$.} Since $\vec{\pi}_2,\ldots,\vec{\pi}_L$ are linearly independent, $C$ has affine dimension $L-2$. Since $\set{\vec{e}_k :k \neq i,j}$ are linearly independent, $B_{i,j}$ has affine dimension $L-2$ or $L-3$. If $B_{i,j}$ has affine dimension $L-3$, then $C \cap B_{i,j}$ has affine dimension at most $L-3$. So, suppose that $B_{i,j}$ has affine dimension $L-2$. If $C \cap B_{i,j}$ has affine dimension $L-2$, then $\aff C = \aff B_{i,j}$. Then, in particular, $\vec{\pi}_1 \in \aff C$. But, this contradicts the linear independence of $\vec{\pi}_1,\ldots, \vec{\pi}_L$. Therefore, $C \cap B_{i,j}$ has affine dimension at most $L-3$.

Because $C$ has affine dimension $L-2$ and $\vec{\eta}$ is a uniformly distributed random vector in $C$, with probability $1$, $\vec{\eta} \not \in \cup_{i,j \in [L], i \neq j} B_{i,j}$. Since $\vec{\pi}_1 \in B_{i,j}$ for all $i,j$ and $\vec{\eta} \in \co(\vec{\lambda}, \vec{\pi}_1)$ by definition, the convexity of $B_{i,j}$ implies that $\vec{\lambda} \not \in \cup_{i\neq j} B_{i,j}$. Since $\vec{\lambda} \in \partial \si_L$, the claim follows.

\item[Step 2:] Let $R^{(n)}_i$ be the residue of $m_\frac{n-1}{n}(\p{i}, Q)$ with respect to $\p{1}$. We claim that there is some finite integer $N \geq 2$ such that for all $n \geq N$, 
\begin{align*}
\text{FaceTest}(R^{(n)}_2, \ldots, R^{(n)}_L)
\end{align*}
\noindent returns $1$. Let $m_\frac{n-1}{n}(\p{i}, Q) = \vec{\tau}^{(i)^T}_n \vec{P}$ for $i \in [L] \setminus \set{1}$; note that $\vec{\tau}_{n}^{(i)} = \frac{1}{n} \vec{\pi}_i + \frac{n-1}{n}\vec{\eta}$ and, consequently, $\vec{\tau}^{(i)}_n \in \si_L$. Since $\vec{\eta} \in \co(\vec{\pi}_2,\ldots,\vec{\pi}_L)^\circ$ with probability $1$ and $\vec{\tau}_n^{(i)} \in \co(\vec{\vec{\pi}_i}, \vec{\eta})^\circ$ for $i \in [L] \setminus \set{1}$ for all $n \in \mathbb{N}$ and $\vec{\pi}_1,\ldots, \vec{\pi}_L$ are linearly independent, it follows that for all $n \in \mathbb{N}$ with probability $1$, $\vec{\pi}_1, \vec{\tau}_n^{(2)}, \ldots, \vec{\tau}_n^{(L)}$ are linearly independent by statement \emph{3} in Lemma \ref{pres_lin_ind}. Fix $i \in [L] \setminus \set{1}$. It suffices to show that there is large enough $N$ such that for $n \geq N$, $\text{Residue}(m_\frac{n-1}{n}(\p{i}, Q) \, | \, \p{1})$ is on the same face as $R$ (the other cases are similar). Let $R_i^{(n)} = \vec{\mu}_i^{(n)^T} \vec{P}$; by Proposition \ref{equiv_opt}, $\vec{\mu}_i^{(n)} \in \si_L$ and $\vec{\mu}_i^{(n)}$ is the residue of $\vec{\tau}_n^{(n)}$ with respect to $\vec{\pi}_1$. By Proposition \ref{equiv_opt}, it suffices to show that $\sS( \vec{\mu}_i^{(n)}) = \sS(\vec{\lambda})$. As $n \longrightarrow \infty$, $\vec{\tau}^{(i)}_n = (1 - \frac{n -1}{n}) \vec{\pi}_i + \frac{n-1}{n} \vec{\eta} \longrightarrow \vec{\eta}$, hence by statement \emph{2} in Lemma \ref{continuity}, $\norm{\vec{\mu}_i^{(n)} - \vec{\lambda}} \longrightarrow 0$. Since with probability $1$, $\vec{\lambda} \in \co (\set{\vec{e}_j : j \neq l})^\circ$, it follows that for some large enough $n$, $\vec{\mu}_i^{(n)} \in \co (\set{\vec{e}_j : j \neq l})^\circ$. 

\item[Step 3:] \sloppy The algorithm recurses on $R_2^{(n)},\ldots, R_L^{(n)}$. Since $\vec{\pi}_1, \vec{\tau}_n^{(2)}, \ldots, \vec{\tau}_n^{(L)}$ are linearly independent, it follows by statement \emph{2} in Lemma \ref{pres_lin_ind} that $\vec{\pi}_1, \vec{\mu}_2^{(n)}, \ldots, \vec{\mu}_L^{(n)}$ are linearly independent. Therefore, by the inductive hypothesis, if $(Q_1, \ldots, Q_{L-1}) \longleftarrow \text{Demix}(R^{(n)}_2,\ldots,R^{(n)}_L)$, then $|\set{P_1,\ldots, P_L} \setminus \set{Q_1, \ldots, Q_{L-1}}| = 1$. Suppose without loss of generality that $(Q_1, \ldots, Q_{L-1})$ is a permutation of $(P_1, \ldots, P_{L-1})$. Note that $\frac{1}{L} \sum_{i=1}^L \p{i} \in \co(P_1, \ldots, P_L)^\circ$ since $\vec{\Pi}$ has full rank. Since lines \ref{last_loop_start_demix}-\ref{last_loop_end_demix} generates the same sequence of distributions as the sequence given by Lemma \ref{mod_ident_lemma} and the conditions of Lemma \ref{mod_ident_lemma} are satisfied, at the end of line the given for-loop, $Q_L = P_L$ by Lemma \ref{mod_ident_lemma}. The result follows. 
\end{description}
\end{proof}
\begin{rem}
We could replace lines (18-22) of the Demix algorithm with a single application of the multi-sample version of $\kappa^*$ from \citet{blanchard2014}.
\end{rem}
\subsubsection{The Non-Square Demix Algorithm}
Now, we examine the non-square case of the demixing problem ($M > L$). Note that knowledge of $L$ is essential to this approach since one must resample exactly $L$ distributions in order to run the square Demix algorithm.
\begin{algorithm}
\caption{NonSquareDemix($\p{1}, \ldots, \p{M}$)}
\label{non_square_demix_alg}
\begin{algorithmic}[1]
\STATE $R_1, \ldots, R_L \longleftarrow \text{ independently uniformly distributed elements in} \co (\p{1}, \ldots, \p{M})$
\STATE $(Q_1,\ldots, Q_L) \longleftarrow$ Demix($R_1,\ldots,R_L$)\\
\RETURN $(Q_1,\ldots, Q_L)$
\end{algorithmic}
\end{algorithm}
\begin{cor}
\label{non_square}
Suppose $M > L$. Let $P_1, \ldots, P_L$ be jointly irreducible and $\vec{\Pi}$ satisfy \textbf{(B$'$)}. Then, with probability $1$, NonSquareDemix$(\p{1}, \ldots, \p{M})$ terminates and returns $(Q_1,\ldots, Q_L)$ such that $(Q_1,\ldots, Q_L)$ is a permutation of $(P_1, \ldots, P_L)$.
\end{cor}
\begin{proof}[Corollary \ref{non_square}]
We can write $R_i = \vec{\tau}_i^T \vec{P}$ where $\vec{\tau}_i \in \si_L$ and $i=1,\ldots, L$. $\vec{\tau}_1, \ldots, \vec{\tau}_L$ are drawn uniformly independently from a set with positive $(L-1)$-dimensional Lebesgue measure since $\vec{\Pi}$ has full rank by hypothesis. By Lemma \ref{random_lin_indep}, $\vec{\tau}_1,\ldots, \vec{\tau}_L$ are linearly independent with probability $1$. Then, by Theorem \ref{demix_identification}, with probability 1, Demix($R_1,\ldots,R_L$) terminates and Demix($R_1,\ldots,R_L$) returns a permutation of $(P_1,\ldots, P_L)$.
\end{proof}
\subsection{Decontamination of the Partial Label Model for Classification}
\subsubsection{Lemmas}
Lemma \ref{single} gives a condition on the mixture proportions under which the multi-sample residue is unique. Lemma 2 in \citet{blanchard2014} is very similar and is proved in a very similar way. We give a useful generalization here that reproduces many of the same details. 
\begin{lemma}
\label{single}
Let $l, k \in [L]$. Let $\vec{\tau}_1, \ldots, \vec{\tau}_L \in \si_L$ be linearly independent. We have that condition 1 implies condition 2 and condition 2 implies condition 3. 
\begin{enumerate}

\item There exists a decomposition
\begin{align*}
\vec{\tau}_l & = \kappa\vec{e}_k + (1-\kappa) \vec{\tau}^\prime_l
\end{align*}
where $\kappa > 0$ and $\vec{\tau}^\prime_l \in \co(\set{\vec{\tau}_j : j \neq l})$. Further, for every $\vec{e}_i$ such that $i \neq k$, there exists a decomposition:
\begin{align*}
\vec{e}_i & = \sum_{j=1}^L a_j \vec{\tau}_j
\end{align*}
such that $a_l < \frac{1}{\kappa}$. 

\item Let
\begin{align*}
\vec{T} & = \begin{pmatrix}
\vec{\tau}^T_1 \\
\vdots \\
\vec{\tau}^T_L
\end{pmatrix}
\end{align*}
The matrix $\vec{T}$ is invertible and $\vec{T}^{-1}$ is such that $(\vec{T}^{-1})_{l,k} > 0$ and $(\vec{T}^{-1})_{l,i} \leq 0$ for $i \neq k$ and $(\vec{T}^{-1})_{l,k} > (\vec{T}^{-1})_{j,k}$ for $j \neq l$. In words, the $(l,k)$th entry in $\vec{T}^{-1}$ is positive, every other entry in the $l$th row of $\vec{T}^{-1}$ is nonpositive and every other entry in the $k$th column of $\vec{T}^{-1}$ is strictly less than the $(l,k)$th entry. \footnote{$(\vec{T}^{-1})_{i,j}$ is the $i \times j$ entry in the matrix $\vec{T}^{-1}$.}

\item The residue of $\vec{\tau}_l$ with respect to $\set{\vec{\tau}_j, j \neq l}$ is $\vec{e}_k$. 
\end{enumerate}

\end{lemma}
\begin{proof}[lemma \ref{single}]
Without loss of generality, let $l = 1$ and $k=1$. First, we show that condition \emph{1} implies condition \emph{2}. Suppose that condition \emph{1} holds. Then, there exists $\kappa > 0$ such that 
\begin{align*}
\vec{\tau}_1 & = \kappa \vec{e}_1 + (1-\kappa) \sum_{i=2}^L \mu_i \vec{\tau}_i
\end{align*}
with $\mu_i \geq 0$ for $i \in [L] \setminus \set{1}$. Then,
\begin{align*}
\vec{e}_1 =  \frac{1}{\kappa}(\vec{\tau}_1 - \sum_{i \geq 2} (1-\kappa) \mu_i \vec{\tau}_i)
\end{align*}
Hence, the first row of $\vec{T}^{-1}$ is given by $\frac{1}{\kappa}(1, -(1-\kappa)\mu_2, \cdots, -(1-\kappa)\mu_L)$. This shows that the first row is such that $(\vec{T}^{-1})_{1,1} > 0$ and $(\vec{T}^{-1})_{1,i} \leq 0$ for $i \neq 1$.

Consider $\vec{e}_i$ such that $i \neq 1$. Then, we have the relation: $\vec{e}_i = \sum_{j=1}^L a_j \vec{\tau}_j$, which gives the $i$th row of $\vec{T}^{-1}$. By assumption, $a_1 < \frac{1}{\kappa}$, so the $(i,1)$th entry is strictly less than the $(1,1)$th entry. Hence, \emph{2} follows. 

Now, we prove that condition \emph{2} implies condition \emph{3}. Suppose condition \emph{2} is true. Consider the optimization problem:
\begin{align*}
\max_{\vec{\nu}, \vec{\gamma}} \sum_{i=2}^L \nu_i & \textit{ s.t. } \vec{\tau}_1 = (1 - \sum_{i \geq 2} \nu_i) \vec{\gamma} + \sum_{i=2}^L \nu_i \vec{\tau}_i 
\end{align*}
over $\vec{\gamma} \in \Delta_L$ and $\vec{\nu} = (\nu_2, \cdots, \nu_L) \in C_{L-1} = \set{(\nu_2, \cdots, \nu_L) : \nu_i \geq 0 ; \sum_{i=2}^L \nu \leq 1}$. 

By the same argument given in the proof of Lemma 2 of \citet{blanchard2014}, this optimization problem is equivalent to the program
\begin{align*}
\max_{\vec{\gamma} \in \si_L} \vec{e}_1^T(\vec{T}^T)^{-1} \vec{\gamma} \textit{ s.t. } \vec{\nu}((\vec{T}^T)^{-1} \vec{\gamma}) \in C_{L-1}
\end{align*}
where $\vec{\nu}(\vec{\eta}) = \eta_1^{-1}(-\eta_2, \cdots, -\eta_L)$. The above objective is of the the form $\vec{a}^T \vec{\gamma}$ where $\vec{a}$ is the first column of $\vec{T}^{-1}$. By assumption, for every $i \in [L]$, $\vec{T}^{-1}_{1,1} > \vec{T}^{-1}_{i,1}$. Therefore, the unconstrained maximum over $\vec{\gamma} \in \si_L$ is attained uniquely by $\vec{\gamma} = \vec{e}_1$. Notice that $(\vec{T}^T)^{-1}\vec{e}_1$ is the first row of $\vec{T}^{-1}$. Denote this vector $\vec{b} = (b_1, \cdots, b_L)$. We show that $\vec{\nu}(\vec{b}) = b_1^{-1}(-b_2, \cdots, -b_L) \in C_{L-1}$. By assumption, $\vec{b}$ has its first coordinate positive and the other coordinates are nonpositive. Therefore, all of the components of $\vec{\nu}(\vec{b})$ are nonnegative. Furthermore, the sum of the components of $\vec{\nu}(\vec{b})$ is
\begin{align*}
\sum_{i=2}^L \frac{-b_i}{b_1} & = 1 - \frac{\sum_{i=1}^L b_i}{b_1} = 1 - \frac{1}{b_1}  \leq 1
\end{align*}
The last equality follows because the rows of $\vec{T}^{-1}$ sum to $1$ since $\vec{T}$ is a stochastic matrix. Then, we have $\vec{\nu}((\vec{T}^T)^{-1}\vec{e}_1) \in C_{L-1}$. Consequently, the unique maximum of the optimization problem is attained for $\vec{\gamma} = \vec{e}_1$. This establishes \emph{3}. 
\end{proof}
\begin{algorithm}
\caption{FindSet($\vec{B}$)}
\textbf{Input: } $\vec{B}$ is a $N \times K$ binary matrix
\begin{algorithmic}
\label{find_set_alg}
\STATE $\vec{v} \longleftarrow \vec{B}_{1,:}$
\FOR{$k = 2, \ldots, N$}
\STATE $\vec{v}^\prime \longleftarrow \min(\vec{v}, \vec{B}_{k,:})$
\IF{$\text{sum}(\vec{v}^\prime) \geq 1$}
\STATE $\vec{v} \longleftarrow \vec{v}^\prime$
\ENDIF
\ENDFOR
\RETURN $\vec{v}$
\end{algorithmic}
\end{algorithm}

Lemma \ref{find_set} establishes that if $\vec{S}$ satisfies \textbf{(C)}, then we can find a collection of contaminated distributions $\p{1}, \ldots, \p{L}$ that only have a single base distribution in common; this observation is essential for establishing that VertexTest behaves appropriately. 
\begin{lemma}
\label{find_set}
Let $\vec{B}$ be a $N \times K$ binary matrix with at least one nonzero entry. If no two columns of $\vec{B}$ are identical, then there exists $\set{i_1, \ldots, i_n} \subseteq [N]$ such that $\min(\vec{B}_{i_1,:}, \ldots, \vec{B}_{i_n,:})$ contains a single $1$. 
\end{lemma}
\begin{proof}
The proof is by contraposition. Suppose that there does not exist $\set{i_1, \ldots, i_n} \subseteq [N]$ such that $\min(\vec{B}_{i_1,:}, \ldots, \vec{B}_{i_n,:})$ contains a single $1$. Without loss of generality, reorder the rows of $\vec{B}$ such that a non-zero entry of $\vec{B}$ appears in $\vec{B}_{1,:}$. Let $\vec{v}$ denote the vector returned by $\text{FindSet}(\vec{B})$ (see Algorithm \ref{find_set_alg}). Then, $\vec{v}$ contains at least two $1$s. Let $v_i = 1$ and $v_j = 1$ where $i \neq j$.  Let $\vec{w}^{(k)} \longleftarrow \min(\vec{v}, \vec{B}_{k,:})$. Suppose that there is some $k$ such that $w^{(k)}_i = 1$ and $w^{(k)}_j = 0$. Then, inspection of the Algorithm \ref{find_set_alg} shows that $\vec{v}$ would have to already be such that $v_i = 1$ and $v_j = 0$, which is a contradiction. Hence, for all $k \in [N]$, $w^{(k)}_i = w^{(k)}_j$. Then, $\vec{B}_{:,i} = \vec{B}_{:,j}$.  
\end{proof}
\begin{lemma}
\label{random_lin_indep}
If $\vec{v}_1, \ldots, \vec{v}_k \in \si_L$ are linearly independent and $\vec{w}_{k+1}, \ldots, \vec{w}_L \in \si_L$ are random vectors drawn independently from the $(L-1)$-dimensional Lebesgue measure on a set $A \subset \si_L$ with positive $(L-1)$-dimensional Lebesgue measure, then $\vec{v}_1, \ldots, \vec{v}_k, \vec{w}_{k+1}, \ldots, \vec{w}_L$ are linearly independent with probability $1$.
\end{lemma}
\begin{proof}
Suppose that we sample $\vec{w}_{k+1}, \ldots, \vec{w}_L$ sequentially. Then, $\vec{v}_1, \ldots, \vec{v}_k, \vec{w}_{k+1}$ are linearly independent if $\vec{w}_{k+1} \not \in \spa(\vec{v}_1, \ldots, \vec{v}_k)$. Since $\spa(\vec{v}_1, \ldots, \vec{v}_k) \cap \si_L$ is a $(k-1)$-dimensional simplex in $\si_L$, it has $(L-1)$-dimensional Lebesgue measure $0$. Hence, with probability $1$, $\vec{w}_{k+1} \not \in \spa(\vec{v}_1,\ldots, \vec{v}_k)$. It follows inductively that with probability $1$, $\vec{v}_1, \ldots, \vec{v}_k, \vec{w}_{k+1}, \ldots, \vec{w}_L$ are linearly independent.
\end{proof}
\subsubsection{VertexTest Algorithm}

\begin{algorithm}
\caption{$\text{VertexTest}(\vec{S}, (\p{1}, \ldots, \p{L})^T, (Q_1, \ldots, Q_L)^T)$}
\begin{algorithmic}[1]
\label{vertex_test_alg}
\FOR{$i = 1, \ldots, L$}
\FOR{$j = 1, \ldots, L$}
\IF{$\kappa^*(Q_i \, | \, Q_j) \in (0,1]$}
\RETURN $(0, \vec{I})$
\ENDIF
\ENDFOR
\ENDFOR
\STATE $\vec{C} \longleftarrow \begin{pmatrix}
 1 & \ldots & 1 \\
 \vdots & \ddots & \vdots \\
 1 & \ldots & 1
 \end{pmatrix} \in \bbR^{L \times L}$
\FOR{$i = 1, \ldots, L$}
\FOR{$j = 1, \ldots, L$}
\IF{$\kappa^*(\p{j} \, | \, Q_i) \in (0,1]$}
\STATE $C_{i,:} \longleftarrow \min(C_{i,:}, S_{j,:})$ \label{make_0_not_hat}
\ENDIF
\ENDFOR
\ENDFOR
\FOR{$k = 1 \text{ to } L$} \label{k_loop_C}
\FOR{$i = 1 \text{ to } L$}
\IF{ $\vec{C}_{i,:} == \vec{e}_j^T$ \text{ for some } $j$}
\STATE $\vec{C}_{:,j} \longleftarrow \vec{e}_i$ \label{make_0_not_hat_2nd_loop}
\ENDIF
\ENDFOR
\ENDFOR
\IF{$\vec{C} \text{ is a permutation matrix}$}
\RETURN $(1, \vec{C}^T)$
\ELSE
\RETURN $(0, \vec{C}^T)$
\ENDIF
\end{algorithmic}
\end{algorithm}

To develop some intuition regarding the behavior of VertexTest, consider the following simple example. Suppose that
\begin{align*}
\vec{S} & = \begin{pmatrix}
1 & 1 & 0 \\
1 & 0 & 1 \\
1 & 1 & 1 \\
\end{pmatrix}
\end{align*}
and $Q_1 = P_2$, $Q_2 = P_3$, and $Q_3 = P_1$. Let $\vec{C}^{(k)}$ denote the value of $\vec{C}$ in the $k$th iteration of the loop starting on line \ref{k_loop_C}. Then,
\begin{align*}
\kappa^*( \p{1} \, | Q_1) \in (0,1], \quad & \kappa^*( \p{2} \, | Q_1) = 0, \quad & \kappa^*( \p{3} \, | Q_1) \in (0,1] \\
 \kappa^*( \p{1} \, | Q_2) = 0, \quad & \kappa^*( \p{2} \, | Q_2) \in (0,1], \quad & \kappa^*( \p{3} \, | Q_2) \in (0,1] \\
\kappa^*( \p{1} \, | Q_3) \in (0,1], \quad & \kappa^*( \p{2} \, | Q_3) \in (0,1], \quad & \kappa^*( \p{3} \, | Q_3) \in (0,1] \\
\end{align*}
\begin{align*}
\vec{C}^{(0)} & = \begin{pmatrix}
1 & 1 & 0 \\
1 & 0 & 1 \\
1 & 0 & 0 \\
\end{pmatrix} & \vec{C}^{(1)} = \begin{pmatrix}
0 & 1 & 0 \\
0 & 0 & 1 \\
1 & 0 & 0 \\
\end{pmatrix}
\end{align*}
Once we obtain $\vec{C}^{(0)}$, we know that $Q_3 = P_1$ and, therefore, $Q_2 \neq P_1$ and $Q_1 \neq P_1$. Using this information, we obtain $\vec{C}^{{(1)}^T}$, which is the desired permutation matrix. On the other hand, if $Q_1 = P_2$, $Q_2 = P_3$, and $Q_3 = \frac{1}{2} P_1 + \frac{1}{2} P_3$, then $\kappa^*(Q_3 \, | \, Q_2) \in (0,1)$ and, therefore, VertexCover would return $(0, I)$. 

Lemma \ref{vertex_test} establishes the VertexTest algorithm has the desired behavior. 
\begin{lemma}
\label{vertex_test}
Let $\vec{\eta}_1, \ldots, \vec{\eta}_L \in \si_L$ and $Q_i = \vec{\eta}_i^T \vec{P}$ for $i \in [L]$. Suppose that $P_1, \ldots, P_L$ satisfy \textbf{(A$'$)}, $\vec{\Pi}$ satisfies \textbf{(B$'$)}, and $\vec{S}$ satisfies \textbf{(C)}. Then, $\text{VertexTest}(\vec{S}, (\p{1}, \ldots, \p{L})^T, (Q_1, \ldots, Q_L)^T)$ returns $(1, \vec{C})$  if and only if $(Q_1, \ldots, Q_L)^T$ is a permutation of $\vec{P}$. Further, if   $\text{VertexTest}(\vec{S}, (\p{1}, \ldots, \p{L})^T, (Q_1, \ldots, Q_L)^T)$ returns $(1, \vec{C})$, then $\vec{C} (Q_1, \ldots, Q_L)^T = \vec{P}$.
\end{lemma}
\begin{proof}
\begin{description}
\item[If:] Suppose that $(Q_1, \ldots, Q_L)^T$ is a permutation of $(P_1, \ldots, P_L)^T$. We prove two loop invariants for the loop starting on line \ref{k_loop_C}, which together imply the if-direction. Let $\vec{C}^{(k)}$ denote the matrix $\vec{C}$ in the VertexTest algorithm at line \ref{k_loop_C} after going through the loop starting on line \ref{k_loop_C} $k$ times. 
\begin{description}
\item[Claim 1:] Let $B_i^{(k)} = \set{P_j : C^{(k)}_{i,j} = 1}$. We show that $Q_i \in B_i^{(k)}$ for all $k \in [L]$ and all $i \in [L]$. Consider the base case: $k = 0$. Fix some $i \in [L]$. Let $A = \set{ j : \kappa^*(\p{j} \, | \, Q_i) \in (0,1]}$. By construction, $B^{(0)}_i = \set{P_k : k \in \cap_{j \in A} \sS(\p{j})}$ since
\begin{align*}
P_l \in B_i^{(0)} \Longleftrightarrow & C_{i,l}^{(0)} = 1 \\
\Longleftrightarrow & S_{j,l} = 1 \, \forall j \text{ s.t. } \kappa^*(\p{j} \, | \, Q_i) \in (0,1] \\
\Longleftrightarrow & P_l \in \set{P_k : k \in \cap_{j \in A} \sS(\p{j})}
\end{align*}
where we used the fact that $S_{j,l} = 1$ iff $P_l \in \sS(\p{j})$. By statement 3 of Lemma \ref{facts}, $\kappa^*(\p{j} \, | \, Q_i) \in (0,1]$ implies that $\sS(Q_i) \subseteq \sS(\p{j})$, so that $\sS(Q_i) \subseteq \cap_{j \in A} \sS(\p{j})$. Since $Q_i$ is one of the base distributions, $Q_i \in B^{(0)}_i$. 

\medskip

Now, suppose that $Q_i \in B^{(n-1)}_i$ for all $i \in [L]$; we show that $Q_i \in B^{(n)}_i$ for all $i \in [L]$. Suppose that for all $l$ such that $\vec{C}^{(n-1)}_{l,:} = \vec{e}_{j_l}^T$ for some $j_l$, we had $\vec{C}^{(n-2)}_{l,:} = \vec{e}_{j_l}^T$. Then, the algorithm does nothing in the $n$th iteration, so that $Q_i \in B^{(n-1)}_i = B^{(n)}_i$ for all $i \in [L]$. 

\medskip

Now, suppose that there is some $l$ such that $\vec{C}^{(n-1)}_{l,:} = \vec{e}_{j_l}$ for some $j_l$ and $\vec{C}^{(n-2)}_{l,:} \neq \vec{e}_{j_l}$. By the inductive hypothesis, $Q_l = P_{j_l}$. Then, for all $k \neq l$, we cannot have that $Q_k = P_{j_l}$ because otherwise $(Q_1, \ldots, Q_L)^T$ would not be a permutation of $(P_1, \ldots, P_L)^T$. Therefore, setting $\vec{C}^{(n-1)}_{:,j_l} \longleftarrow \vec{e}_l$ cannot remove $Q_i$ from $B^{(n-1)}_i$ for any $i \in [L]$. It follows that $Q_i \in B^{(n)}_i$ for all $i \in [L]$. The claim follows inductively.

\item[Claim 2:] We claim that $\vec{C}^{(k)}$ has at least $k$ columns that are equal to distinct members of the standard basis $\set{\vec{e}_1, \ldots, \vec{e}_L}$. We give a proof by induction. $k=0$ is trivial. We show $k=1$. Since $\vec{S}$ has distinct columns and at least one nonzero entry, by Lemma \ref{find_set}, there is a set $\set{i_1, \ldots, i_l} \subset [L]$ such that $\min(\vec{S}_{i_1, :}, \ldots, \vec{S}_{i_l, :}) = \vec{e}_i^T$ for some $i$. Since by assumption $(Q_1, \ldots, Q_L)^T$ is a permutation of $(P_1, \ldots, P_L)^T$, there is some $Q_j$ such that $Q_j = P_i$. Note that by definition of $\vec{S}$, $\sS(Q_j) = \set{i} \subseteq \sS(\p{k})$ for all $k \in \set{i_1, \ldots, i_l}$. By statement 3 of Lemma \ref{facts}, $\kappa^*(\p{k} \, | \, Q_j) \in (0,1]$ for all $k \in \set{i_1, \ldots, i_l}$. Note that $vec{C}_{j,i}^{(0)} \neq 0$ by Claim 1; it follows that $\vec{C}_{j,:}^{(0)} = \vec{e}_i^T$. Therefore, in the $1$st iteration of the loop starting on line \ref{k_loop_C}, the condition $\vec{C}_{j,:}^{(0)} = \vec{e}_i^T$ is satisfied and at least one column of $\vec{C}^{(0)}$ is converted into one of the members of the standard basis in line \ref{make_0_not_hat_2nd_loop}. This proves the case $k=1$. 

\medskip

Now, suppose that $n \geq k$ of the columns of $\vec{C}^{(k)}$ are equal to distinct members of the standard basis. If $n > k$, then we are done with the inductive step. Therefore, suppose that $n = k$.  Let $i_1, \ldots, i_k$ denote the column indices of the columns in $\vec{C}^{(k)}$ that are equal to members of the standard basis and let $j_1, \ldots, j_k$ denote the row indices of their corresponding non-zero entries. Let $\vec{S}^\prime$ denote the matrix obtained by deleting the columns $i_1, \ldots, i_k$ from $\vec{S}$. Since $\vec{S}$ does not contain two identical columns, $\vec{S}^\prime$ does not contain two identical columns. We claim that $\vec{S}^\prime$ has at least one non-zero entry. Suppose to the contrary that $\vec{S}^\prime$ has only $0$ entries. Then, by definition of $\vec{S}$, $\set{\p{1}, \ldots, \p{L}} \subset \co(P_{i_1}, \ldots, P_{i_k})$. But, if $k < L$, this contradicts the fact that $\p{1}, \ldots, \p{L}$ are linearly independent (which we have since $P_1,\ldots, P_L$ are jointly irreducible and $\vec{\Pi}$ is full rank).

\medskip

Therefore, by Lemma \ref{find_set}, there is a set of rows $l_1, \ldots, l_t$ of $\vec{S}^\prime$ such that $\vec{e}^T \equiv \min(\vec{S}^\prime_{l_1, :}, \ldots, \vec{S}^\prime_{l_t,:})$ is a vector of zeros with a unique $1$ in one of its entries. Map $\vec{e}^T$ to the $\vec{e}_i^T$ obtained by filling in zeros corresponding to the columns $i_1, \ldots, i_k$ that were previously deleted to obtain $\vec{S}^\prime$ from $\vec{S}$. 

\medskip

By assumption, there is some $Q_j$ such that $Q_j = P_i$. Note that $i \not \in \set{i_1, \ldots, i_k}$ by our construction of $\vec{e}^T_i$. Note that $j \not \in \set{j_1, \ldots, j_k}$ since otherwise for some $j_l \in \set{j_1, \ldots, j_k}$, we would have $P_i = Q_{j_l} = P_{i_l}$ and $P_i \neq P_{i_l}$, which is a contradiction. 

\medskip

Since $\min(\vec{S}_{l_1,:}^\prime, \ldots, \vec{S}_{l_t,:}^\prime)$ has a unique $1$, $\min(\vec{S}_{l_1,:}, \ldots, \vec{S}_{l_t,:})$ can only have $1$s in positions $i_1, \ldots, i_k, i$ and has at least a single $1$ in position $i$. Therefore, by definition of $\vec{S}$ and statement 3 of Lemma \ref{facts}, $\vec{C}^{(0)}_{j,:}$ can only have nonzero entries in indices $i,i_1, \ldots, i_k$. By assumption, $C^{(k)}_{j,i_1}, \ldots, C^{(k)}_{j,i_k}$ are $0$. $\vec{C}^{(k)}_{j,i} \neq 0$ since that would violate Claim $1$.  Therefore, $\vec{C}^{(k)}_{j,:} = \vec{e}^T_i$. Then, in the $(k+1)$th iteration of the loop starting on line \ref{k_loop_C}, in line \ref{make_0_not_hat_2nd_loop}, $\vec{C}^{(k)}_{:,i}$ is converted into one of the standard basis vectors. This establishes the inductive step.
\end{description}
The above two claims imply that once the loop starting on line \ref{k_loop_C} has terminated, $C$ is a permutation matrix such that if $\vec{C}_{:,j} = \vec{e}_i$, then $Q_i = P_j$. Then, $(Q_1, \ldots, Q_L) \vec{C} = (P_1, \ldots, P_L)$. Taking the transpose of both sides, the result follows. 

\item[Only If:] Suppose that $(Q_1, \ldots, Q_L)^T$ is not a permutation of $(P_1, \ldots, P_L)^T$. Suppose to the contrary that VertexTest returns $(1, \vec{C})$ such that $\vec{C}$ is a permutation matrix; we derive a contradiction. Since $(Q_1, \ldots, Q_L)^T$ is not a permutation of $(P_1, \ldots, P_L)^T$ and $\vec{C}^{(L)}$ is a permutation matrix by assumption, there exists $k \in [L]$ and $i \in [L]$ such that $Q_i \not \in \co(B_i^{(k)})$. Fix $k$ to be the smallest $k$ such that there exists $i \in [L]$ such that $Q_i \not \in \co(B_i^{(k)})$.

\medskip

We claim that $k \neq 0$. Let $A = \set{ j : \kappa^*(\p{j} \, | \, Q_i) \in (0,1]}$. By construction, before entering the loop on line \ref{k_loop_C}, $B^{(0)}_i = \set{P_k : k \in \cap_{j \in A} \sS(\p{j})}$. By statement 3 of Lemma \ref{facts}, $\sS(Q_i) \subseteq \cap_{j \in A} \sS(\p{j})$, so that $Q_i \in \co(B_i^{(0)})$. Thus, $k \neq 0$.

\medskip

Let $D = \set{j : C_{i,j}^{(k-1)} = 1, C_{i,j}^{(k)} = 0}$. For each $j \in D$, there exists $l_j \in [L]$ such that $\vec{C}_{l_j,:}^{(k-1)} = \vec{e}_j^T$. Then, since $k$ is the smallest integer such that there exists $i \in [L]$ such that $Q_i \not \in \co(B_i^{(k)})$, for each $j \in D$, there is $l_j$ such that $Q_{l_j} = P_j$. 

\medskip

Fix $i$ such that $Q_i \in \co(B_i^{(k-1)})$ and $Q_i \not \in \co(B_i^{(k)})$. If $Q_i = P_j$ for some $j \in D$, then we are done since $\kappa^*(Q_{l_j} \, | \, Q_i) = 1$, implying that VertexTest would output $(0, I)$--a contradiction. Thus, $Q_i$ is not any of the base distributions. There must be some $E \subset B_i^{(k-1)}$ such that $Q_i \in E^\circ$. Then, we must remove some $P_j$ from $E$ such that $Q_i \not \in \co(E \setminus \set{P_j})$. But, we only remove $P_j$ from $E$ if $j \in D$, which means there is $l_j$ such that $Q_{l_j} = P_j$. Since $\sS(Q_{l_j}) \subset \sS(Q_i)$, $\kappa^*(Q_i \, | \, Q_{l_j}) > 0$, implying that VertexTest would output $(0, I)$--a contradiction.
\end{description}
\end{proof}
\subsubsection{Proof of Theorem \ref{partial_identication}}
\begin{proof}[Theorem \ref{partial_identication}]
We adopt the notation from the description of Algorithm \ref{partial_label_alg} with the exception that we make explicit the dependence on $k$ by writing $W_i^{(k)}$ instead of $W_i$ and $\bar{Q}^{(k)}_i$ instead of $\bar{Q}_i$. We show that there is a $K$ such that for all $k \geq K$, $(S^{(k)}_1, \ldots, S^{(k)}_L)^T$ is a permutation of $(P_1, \ldots, P_L)^T$. Then, the result will follow from Lemma \ref{vertex_test}. 

Let $Q_i = \vec{\tau}^T_i \vec{P}$, $\bar{Q}^{(k)}_i = \vec{\tau}^{{(k)}^T} \vec{P}$, and  $W_i^{(k)} = \vec{\gamma}_i^{{(k)}^T} \vec{P}$. We prove the first claim inductively. First, we consider the base case: there is large enough $k$ such that $W_1^{(k)} = P_i$ for some $i \in [L]$. We will apply Lemma \ref{single}. By Lemma \ref{random_lin_indep}, $\vec{\tau}_1, \ldots, \vec{\tau}_L$ are linearly independent. Therefore, $\aff(\vec{\tau}_2, \ldots, \vec{\tau}_L)$ gives a hyperplane with an associated open halfspace $\vec{H}$ that contains $\vec{\tau}_1$ and at least one $\vec{e}_j$. We can pick $k$ large enough such that for all $\vec{e}_j \in \vec{H}$, $\lambda_k \equiv \frac{1}{k} \vec{\tau}_1 + \frac{k-1}{k} \vec{\tau}^{(k)}_1 \in \co(\vec{e}_j, \vec{\tau}_2, \ldots, \vec{\tau}_L)$. Then, for all $\vec{e}_j \in \vec{H}$, there exists $\kappa_j > 0$ such that 
\begin{align*}
\vec{\lambda}_k & = \kappa_j \vec{e}_j + (1- \kappa_j) \tilde{\tau}_j
\end{align*}
for some $\tilde{\tau}_j \in \co(\vec{\tau}_2, \ldots, \vec{\tau}_L)$. We claim that there is a unique smallest $\kappa_j$. Suppose to the contrary that there is $i \neq j$ such that $\kappa_i = \kappa_j$. Then, 
\begin{align*}
\vec{\lambda}_k & = \kappa \vec{e}_i + (1- \kappa) \tilde{\tau}_i \\
\vec{\lambda}_k & = \kappa \vec{e}_j + (1- \kappa) \tilde{\tau}_j
\end{align*}
where $\kappa = \kappa_i = \kappa_j$. Then, $(1-\kappa) (\tilde{\vec{\tau}}_j - \tilde{\vec{\tau}}_i) - \kappa(\vec{e}_i - \vec{e}_j) = 0$, from which it follows that $\vec{e}_i - \vec{e}_j \in \spa(\vec{\tau}_2, \ldots, \vec{\tau}_L)$. But, by Lemma \ref{random_lin_indep}, $\vec{e}_i - \vec{e}_j, \vec{\tau}_2, \ldots, \vec{\tau}_L$ are linearly independent with probability $1$ and, hence, we have a contradiction. Therefore, there is a unique $j$ that minimizes $\kappa_j$. Note that for all $\vec{e}_i \not \in \vec{H}$, if we write $\vec{e}_i = \sum_{l \geq 2 } a_l \vec{\tau}_l + a_1 \vec{\lambda}_k$, then $a_1 \leq 0$. Then, by Lemma \ref{single}, $\vec{e}_j$ is the residue of $\vec{\lambda}_k$ with respect to $\vec{\tau}_2, \ldots, \vec{\tau}_L$. Therefore, by Proposition \ref{equiv_opt}, if $W^{(k)}_1 \longleftarrow \text{Residue}(\frac{1}{k} Q_1 + (1- \frac{1}{k}) \bar{Q}_1 \, | \,   \set{Q_j }_{j > 1} )$, $W^{(k)}_1$ is one of the base distributions. 

The inductive step is similar. Suppose that $W_1^{(k)}, \ldots, W_n^{(k)}$ are distinct base distributions for $k \geq K_1$. We show that there exists $K_2$ such that for $k \geq K_2$, $W_1^{(k)}, \ldots, W_{n+1}^{(k)}$ are distinct base distributions. Let $k \geq K_1$. By hypothesis, $|\set{\vec{\gamma}_1^{(k)}, \ldots, \vec{\gamma}_n^{(k)}}| = n$ and $\set{\vec{\gamma}_1^{(k)}, \ldots, \vec{\gamma}_n^{(k)}} \subseteq \set{\vec{e}_1, \ldots, \vec{e}_L}$. Without loss of generality, let $\vec{\gamma}_i^{(k)} = \vec{e}_i$ for $i \leq n$. By Lemma \ref{random_lin_indep}, $\vec{e}_1, \ldots, \vec{e}_n, \vec{\tau}_{n+1}, \ldots, \vec{\tau}_L$ are linearly independent with probability $1$. The rest of the argument is identical to the base case. 

The result follows from applying Lemma \ref{vertex_test}.
\end{proof}

We remark here that the partial label algorithm is basically identical for the non-square case. The above proofs did not make use of the fact that we assumed $M =L$. 
\subsection{Estimation}

\subsubsection{ResidueHat Results}
Let $A_1, A_2, \ldots$ denote positive constants whose values change from line to line. We introduce the following definitions.
\begin{defn}
We say that the distribution $F$ satisfies the \emph{support condition} \textbf{(SC)} with respect to $H$ if there exists a distribution $G$ and $\gamma \in [0,1]$ such that $\supp(H) \not \subseteq \supp(G)$ and $F = (1-\gamma) G + \gamma H$.
\end{defn}
\begin{defn}
If
\begin{equation*}
\sup_{S \in \sS} |\est{F}(S) - F(S)| \overset{i.p.}{\longrightarrow} 0
\end{equation*}
as $\vec{n} \longrightarrow \infty$, we say that $\est{F} \longrightarrow F$ uniformly (or $\est{F}$ converges uniformly to $F$).
\end{defn}

\begin{defn}
Let $\est{F}$ be a ResidueHat estimator of a distribution $F$. We say that $\est{F}$ satisfies a Uniform Deviation Inequality \textbf{(UDI)} with respect to $\sS$ if there exists constants $A_1, A_2 > 0$ such that for large enough $\min(n_i :i \in \sD(\est{F}))$, for all $S \in \sS$
\begin{align*}
|\est{F}(S) - F(S)| < A_1\gamma(\sD(\est{F}))
\end{align*}
\noindent with probability at least $1 - A_2 \sum_{i \in \sD(\est{F})} \frac{1}{n_i}$
\end{defn}
Henceforth, for the purposes of abbreviation, we will only say that a ResidueHat estimator satisfies a Uniform Deviation Inequality \textbf{(UDI)} and omit ``with respect to $\sS$" because the context makes this clear.  
\begin{defn}
Let $\est{F}$ and $\est{H}$ be ResidueHat estimators. We say that $\est{\kappa}(\est{F} \, | \, \est{H})$ satisfies a Rate of Convergence \textbf{(RC)} if there exists constants $A_1,A_2 > 0$ such that for large enough $\min(n_i :i \in \sD(\est{H}) \cup \sD(\est{F}))$,
\begin{align*}
|\est{\kappa}(\est{F} \, | \, \est{H})  - \kappa^*(F \, | \, H) | \leq A_1 \gamma(\sD(\est{H}) \cup \sD(\est{F}))
\end{align*}
\noindent with probability at least $1 - A_2 \sum_{i \in \sD(\est{F}) \cup \sD(\est{H})} \frac{1}{n_i}$.
\end{defn}
Lemma \ref{sup_cond} gives sufficient conditions under which $F$ satisfies \textbf{(SC)} with respect to $H$.
\begin{lemma}
\label{sup_cond}
Let $P_1, \ldots, P_L$ satisfy \textbf{(A$''$)} and let $F,H \in \co(P_1, \ldots, P_L)$ such that $F \neq H$. Then, $F$ satisfies \textbf{(SC)} with respect to $H$.
\end{lemma}
\begin{proof}
Let $A = \argmin(|B| : B \subseteq \set{P_1, \ldots, P_L}, F,H \in \co(B))$. Without loss of generality, suppose that $A = \set{P_1, \ldots, P_K}$. $F$ either lies on the boundary of $\co(P_1, \ldots,P_K)$ or doesn't. If $F$ lies on the boundary of $\co(P_1, \ldots,P_K)$, then we pick $G = F$ and $\gamma = 0$ to obtain $F = (1-\gamma) F + \gamma H$. Since $P_1, \ldots, P_L$ satisfy \textbf{(A$''$)}, $\supp(H) \not \subseteq \supp(F)$. 

Now, suppose that $F \in \co(P_1, \ldots, P_K)^\circ$. Let $G \longleftarrow \text{Residue}(F \, | \, H)$; we can write $F = (1-\gamma)G + \gamma H$ for $\gamma \in [0,1)$ since $F \neq H$. Then, by Statement \emph{2} of Lemma \ref{facts} and Proposition \ref{equiv_opt}, $G$ is on the boundary of $\co(P_1, \ldots, P_K)$. 

Without loss of generality, suppose that $G \in \co(P_1, \ldots, P_{K-1})$. Since $F = (1- \gamma) G + \gamma H \not \in \co(P_1, \ldots, P_{K-1})$, $H \not \in \co(P_1, \ldots, P_{K-1})$. Since $P_1, \ldots, P_L$ satisfy \textbf{(A$''$)}, $\supp(H) \not \subseteq \supp(G)$. This completes the proof.
\end{proof}
Lemma \ref{vc_type_ineq} gives sufficient conditions under which an estimator $\est{G}$ satisfies a \textbf{(UDI)}.
\begin{lemma}
\label{vc_type_ineq}
Let
\begin{enumerate}
\item $F$ and $H$ be distributions such that $F \neq H$,

\item $G \longleftarrow \text{Residue}(F \, | \, H)$, and

\item $\est{G} \longleftarrow \text{ResidueHat}(\est{F} \, | \, \est{H})$.
\end{enumerate}
If $\est{\kappa}(\est{F} \, | \, \est{H})$ satisfies a \textbf{(RC)}, $\est{H}$ satisfies a \textbf{(UDI)}, and $\est{F}$ satisfies a \textbf{(UDI)}, then $\est{G}$ satisfies a \textbf{(UDI)}. 
\end{lemma}
\begin{proof}
For the sake of abbreviation, let $\est{\kappa} = \est{\kappa}(\est{F} \, | \, \est{H})$, $\kappa^* = \kappa^*(F \, | \, H)$, $\est{\alpha} = \frac{1}{1- \est{\kappa}}$ and $\alpha^* = \frac{1}{1-\kappa^*}$. We claim that there are constants $A_1, A_2 > 0$ such that for sufficiently large $\min(n_i : i \in \sD(\est{H}) \cup \sD(\est{F}))$,
\begin{align}
\Pr(|\est{\alpha}  - \alpha^* | < A_1 \gamma(\sD(\est{H}) \cup \sD(\est{F}))) & \geq 1-  A_2 \sum_{i \in \sD(\est{H}) \cup \sD(\est{F})} \frac{1}{n_i} \label{alpha_rate}
\end{align}
\noindent Since $\est{\kappa}$ satisfies a \textbf{(RC)}, there exists constants $A_1,A_2 > 0$ such that for large enough $\min(n_i :i \in \sD(\est{H}) \cup \sD(\est{F}))$,
\begin{align*}
|\est{\kappa} - \kappa^*| \leq A_1 \gamma(\sD(\est{H}) \cup \sD(\est{F}))
\end{align*}
\noindent with probability at least $1 - A_2 \sum_{i \in \sD(\est{F}) \cup \sD(\est{H})} \frac{1}{n_i}$. $\kappa^* < 1$ by Proposition \ref{label_noise_bin_case}, so we can let $\min(n_i :i \in \sD(\est{H}) \cup \sD(\est{F}))$ large enough so that $1-\kappa^* - A_1 \gamma(\sD(\est{H}) \cup \sD(\est{F})) > 0$. Then, on this same event, for large enough $\min(n_i :i \in \sD(\est{H}) \cup \sD(\est{F}))$,
\begin{align*}
|\frac{1}{1-\kappa^*} - \frac{1}{1-\est{\kappa}}| & \leq A_1 \frac{\gamma(\sD(\est{H}) \cup \sD(\est{F}))}{(1-\kappa^*)(1-\est{\kappa})} \\
& \leq A_1 \frac{\gamma(\sD(\est{H}) \cup \sD(\est{F}))}{(1-\kappa^*)(1-\kappa^* - A_1 \gamma(\sD(\est{H}) \cup \sD(\est{F})))}
\end{align*}
This proves the claim.

We can write $G= \alpha F + (1-\alpha) H$ with $\alpha \geq 1$. Then, by the triangle inequality, 
\begin{align*}
|\est{G} - G| & = | \est{\alpha} \est{F} + (1-\est{\alpha}) \est{H} -  \alpha F - (1-\alpha) H | \\
& \leq |\est{\alpha} \est{F} - \alpha F | + |(1-\est{\alpha}) \est{H} - (1-\alpha) H| \\
& = |\est{\alpha} \est{F} -\est{\alpha} F + \est{\alpha} F - \alpha F | + |(1-\est{\alpha}) \est{H} - (1-\est{\alpha}) H  + (1-\est{\alpha})H - (1-\alpha) H| \\
& \leq |\est{\alpha}| |\est{F} - F| +  |\est{\alpha} - \alpha | + |1-\est{\alpha}||\est{H} -  H | + |\est{\alpha}- \alpha|
\end{align*}
\noindent Since $\est{F}$ satisfies a \textbf{(UDI)}, $\est{H}$ satisfies a \textbf{(UDI)}, and inequality (\ref{alpha_rate}) holds, the result follows by the union bound. 
\end{proof}
Lemma \ref{upper_bound} gives sufficient conditions under which an estimator $\est{\kappa}$ satisfies one part of the inequality in a \textbf{(RC)}.
\begin{lemma}
\label{upper_bound}
Let $\est{F}$ and $\est{H}$ be estimates of distributions $F$ and $H$ that satisfy \textbf{(UDI)}s. Then, there exist constants $A_1,A_2 > 0$ such that for large enough $\min(n_i :i \in \sD(\est{H}) \cup \sD(\est{F}))$,
\begin{align*}
\kappa^*(F \, | \, H) - \est{\kappa}(\est{F} \, | \, \est{H})  \leq A_1 \gamma(\sD(\est{H}) \cup \sD(\est{F}))
\end{align*}
\noindent with probability at least $1 - A_2 \sum_{i \in \sD(\est{F}) \cup \sD(\est{H})} \frac{1}{n_i}$.
\end{lemma}
\begin{proof}
For abbreviation, let $\kappa^* = \kappa(F \, | \, H)$ and $\est{\kappa} = \est{\kappa}(\est{F} \, | \, \est{H})$. Fix $S \in \sS$ such that $H(S) > 0$. Since by hypothesis $\est{F}$ and $\est{H}$ satisfy \textbf{(UDI)}, there exist constants $A_1, A_2 > 0$ such that for large enough $\min(n_i : i \in \sD(\est{F}) \cup \sD(\est{H}))$, with probability at least $1 - A_1[\sum_{i \in \sD(\est{F}) \cup \sD(\est{H})} \frac{1}{n_i}]$,
\begin{align*}
|\est{F}(S) - F(S) | & < A_2 \gamma(\sD(\est{F})) \\
|\est{H}(S) - H(S) | & < A_2 \gamma(\sD(\est{H}))
\end{align*}
\noindent If $A_2 \leq 1$, then $\kappa^* \leq \est{\kappa}$ on this event since
\begin{align*}
\kappa^* = \inf_{S \in \sS} \frac{F(S)}{H(S)} \leq \frac{F(S)}{H(S)}
& \leq \frac{\widehat{F}(S) + A_2 \gamma(\sD(\est{F})) }{(\widehat{H}(S) - A_2 \gamma(\sD(\est{H})))_+} \leq \frac{\widehat{F}(S) + \gamma(\sD(\est{F})) }{(\widehat{H}(S) - \gamma(\sD(\est{H})))_+}
\end{align*}
and we take the infimimum over $S \in \sS$ such that $H(S) > 0$. 

Now, suppose that $A_2 > 1$. Let $\min(n_i : i \in \sD(\est{F}) \cup \sD(\est{H}))$ large enough so that $H(S) - 2A_2\gamma(\sD(\est{H})) > 0$ and with high probability $\est{H}(S) - A_2 \gamma(\sD(\est{H}) > 0$. Then, on this event, 
\begin{align*}
\kappa^* & \leq \frac{\est{F}(S) + A_2 \gamma(\sD(\est{F})) }{H(S)} \\
& = \frac{\est{F}(S) + \gamma(\sD(\est{F})) }{H(S)} + A_3 \gamma(\sD(\est{F}))
\end{align*}
\noindent where we fix $A_3 = \frac{(A_2 - 1)}{H(S)}$. Further,
\begin{align*}
\frac{\est{F}(S) + \gamma(\sD(\est{F})) }{H(S)} & \leq \frac{\est{F}(S) + \gamma(\sD(\est{F})) }{\est{H}(S) - A_2 \gamma(\sD(\est{H})} \\
& = \frac{\est{F}(S) + \gamma(\sD(\est{F})) }{\est{H}(S) - \gamma(\sD(\est{H}))}  \Big[ \frac{\est{H}(S) - \gamma(\sD(\est{H})) }{\est{H}(S) - A_2 \gamma(\sD(\est{H}))} \Big] \\
& \leq \frac{\est{F}(S) + \gamma(\sD(\est{F})) }{\est{H}(S) - \gamma(\sD(\est{H}))} \Big[ 1 + \frac{(A_2-1) \gamma(\sD(\est{H}))}{\est{H}(S) - A_2 \gamma(\sD(\est{H}))} \Big] \\
& \leq \frac{\est{F}(S) + \gamma(\sD(\est{F})) }{\est{H}(S) - \gamma(\sD(\est{H}))}\Big[ 1 + \frac{(A_2-1) \gamma(\sD(\est{H}))}{H(S) - 2A_2\gamma(\sD(\est{H}))} \Big] \\
& \leq \frac{\est{F}(S) + \gamma(\sD(\est{F})) }{\est{H}(S) - \gamma(\sD(\est{H}))} \Big[ 1 + \frac{(A_2 - 1)}{C}\gamma(\sD(\est{H})) \Big]
\end{align*}
\noindent where we pick $0< C < H(S) - 2A_2\gamma(\sD(\est{H}))$. Combining the above inequalities and taking the infimum over $S \in \sS$ such that $H(S) > 0$ gives
\begin{align*}
\kappa^* & \leq \widehat{\kappa} + \frac{\widehat{\kappa} (A_2 - 1)_+}{C}\gamma(\sD(\est{H}))) +  A_3 \gamma(\sD(\est{F})) 
\end{align*}
\noindent Noting that there exists $M \in \bbR$ such that $\widehat{\kappa} < M$ for large enough $\min(n_i : i \in \sD(\est{F}) \cup \sD(\est{H}))$ since there is some $S \in \sS$ such that $H(S) > 0$, we obtain the result.
\end{proof}
Lemma \ref{rate} gives sufficient conditions under which $\est{\kappa}$ satisfies \textbf{(RC)}.
\begin{lemma}
\label{rate}
Let $F$ and $H$ be distributions such that $F \neq H$. If
\begin{itemize}
\item $F$ satisfies \textbf{(SC)} with respect to $H$,

\item $\est{F}$ satisfies \textbf{(UDI)}, and

\item $\est{H}$ satisfies \textbf{(UDI)}, 

\end{itemize}
then $\est{\kappa}(\est{F} \, | \, \est{H})$ satisfies \textbf{(RC)}.
\end{lemma}
\begin{proof}
By Lemma \ref{upper_bound}, there exists constants $A_1,A_2 > 0$ such that for large enough $\min(n_i :i \in \sD(\est{H}) \cup \sD(\est{F}))$,
\begin{align*}
\kappa^*(F \, | \, H) - \est{\kappa}(\est{F} \, | \, \est{H})  \leq A_1 \gamma(\sD(\est{H}) \cup \sD(\est{F}))
\end{align*}
\noindent with probability at least $1 - A_2 \sum_{i \in \sD(\est{F}) \cup \sD(\est{H})} \frac{1}{n_i}$.

The proof of the other direction of the inequality is very similar to the proof of Theorem 2 in \citet{scott2015}. By hypothesis, $F$ satisfies \textbf{(SC)} with respect to $H$, so there exists a distribution $G$ such that $F = \gamma G + (1-\gamma) H$ for some $\gamma \in [0,1]$ and $\supp(H) \not \subseteq \supp(G)$. Therefore, we have that $G$ is irreducible with respect to $H$ and, by Proposition \ref{label_noise_bin_case}, $\kappa^*(F \, | \, H) = \gamma$. For abbreviation, let $\kappa^* = \kappa^*(F \, | \, H)$ and $\est{\kappa} = \est{\kappa} (\est{F} \, | \, \est{H})$. Since $\supp(H) \not \subseteq \supp(G)$, there exists an open set $O$ such that 
\begin{align*}
\frac{G(O)}{H(O)} & = ( 1- \gamma) \frac{G(O)}{H(O)} + \gamma = \kappa^*
\end{align*}
\noindent Then, since $\sS$ contains a generating set for the standard topology on $\bbR^d$, there exists $S \in \sS$ such that
\begin{align*}
\frac{G(S)}{H(S)} = \kappa^*
\end{align*}
Since by hypothesis $\est{F}$ and $\est{H}$ satisfy \textbf{(UDI)}, there exist constants $A_3, A_4 > 0$ such that for large enough $\min(n_i : i \in \sD(\est{F}) \cup \sD(\est{H}))$, with probability at least $1 - A_3[\sum_{i \in \sD(\est{F}) \cup \sD(\est{H})} \frac{1}{n_i}]$,
\begin{align*}
\est{\kappa} & \leq \frac{F(S) + A_4\gamma(\sD(\est{F})) }{(H(S) - A_4\gamma(\sD(\est{H})) )_+} \\
& \leq \frac{F(S) + \epsilon}{(H(S) - \epsilon)_+}
\end{align*}
\noindent where $\epsilon = 2A_4 \gamma(\sD(\est{F}) \cup \sD(\est{H}))$. The rest of the proof is identical to the proof of Theorem 2 from \citet{scott2015} and, therefore, we omit it.
\end{proof}
The following theorem gives sufficient conditions under which a ResidueHat estimator satisfies \textbf{(UDI)}. It is the basis of Proposition \ref{uniform_convergence}.
\begin{lemma}
\label{main_est_ineq}
If $P_1, \ldots, P_L$  satisfy \textbf{(A$''$)} and $\est{G}$ is a ResidueHat estimator of a distribution $G \in \co(P_1, \ldots, P_L)$, then $\est{G}$ satisfies \textbf{(UDI)}.
\end{lemma}
\begin{proof}
A ResidueHat estimator is defined recursively in terms of other ResidueHat estimators, which are empirical distributions in the base case. Let $\est{F}_1, \ldots, \est{F}_k$ denote the ResidueHat estimators in terms of which $\est{G}$ is defined and let $F_1, \ldots, F_k$ denote the distributions that they estimate. Consider the directed graph $(V,E)$ defined as follows: $V = \set{v_1, \ldots, v_k,g}$ where $v_i$ corresponds to estimator $\est{F}_i$ and $g$ to estimator $\est{G}$. Let $(v_i,v_j) \in E$ if there exists $l \in [k]$ such that either $\est{F}_j \longleftarrow \text{ResidueHat}(\est{F}_i \, | \, \est{F}_l)$ or $\est{F}_j \longleftarrow \text{ResidueHat}(\est{F}_l \, | \, \est{F}_i)$; define edges similarly for the nodes corresponding to the estimators $\est{G}$, $\est{F}_i$ and $\est{F}_j$ such that $\est{G} \longleftarrow \text{ResidueHat}(\est{F}_i \, | \, \est{F}_j)$. Notice that \emph{(i)} $(V,E)$ is a directed acyclic graph, \emph{(ii)} the set of nodes in $(V,E)$ with no incoming edges correspond to the empirical distributions among $\est{F}_1, \ldots, \est{F}_k$, \emph{(iii)} the only node with no outgoing edges is $g$, and \emph{(iv)} every node with incoming edges has exactly $2$ incoming edges.

We will show that $\est{F}_1, \ldots, \est{F}_k, \est{G}$ all satisfy \textbf{(UDI)} by considering the estimators in the following order: first, we consider estimators whose nodes have incoming edges from nodes representing empirical distributions. Then, we only consider a node if its incoming edges come from nodes that represent estimators that have either been shown to satisfy \textbf{(UDI)} or are empirical distributions. There are four cases that we must consider. Let $\est{F}_l \longleftarrow \text{ResidueHat}(\est{F}_i \, | \, \est{F}_j)$.
\begin{enumerate}
\item $\est{F}_i$ and $\est{F}_j$ are empirical distributions,

\item $\est{F}_i$ is an empirical distribution and $\est{F}_j$ satisfies \textbf{(UDI)},

\item $\est{F}_i$ satisfies \textbf{(UDI)} and $\est{F}_j$ is an empirical distribution, and

\item $\est{F}_i$ and $\est{F}_j$ satisfy \textbf{(UDI)}.
\end{enumerate}
 
Suppose that $\est{F}_i$ and $\est{F}_j$ are empirical distributions.  Since $\est{F}_i$ and $\est{F}_j$ are empirical distributions, the VC inequality applies to $\est{F}_i$ and $\est{F}_j$. Therefore, $\est{F}_i$ and $\est{F}_j$ satisfy \textbf{(UDI)}. Since $\est{F}_l$ is a ResidueHat estimator, by definition $F_i, F_j \in \co(P_1, \ldots, P_L)$ and $F_i \neq F_j$. Then, by Lemma \ref{sup_cond}, $F_i$ satisfies \textbf{(SC)} with respect to $F_j$. Then, by Lemma \ref{rate}, $\est{\kappa}(\est{F}_i \, | \, \est{F}_j)$ satisfies \textbf{(RC)}. Then, all of the assumptions of Lemma  \ref{vc_type_ineq} are satisfied, so $\est{F}_l$ satisfies \textbf{(UDI)}. Note that $F_l \in \co(P_1, \ldots, P_L)$ by Proposition \ref{equiv_opt}. The rest of the cases are similar. Making the appropriate argument at each node in the directed acylcic graph eventually shows that $\est{G}$ satisfies \textbf{(UDI)}. 
\end{proof}
\begin{proof}[Proposition \ref{uniform_convergence}]
By Lemma \ref{main_est_ineq}, $\est{G}$ satisfies \textbf{(UDI)}. Explicitly, there exist constants $A_1, A_2 > 0$ such that for large enough $\min(n_i : i \in \sD(\est{G}))$ with probability at least $1 - A_1 \sum_{i \in \sD(\est{G})} \frac{1}{n_i}$, $\est{G}$ satisfies for every $S \in \sS$,
\begin{align*}
|\est{G}(S) - G(S)| & < A_2 \gamma(\sD(\est{G})) =  A_2 \sum_{i \in \sD(\est{G})}  \epsilon_i(\frac{1}{n_i}) \longrightarrow 0 
\end{align*}
as $\vec{n} \longrightarrow \infty$.
\end{proof}
\subsubsection{The FaceTestHat Algorithm}
The FaceTestHat algorithm (see Algorithm \ref{face_test_hat_alg}) differs from the FaceTest algorithm in that it requires the specification of $\epsilon > 0$. As will become clear in the following section, this difference introduces further changes in DemixHat.
\begin{algorithm}
\caption{FaceTestHat($\est{Q}_1,\cdots, \est{Q}_K \, | \, \epsilon$)}
\begin{algorithmic}[1]
\label{face_test_hat_alg}
\FOR{$i=1,\cdots, K$}
\FOR{$j=1,\cdots, K$}
\IF{$i \neq j$}
\STATE $\est{R}_{i,j} \longleftarrow \text{ResidueHat}(\est{Q}_i \, | \, \est{Q}_j)$
\IF{$\est{\kappa}( \est{Q}_i \, | \, \est{R}_{i,j}) \geq 1-\epsilon$}
\RETURN $0$
\ENDIF
\ENDIF
\ENDFOR
\ENDFOR
\RETURN $1$
\end{algorithmic}
\end{algorithm}

\begin{lemma}
\label{face_test_hat}
Let $\epsilon > 0$.	Suppose that $P_1, \ldots, P_L$ satisfy \textbf{(A$''$)} and that $Q_1, \ldots, Q_K$ are distinct distributions lying in $\co(P_1, \ldots, P_L)$. Let $\est{Q}_i$ be a ResidueHat estimate of $Q_i$ $\forall i \in [K]$. Then, as $\vec{n} \longrightarrow \infty$, with probability tending to $1$, if FaceTestHat($\est{Q}_1,\cdots,\est{Q}_K \, | \, \epsilon$) returns $1$, then $Q_1,\cdots,Q_K$ are in the interior of the same face. 
\end{lemma}
\begin{proof}[Lemma \ref{face_test_hat}]
Let $\epsilon > 0$, $\kappa^*_{i,j} = \kappa^*(Q_i \, | \, R_{i,j} )$ and $\est{\kappa}_{i,j} = \est{\kappa}( \est{Q}_i \, | \, \est{R}_{i,j})$. We prove the contrapositive. Suppose that $Q_1, \ldots, Q_K$ are not on the interior of the same face. Then, by Proposition \ref{face_test}, $\text{FaceTest}(Q_1,\ldots,Q_K)$ returns $0$, which occurs if and only if there exist $i \neq j$ such that $\kappa^*_{i,j} = 1$.  

Since $\est{Q}_i$ and $\est{Q}_j$ are ResidueHat estimators, by Lemma \ref{main_est_ineq}, $\est{Q}_i$ and $\est{Q}_j$ satisfy \textbf{(UDI)}. Since $Q_i \neq Q_j$ by hypothesis, $\est{R}_{i,j}$ is a ResidueHat estimator and, therefore, satisfies a \textbf{(UDI)} by Lemma \ref{main_est_ineq}. 

By Lemma \ref{upper_bound}, there exist constants $A_1, A_2> 0$ such that there is some $N_0$ such that if $\min(n_s : s \in \sD(\est{R}_{i,j})) \geq N_0$, with probability at least $1 - A_1 \sum_{s \in \sD(\est{R}_{i,j})} \frac{1}{n_s}$, we have that $\est{\kappa}_{i,j} + A_2 \gamma(\sD(\est{R}_{i,j})) \geq \kappa^*_{i,j}$.  Let $N_1$ be such that if $\min(n_s : s \in \sD(\est{R}_{i,j})) \geq N_1$, then $A_2\gamma(\sD(\est{R}_{i,j})) < \epsilon$. Then, if $\min(n_s : s \in \sD(\est{R}_{i,j})) \geq \max(N_0, N_1)$, with probability at least $1 - A_1 \sum_{s \in \sD(\est{R}_{i,j})} \frac{1}{n_s}$
\begin{align*}
 \est{\kappa}_{i,j} & \geq \kappa^*_{i,j} - \gamma(\sD(\est{R}_{i,j})) \\
& \geq 1 - \epsilon
\end{align*}
\noindent So, FaceTestHat($\set{\est{Q}_1,\cdots,\est{Q}_K}$) returns $0$.
\end{proof}
\begin{lemma}
\label{face_test_hat_term}
Suppose that $P_1, \ldots, P_L$ satisfy \textbf{(A$''$)} and that $Q_1, \ldots, Q_K$ are distinct distributions lying on the interior of the same face of $\co(P_1, \ldots, P_L)$. Let $\est{Q}_i$ be a ResidueHat estimate of $Q_i$ $\forall i \in [K]$. Then, there exists $\delta > 0$ such that if $\delta > \epsilon > 0$, then with probability increasing to $1$ as $\vec{n} \longrightarrow \infty$, FaceTestHat($\est{Q}_1, \ldots, \est{Q}_L \, | \, \epsilon$) returns $1$.
\end{lemma}
\begin{proof}
Let $\kappa^*_{i,j} = \kappa^*( Q_i, \, | \, R_{i,j})$ and $\est{\kappa}_{i,j} = \est{\kappa}(\est{Q}_i \, | \, \est{R}_{i,j})$. By Proposition \ref{face_test}, since $Q_1, \ldots, Q_K$ are distinct distributions lying on the interior of the same face of $\co(P_1, \ldots, P_L)$, $\kappa^*_{i,j} < 1$ for all $i,j \in [K]$ not equal. We verify the conditions of Lemma \ref{rate} to obtain consistency of each $\est{\kappa}_{i,j}$. Since every $\est{Q}_i$ is a ResidueHat estimator, every $\est{Q}_i$ satisfies \textbf{(UDI)} by Lemma \ref{main_est_ineq}. Since $Q_i \neq Q_j$, $\est{R}_{i,j}$ is a ResidueHat estimator and, therefore, satisfies \textbf{(UDI)} by Lemma \ref{main_est_ineq}.  Since for all $i \in [K]$, $Q_i$ in the interior of the same face, by Statement \emph{2} of Lemma \ref{facts} and Proposition \ref{equiv_opt}, we have that $R_{i,j} \neq Q_i$ for all $i,j \in [K]$. Since $P_1, \ldots, P_L$ satisfy \textbf{(A$''$)}, $R_{i,j}, Q_i \in \co(P_1, \ldots, P_L)$, and $Q_i \neq R_{i,j}$, it follows by Lemma \ref{sup_cond} that $Q_i$ satisfies \textbf{(SC)} with respect to $R_{i,j}$. Since the conditions of Lemma \ref{rate} are satisfied, each $\est{\kappa}_{i,j}$ satisfies \textbf{(RC)} and is therefore consistent to $\kappa^*_{i,j}$. Pick $\delta = \min_{i,j} \frac{1-\kappa^*_{i,j}}{2}$. Then, the result follows.
\end{proof}
\subsubsection{The DemixHat Algorithm}
The DemixHat algorithm (see Algorithm \ref{demix_hat_alg}) differs from the Demix algorithm in that it applies the FaceTestHat algorithm with increasing values of $\epsilon >0$. At some point, it is guaranteed to estimate distributions on the same face of the simplex and pick $\epsilon$ sufficiently small so that FaceTestHat can determine that they are on the same face with probability tending to $1$. 
\begin{algorithm}
\caption{DemixHat($\est{S}_1, \ldots, \est{S}_K$)}
\textbf{Input: }$\est{S}_1, \ldots, \est{S}_K$ are ResidueHat estimates
\begin{algorithmic}[1]
\label{demix_hat_alg}
\IF{$K = 2$}
\STATE $\est{Q}_1 \longleftarrow$ ResidueHat($\est{S}_1 \, | \, \est{S}_2$)
\STATE $\est{Q}_2 \longleftarrow$ ResidueHat($\est{S}_2 \, | \, \est{S}_1$)
\RETURN $(\est{Q}_1, \est{Q}_2)$
\ELSE
\STATE $\est{Q} \longleftarrow \text{ uniformly distributed random element from } \co (\est{S}_2, \ldots, \est{S}_K)$
\STATE $n \longleftarrow 1$
\STATE $T \longleftarrow 0$
\WHILE{$T == 0$} \label{while_loop_demixhat}
\STATE $n \longleftarrow n + 1$
\FOR{$i=2,\cdots, K$} 
\STATE $\est{R}_i \longleftarrow \text{ResidueHat}(m_\frac{n-1}{n}(\est{S}_i, \est{Q}) \, | \, \est{S}_1)$ \label{R_hat_demixhat}
\ENDFOR
\STATE $\epsilon \longleftarrow \frac{1}{2^{n+1}}$
\STATE $T \longleftarrow \text{FaceTestHat}(\est{R}_2, \cdots, \est{R}_K \, | \, \epsilon)$	 	
\ENDWHILE
\STATE $(\est{Q}_1, \cdots, \est{Q}_{K-1})^T \longleftarrow \text{DemixHat}(\est{R}_2,\cdots,\est{R}_K)$\\
\STATE $\est{Q}_K \longleftarrow \frac{1}{K} \sum_{i=1}^K \est{S}_i$\\
\FOR{$i = 1, \ldots, K-1$}
\STATE $\est{Q}_K \longleftarrow \text{ResidueHat}(\est{Q}_K \, | \, \est{Q}_i)$
\ENDFOR
\RETURN $(\est{Q}_1, \cdots, \est{Q}_K)^T$
\ENDIF
\end{algorithmic}
\end{algorithm}

\begin{proof}[Theorem \ref{demix_estim}]
Note that every estimator of a distribution in the DemixHat algorithm is a ResidueHat estimator since the Demix algorithm only considers distributions that are in $\co(P_1, \ldots, P_L)$ and only computes Residue($F \, | \, H$) if $F \neq H$. Therefore, every estimator of a distribution of the DemixHat algorithm satisfies the assumptions of Lemma \ref{main_est_ineq}. 

First, we argue that DemixHat uses ResidueHat estimators that are recursively defined in terms of a finite number of ResidueHat esimators so that the constants in the uniform deviation inequalities associated with the ResidueHat estimators do not go to infinity. We find a very loose bound. DemixHat calls itself at most $L-1$ times and in each call recurses on at most $L-1$ ResidueHat estimators and calculates at most $L-1$ more ResidueHat estimators. Therefore, each ResidueHat estimator is recursively defined in terms of at most $(L-1)^3$ ResidueHat estimators.

Second, we argue that with probability increasing to $1$ as $\vec{n} \longrightarrow \infty$, DemixHat eventually terminates. Let $A_i$ denote the event that DemixHat recurses on $i$ distributions lying in the interior of a $i$-face in the $(L-i)$th recursive call. Consider $A_{L-1}$. Let $\est{R}_i^k$ denote the estimate of the $i$th distribution in line \ref{R_hat_demixhat} in the $k$th iteration of the while loop starting on line \ref{while_loop_demixhat} of the DemixHat algorithm and let $R_i^k$ denote the corresponding distribution. By Theorem \ref{demix_identification}, there exists a smallest $N$ such that for $n \geq N$, every $R_i^n$ lies in the interior of the same face. By Lemma \ref{face_test_hat}, with probability increasing to $1$, for any $\tau > 0$, $\text{FaceTestHat}(\est{R}_2^{k}, \ldots, \est{R}_L^{k} \, | \, \tau)$ returns $0$ if $k < N$ since $R_2^k, \ldots, R_L^k$ are not on the interior of the same face. By Lemma \ref{face_test_hat_term}, there exists $\delta > 0$ such that for $\delta > \epsilon > 0$, with probability increasing to $1$, for $n \geq N$, $\text{FaceTestHat}(R_2^{n}, \ldots, R_L^{n} \, | \, \epsilon)$ returns $1$. Hence, with probability increasing to $1$, the event $A_{L-1}$ occurs. Applying the same argument to $A_i$ for $i < L-1$ and taking the union bound shows that DemixHat terminates with probability increasing to $1$. 

Now, we can complete the proof. Under the assumptions of Theorem \ref{demix_identification}, there is a permutation $\sigma$ such that for each distribution $Q_i$ estimated by $\est{Q}_i$, $P_{\sigma(i)} = Q_i$. By Proposition \ref{uniform_convergence}, as $\vec{n} \longrightarrow \infty$, $\est{Q}_i$ converges uniformly to $Q_i$. The result follows.
\end{proof}

\subsubsection{PartialLabelHat Algorithm}

We introduce a finite sample algorithm for decontamination of the partial label model. It combines the DemixHat algorithm and an empirical version of the VertexTest algorithm. We make an assumption that simplifies our algorithm: $\vec{S}$ satisfies
\begin{description}
\item[(D)] there does not exist $i,j \in [L]$ such that $\vec{S}_{i,:} = \vec{e}_j^T$.
\end{description}
In words, this says that there is no contaminated distribution $\p{i}$ and base distribution $P_j$ such that $\p{i} = P_j$. We emphasize that we make this assumption only to simplify the presentation and development of the algorithm; one can reduce any instance of a partial label model satisfying \textbf{(B$'$)}, \textbf{(C)}, and \textbf{(A$''$)} to an instance of a partial label model that also satisfies \textbf{(D)}. We provide a sketch of such a reduction. Let $J = \set{i : \vec{S}_{i,:} = \vec{e}_j^T \text{ for some } j \in [L]} = \set{j_1, \ldots, j_k}$, the set of indices of contaminated distributions that are equal to some base distribution. Compute $\text{Residue}(\p{i} \, | \, \p{{j_1}})$ for $i \in [L] \setminus J$ if there is $l$ such that $\vec{S}_{i,l} = \vec{S}_{j_1, l} = 1$. Replace $\p{i}$ with $\text{Residue}(\p{i} \, | \, \p{{j_1}})$ (and call it $\p{i}$ for simplicity of presentation). Update $\vec{S}$ and remove $j_1$ from $J$. Repeat this procedure until $J$ is empty. Then, there will be $k$ $\p{i}$ lying in a $k$-face of $\Delta_L$ (for some $k$) that are not equal to any of the base distributions and the other contaminated distributions will be equal to base distributions. Then, it suffices to solve the instance of the partial label model on the $k$-face, which satisfies \textbf{(D)}.

\begin{algorithm}
\caption{$\text{PartialLabelHat}(\vec{S}, (\ed{\tilde{P}}_1, \ldots, \ed{\tilde{P}}_L)^T)$}
\begin{algorithmic}[1]
\label{partial_label_hat}
\STATE $(\est{Q}_1, \ldots, \est{Q}_L)^T \longleftarrow \text{DemixHat}(\ed{\tilde{P}}_1, \ldots, \ed{\tilde{P}}_L)$ 
\STATE $\text{FoundVertices} \longleftarrow 0$
\STATE $k \longleftarrow 2$
\WHILE{$\text{FoundVertices } == 0$}
\STATE $(\text{FoundVertices}, \vec{C}) \longleftarrow \text{VertexTest}(\vec{S}, (\ed{\tilde{P}}_1, \ldots, \ed{\tilde{P}}_L)^T, (\est{Q}_1, \ldots, \est{Q}_L)^T, \frac{1}{k})$
\STATE $k \longleftarrow k + 1$
\ENDWHILE
\RETURN $\vec{C} (\est{Q}_1, \ldots, \est{Q}_L)^T$
\end{algorithmic}
\end{algorithm}

\begin{algorithm}
\caption{$\text{VertexTestHat}(\vec{S}, (\ed{\tilde{P}}_1, \ldots, \ed{\tilde{P}}_L)^T, (\est{Q}_1, \ldots, \est{Q}_L)^T, \epsilon)$}
\begin{algorithmic}[1]
\label{vertex_hat_test}
\STATE $\vec{D} \longleftarrow \begin{pmatrix}
 1 & \ldots & 1 \\
 \vdots & \ddots & \vdots \\
 1 & \ldots & 1
 \end{pmatrix} \in \bbR^{L \times L}$
\FOR{$i = 1, \ldots, L$}
\FOR{$j = 1, \ldots, L$}
\IF{$\text{FaceContainHat}(\ed{\tilde{P}}_j , \est{Q}_i \, | \epsilon) == 1$}
\STATE $\vec{D}_{i,:} \longleftarrow \min(\vec{D}_{i,:}, \vec{S}_{j,:})$ \label{make_0_hat}
\ENDIF
\ENDFOR
\ENDFOR
\FOR{$k = 1 \text{ to } L$} \label{k_loop_D}
\FOR{$i = 1 \text{ to } L$}
\IF{ $\vec{D}_{i,:} == \vec{e}_j^T$ \text{ for some } $j$}
\STATE $\vec{D}_{:,j} \longleftarrow \vec{e}_i$ \label{make_0_hat_2nd_loop}
\ENDIF
\ENDFOR
\ENDFOR
\IF{$\vec{D} \text{ is a permutation matrix}$}
\RETURN $(1, \vec{D}^T)$
\ELSE
\RETURN $(0, \vec{D}^T)$
\ENDIF
\end{algorithmic}
\end{algorithm}

\begin{algorithm}
\label{face_contain_hat}
\caption{$\text{FaceContainHat}(\est{Q}_1, \est{Q}_2 \, | \, \epsilon)$}
\begin{algorithmic}[1]
\STATE $\est{R} \longleftarrow \text{ResidueHat}(\est{Q}_1 \, | \, \est{Q}_2)$
\IF{$\est{\kappa}( \est{Q}_1 \, | \, \est{R}) \leq 1 - \epsilon$}
\RETURN $1$
\ELSE
\RETURN $0$
\ENDIF
\end{algorithmic}
\end{algorithm}

\begin{lemma}
\label{vertex_hat_lemma}
Suppose that $P_1, \ldots, P_L$ satisfy \textbf{(A$''$)}, $\vec{\Pi}$ satisfies \textbf{(B$'$)}, and $\vec{S}$ satisfies \textbf{(C)} and \textbf{(D)}. Suppose that $\est{Q}_1, \ldots, \est{Q}_L$ are ResidueHat estimators that are a permutation of $P_1, \ldots, P_L$. Then, there exists $\delta > 0$ such that 

\begin{description}
\item[(i)] if $\delta > \epsilon > 0$, then with probability increasing to $1$, as $\vec{n} \longrightarrow \infty$, $\text{VertexTest}(\vec{S}, (\ed{\tilde{P}}_1, \ldots, \ed{\tilde{P}}_L)^T, (\est{Q}_1, \ldots, \est{Q}_L)^T, \epsilon)$ returns a permutation matrix $\vec{C}$ such that $\forall i$, $\vec{C}_{i,:} (\est{Q}_1, \ldots, \est{Q}_L)^T$ is a ResidueHat estimator of $P_i$;

\item[(ii)]\sloppy if $\epsilon > \delta > 0$, then with probability increasing to $1$, as $\vec{n} \longrightarrow \infty$, either $\text{VertexTest}(\vec{S}, (\ed{\tilde{P}}_1, \ldots, \ed{\tilde{P}}_L)^T, (\est{Q}_1, \ldots, \est{Q}_L)^T, \epsilon)$ returns the same permutation matrix $\vec{C}$ as in \emph{\textbf{(i)}} or it returns the value $(0, \vec{C})$ indicating that it did not find the permutation matrix.
\end{description}
\end{lemma}
\begin{proof}

\begin{description}
\item[(i)] \sloppy Fix $i,j$. Let $\p{j}$ and $Q_i$ denote the distributions that $\ed{\tilde{P}}_j$ and $\est{Q}_i$ estimate. We show that there exists $\delta_{i,j} > 0$ such that if $0 < \epsilon < \delta_{i,j}$, then $\kappa^*(\p{j} \, | \, Q_i) \in (0,1]$ if and only if $\text{FaceContainHat}(\ed{\tilde{P}}_j, \est{Q}_i \, | \, \epsilon)$ returns $1$. We will take $\delta = \min_{i,j} \delta_{i,j}$; \textbf{(i)} will follow from Lemma \ref{vertex_test} and the observation that $\text{VertexTestHat}(\vec{S}, (\ed{\tilde{P}}_1, \ldots, \ed{\tilde{P}}_L)^T, (\est{Q}_1, \ldots, \est{Q}_L)^T, \epsilon)$ and $\text{VertexTest}(\vec{S}, (\p{1}, \ldots, \p{L})^T, (Q_1, \ldots, Q_L)^T)$ have identical behavior if the claim holds.

Let $\est{R} \longleftarrow \text{ResidueHat}(\p{j} \, | \, \est{Q}_i)$ and $R$ denote the distribution that $\est{R}$ estimates. Let $\est{\kappa} = \est{\kappa}( \ed{\tilde{P}}_j \, | \est{R})$. We verify the conditions of Lemma \ref{upper_bound} to obtain an upper bound on $\est{\kappa}$. It follows from the hypothesis that $Q_i \neq \p{j}$. Further, since $\est{Q}_i$ is a ResidueHat estimator by hypothesis, $\ed{\tilde{P}}_j$ is an empirical distribution, and $Q_i, \p{j} \in \co(P_1, \ldots,P_L)$, we have that $\est{R}$ is a ResidueHat estimator by definition. By Lemma \ref{main_est_ineq}, $\widehat{R}$ satisfies \textbf{(UDI)}; by the VC inequality, $\ed{\tilde{P}}_j$ satisfies \textbf{(UDI)}. Then, by Lemma \ref{upper_bound}, there exist constants $A_1,A_2 > 0$ such that for large enough $\min(n_i :i \in \sD(\est{R}) \cup \sD(\ed{\tilde{P}}_j))$,
\begin{align}
\kappa^*(\p{j} \, | \, R) - \est{\kappa}(\ed{\tilde{P}}_j \, | \, \est{R})  \leq A_1 \gamma(\sD(\est{R}) \cup \sD(\ed{\tilde{P}}_j)) \label{inequality}
\end{align}
\noindent with probability at least $1 - A_2 \sum_{i \in \sD(\est{R}) \cup \sD(\ed{\tilde{P}}_j)} \frac{1}{n_i}$.

Now, suppose that $\kappa^*(\p{j} \, | \, Q_i) = 0$. Then, $\p{j} = R$ by definition of $\kappa^*(\p{j} \, | \, Q_i)$. Then, $\kappa^*(\p{j} \, | \, R) = 1$ by definition of $\kappa^*(\p{j} \, | \, R)$. Let $\epsilon > 0$ and $\min(n_i :i \in \sD(\est{R}) \cup \sD(\ed{\tilde{P}}_j))$ large enough so that (\ref{inequality}) holds with the given probability and $A_1 \gamma(\sD(\est{R}) \cup \sD(\ed{\tilde{P}}_j)) \leq \epsilon$. Then, with probability at least $1 - A_2 \sum_{i \in \sD(\est{R}) \cup \sD(\ed{\tilde{P}}_j)} \frac{1}{n_i}$,
\begin{align}
\est{\kappa}(\ed{\tilde{P}}_j \, | \, \est{R}) & \geq \kappa^*(\p{j} \, | \, R) - A_1 \gamma(\sD(\est{R}) \cup \sD(\ed{\tilde{P}}_j))  \geq 1 - \epsilon \label{face_contain_test_inequality}
\end{align}
Therefore, for any $\epsilon > 0$, with probability tending to $1$ as $\vec{n} \longrightarrow \infty$, $\text{FaceContainHat}(\ed{\tilde{P}}_j, \est{Q}_i \, | \epsilon)$ returns $0$.

Suppose that $\kappa^*(\p{j} \, | \, Q_i) \in (0,1]$. Since $\p{j} \neq Q_i$ by hypothesis,  $\kappa^*(\p{j} \, | \, Q_i) \in (0,1)$. Therefore, $R \neq \p{j}$ so that $\kappa^*(\p{j} \, | \, R) < 1$. Since $P_1, \ldots, P_L$ satisfy \textbf{(A$''$)}, by Lemma \ref{sup_cond}, $\p{j}$ satisfies \textbf{(SC)} with respect to $R$. Therefore, by Lemma \ref{rate}, $\est{\kappa}(\ed{\tilde{P}}_j \, | \, \est{R})$ satisfies \textbf{(RC)}. In particular, $\est{\kappa}(\ed{\tilde{P}}_j \, | \, \est{R}) $ is consistent to $\kappa^*(\p{j} \, | \, R)$.

Let $\delta_{i,j} = \frac{1-\kappa^*(\p{j} \, | \, R) }{2}$. Then, for $\epsilon < \delta_{i,j}$, with probability tending to $1$, as $\vec{n} \longrightarrow \infty$, $\est{\kappa} < 1 - \epsilon$. Therefore, $\text{FaceContainHat}(\ed{\tilde{P}}_j, \est{Q}_i \, | \, \epsilon)$ returns $1$.

\item[(ii)] Let $\delta = \min_{i,j} \delta_{i,j}$ and suppose that $\epsilon > \delta > 0$. We again consider compare the execution of 
\begin{align*}
\text{VertexTestHat}&(\vec{S}, (\ed{\tilde{P}}_1, \ldots, \ed{\tilde{P}}_L)^T, (\est{Q}_1, \ldots, \est{Q}_L)^T, \epsilon) \\
\text{VertexTest}&(\vec{S}, (\p{1}, \ldots, \p{L})^T, (Q_1, \ldots, Q_L)^T)
\end{align*}
Let $\vec{D}^{(k)}$ denote the value of the matrix $\vec{D}$ in the VertexTestHat at line \ref{k_loop_D} after going through the loop starting on line \ref{k_loop_D} $k$ times. Let $\vec{C}^{(k)}$ denote the matrix $\vec{C}$ in the VertexTest algorithm at line \ref{k_loop_C} after going through the loop starting on line \ref{k_loop_C} $k$ times. Let $A_i^{(k)} = \set{P_j : \vec{D}^{(k)}_{i,j} = 1}$. Let $B_i^{(k)} = \set{ P_j : \vec{C}^{(k)}_{i,j} = 1}$ as in the proof of Lemma \ref{vertex_test}. 

Claim 1 of the proof of Lemma \ref{vertex_test} shows that $Q_i \in B_i^{(k)}$ for all $k \in [L]$ and all $i \in [L]$. Therefore, it suffices to show that $A_{i}^{(k)} \supset B_{i}^{(k)}$ for all $k$ and $i$. We prove this inductively. We claim that, after the first loop, $A_i^{(0)} \supset B_{i}^{(0)}$: by the inequality \ref{face_contain_test_inequality}, if  $\kappa^*(\p{j} \, | \, Q_i) = 0$, then for all $\epsilon > 0$,  with probability tending $1$ as $\vec{n} \longrightarrow \infty$, $\text{FaceContainHat}(\p{j}, Q_i \, | \epsilon)$ returns $0$. Therefore, VertexTestHat sets an entry $\vec{D}^{(0)}_{i,j}$ to $0$ in line \ref{make_0_hat} only if VertexTest sets the entry $\vec{C}^{(0)}_{i,j}$ to $0$ in line \ref{make_0_not_hat}. This proves the base case.

Now, suppose that $A_i^{(k)} \supset B_i^{(k)}$.  Suppose $\vec{D}^{(k)}_{i,:} = \vec{e}_j^T$ for some $j$ and $i$. There must be some $j$ such that $\vec{C}^{(k)}_{i,j} = 1$ since $(Q_1, \ldots, Q_L)^T$ are a permutation of $(P_1, \ldots, P_L)^T$ by hypothesis and Claim 1 of the proof of \ref{vertex_test}. Then, by the inductive hypothesis, we must have $\vec{C}^{(k)}_{i,:} = \vec{e}_j^T$. Therefore, in the $(k+1)$th iteration, VertexTestHat sets an entry $\vec{D}^{(k)}_{i,j}$ to $0$ in line \ref{make_0_hat_2nd_loop} only if VertexTest sets the entry $\vec{C}^{(k)}_{i,j}$ to $0$ in line \ref{make_0_not_hat_2nd_loop}.
Hence, we delete an element $P_l$ from $A_i^{(k)}$ in the $(k+1)$th iteration of VertexTestHat only if we delete $P_l$ from $B_{i}^{(k)}$ in the $(k+1)$th iteration of VertexTest. This gives $A_i^{(k+1)} \supset B_i^{(k+1)}$. This proves the inductive step, completing the proof.
\end{description}
\end{proof}

\begin{thm}
\label{partial_label_hat}
Let $\epsilon > 0$. Suppose that $P_1, \ldots, P_L$ satisfy \textbf{(A$''$)}, $\vec{\Pi}$ satisfies \textbf{(B$'$)}, and $\vec{S}$ satisfies \textbf{(C)} and \textbf{(D)}. Then, with probability tending to $1$ as $\vec{n} \longrightarrow \infty$, PartialLabelHat($ \vec{S}, (\ed{\tilde{P}}_1, \ldots, \ed{\tilde{P}}_L)^T$) returns $(\est{Q}_1,\ldots, \est{Q}_L)^T$ such that for every $i \in [L]$,
\begin{equation*}
\sup_{S \in \sS} |\est{Q}_i(S) - P_i(S)| < \epsilon.
\end{equation*}
\end{thm}

\begin{proof}[Theorem \ref{partial_label_hat}]
Let $(\est{Q}_1, \ldots, \est{Q}_L) \longleftarrow \text{DemixHat}(\ed{\tilde{P}}_1, \ldots, \ed{\tilde{P}}_L)$. By Theorem \ref{demix_estim}, there exists a permutation $\sigma: [L] \longrightarrow [L]$ such that for every $i \in [L]$,
\begin{equation*}
\sup_{S \in \sS} |\est{Q}_i(S) - P_{\sigma(i)}(S)| < \epsilon.
\end{equation*}
From the proof of Theorem \ref{demix_estim}, it is clear that each $\est{Q}_i$ is a ResidueHat estimator. The assumptions of Lemma \ref{vertex_hat_lemma} are satisfied. The result follows immediately from  Lemma \ref{vertex_hat_lemma}.
\end{proof}

\clearpage
\bibliography{references}

\begin{thebibliography}{24}
\providecommand{\natexlab}[1]{#1}
\providecommand{\url}[1]{\texttt{#1}}
\expandafter\ifx\csname urlstyle\endcsname\relax
  \providecommand{\doi}[1]{doi: #1}\else
  \providecommand{\doi}{doi: \begingroup \urlstyle{rm}\Url}\fi

\bibitem[Arora et~al.(2012)Arora, Ge, and Moitra]{arora_beyond_2012}
S.~Arora, R.~Ge, and A.~Moitra.
\newblock {L}earning {T}opic {M}odels--going beyond {SVD}.
\newblock \emph{Foundations of Computer Science}, 2012.

\bibitem[Arora et~al.(2013)Arora, Ge, Halpern, Mimno, Moitra, Sontag, Wu, and
  Zhu]{arora_practical_2012}
S.~Arora, R.~Ge, Y.~Halpern, D.~Mimno, A.~Moitra, D.~Sontag, Y.~Wu, and M.~Zhu.
\newblock A {P}ractical {A}lgorithm for {T}opic {M}odeling with {P}rovable
  {G}uarantees.
\newblock \emph{Proceedings of the 30th International Conference on Machine
  Learning}, 2013.

\bibitem[Axler(2015)]{axler}
S.~Axler.
\newblock \emph{{L}inear {A}lgebra {D}one {R}ight}.
\newblock Springer, 3 edition, 2015.

\bibitem[Berkman et~al.(1989)Berkman, Singer, and Manton]{berkman}
L.~Berkman, B.~H. Singer, and K.~Manton.
\newblock {B}lack/{W}hite {D}ifferences in {H}ealth {S}tatus and {M}ortality
  among the {E}lderly.
\newblock \emph{Demography}, 26:\penalty0 661–678, 1989.

\bibitem[Blanchard and Scott(2014)]{blanchard2014}
G.~Blanchard and C.~Scott.
\newblock {D}econtamination of {M}utually {C}ontaminated {M}odels.
\newblock \emph{Proc. 17th Int. Conf. Artificial Intelligence and Statistics
  (AISTATS)}, pages 1--9, 2014.

\bibitem[Blanchard et~al.(2010)Blanchard, Lee, and Scott]{blanchard2010}
G.~Blanchard, G.~Lee, and C.~Scott.
\newblock {S}emi-{S}upervised {N}ovelty {D}etection.
\newblock \emph{Journal of Machine Learning Research}, 11:\penalty0 2973--3009,
  2010.

\bibitem[Blei et~al.(2003)Blei, Ng, and Jordan]{blei2003}
D.~Blei, A.~Ng, and M.~Jordan.
\newblock {L}atent {D}irichlet {A}llocation.
\newblock \emph{Journal of Machine Learning research}, 3:\penalty0 993--1022,
  2003.

\bibitem[Cour et~al.(2011)Cour, Sapp, and Taskar]{cour2011}
T.~Cour, B.~Sapp, and B.~Taskar.
\newblock Learning from partial labels.
\newblock \emph{Journal of Machine Learning}, 12:\penalty0 1501--1536, 2011.

\bibitem[Ding et~al.(2013)Ding, Rohban, Ishwar, and Saligrama]{ding_2013}
W.~Ding, M.~Rohban, P.~Ishwar, and V.~Saligrama.
\newblock {T}opic {D}iscovery through {D}ata {D}ependent and {R}andom
  {P}rojections.
\newblock \emph{Proceedings of the 30th International Conference on Machine
  Learning}, 2013.

\bibitem[Ding et~al.(2014)Ding, Rohban, Ishwar, and Saligrama]{ding2014}
W.~Ding, M.~Rohban, P.~Ishwar, and V.~Saligrama.
\newblock {E}fficient {D}istributed {T}opic {M}odeling with {P}rovable
  {G}uarantees.
\newblock \emph{In Proceedings of the 17th International Conference on
  Artificial Intelligence and Statistics}, pages 167--175, 2014.

\bibitem[Donoho and Stodden(2003)]{donoho}
D.~Donoho and V.~Stodden.
\newblock {W}hen does {N}on-negative {M}atrix {F}actorization give a {C}orrect
  {D}ecomposition into {P}arts?
\newblock \emph{Advances in neural information processing systems}, 2003.

\bibitem[Jain et~al.(2016)Jain, White, Trosset, and Radivojac]{jain2016}
S.~Jain, M.~White, M.~W. Trosset, and P.~Radivojac.
\newblock Nonparametric semi-supervised learning of class proportions.
\newblock \emph{arXiv preprint arXiv:1601.01944}, 2016.

\bibitem[Jin and Ghahramani(2002)]{jin2002}
R.~Jin and Z.~Ghahramani.
\newblock Learning with multiple labels.
\newblock \emph{Advances in Neural Information Processing Systems}, pages
  897--904, 2002.

\bibitem[Li and Perona(2005)]{li}
F.~Li and P.~Perona.
\newblock A {B}ayesian {H}ierarchical {M}odel for {L}earning {N}atural {S}cene
  {C}ategories.
\newblock \emph{Computer Vision and Pattern Recognition}, 2005.

\bibitem[Liu and Dietterich(2012)]{liu2012}
L.-P. Liu and T.~G. Dietterich.
\newblock A conditional multinomial mixture model for superset label learning.
\newblock \emph{Advances in Neural Information Processing Systems}, pages
  557--565, 2012.

\bibitem[Liu and Dietterich(2014)]{liu2014}
L.-P. Liu and T.~G. Dietterich.
\newblock Learnability of the superset label learning problem.
\newblock \emph{Proceedings of the 31st International Conference on Machine
  Learning}, 32, 2014.

\bibitem[Luong et~al.(2013)Luong, Socher, and Manning]{luong2013}
M.-T. Luong, R.~Socher, and C.~D. Manning.
\newblock Better word representations with recursive neural networks for
  morphology.
\newblock \emph{CoNLL-2013}, 104, 2013.

\bibitem[Nyugen and Caruana(2008)]{nyugen2008}
N.~Nyugen and R.~Caruana.
\newblock Classification with partial labels.
\newblock \emph{Proceedings of the 14th International Conference on Knowledge
  Discovery and Data Mining}, pages 551--559, 2008.

\bibitem[Poczos et~al.(2012)Poczos, Xiong, Sutherland, and
  Schneider]{poczos2012}
B.~Poczos, L.~Xiong, D.~J. Sutherland, and J.~Schneider.
\newblock Nonparametric kernel estimators for image classification.
\newblock \emph{Computer Vision and Pattern Recognition (CVPR)}, 2012.

\bibitem[Pritchard et~al.(2000)Pritchard, Stephens, Rosenberg, and
  Donnelly]{pritchard}
J.~K. Pritchard, M.~Stephens, N.~A. Rosenberg, and P.~Donnelly.
\newblock {A}ssociation {M}apping in {S}tructured {P}opulations.
\newblock \emph{American Journal of Human Genetics}, 67:\penalty0 170–181,
  2000.

\bibitem[Recht et~al.(2012)Recht, Re, Tropp, and Bittorf]{recht_2012}
B.~Recht, C.~Re, J.~Tropp, and V.~Bittorf.
\newblock {F}actoring {N}on-negative {M}atrices with {L}inear {P}rograms.
\newblock \emph{In Advances in Neural Information Processing Systems}, pages
  1214--1222, 2012.

\bibitem[Sanderson and Scott(2014)]{sanderson2014}
T.~Sanderson and C.~Scott.
\newblock {C}lass {P}roportion {E}stimation with {A}pplication to {M}ulticlass
  {A}nomaly {R}ejection.
\newblock \emph{Proc. 17th Int. Conf. Artificial Intelligence and Statistics
  (AISTATS)}, W\&CP 33:\penalty0 850--858, 2014.

\bibitem[Scott(2015)]{scott2015}
C.~Scott.
\newblock A {R}ate of {C}onvergence for {M}ixture {P}roportion {E}stimation,
  with {A}pplication to {L}earning from {N}oisy {L}abels.
\newblock \emph{Proc 18th Int. Conf. on Artificial Intelligence and
  Statistics}, pages 838--846, 2015.

\bibitem[Scott et~al.(2013)Scott, Blanchard, and Handy]{scott2013}
C.~Scott, G.~Blanchard, and G.~Handy.
\newblock {C}lassification with {A}symmetric {L}abel {N}oise: {C}onsistency and
  {M}aximal {D}enoising.
\newblock \emph{Proc. Conf. on Learning Theory (COLT)}, 30:\penalty0 489--511,
  2013.

\end{thebibliography}

\end{document}